\documentclass[11pt]{article}
\usepackage{jmlr2e}
\firstpageno{1}
\usepackage{microtype}
\usepackage{graphicx}
\usepackage{subfigure}
\usepackage{booktabs} 
\usepackage{multirow}
\usepackage{amssymb}
\usepackage{mathtools}

\usepackage{enumitem}

\usepackage{amsmath}

\usepackage{algorithm}
\usepackage{algorithmic}
\usepackage{makecell}
\usepackage{bookmark}
\usepackage[titletoc,page]{appendix}
\usepackage[nottoc,notlot,notlof]{tocbibind}

\newtheorem{assumption}{Assumption}

\def \X {\mathcal{X}}

\def \R {\mathbb{R}}

\def \x {\mathbf{x}}

\def \E {\mathbb{E}}

\def \x {\mathbf{x}}

\def \y {\mathbf{y}}

\def \F {\mathcal{F}}

\def \P {\mathcal{P}}

\def \C {\mathbf C}

\def \bx {\boldsymbol{x}}
\def \by {\boldsymbol{y}}
\def \bz {\boldsymbol{z}}
\def \bv {\boldsymbol{v}}
\DeclareMathOperator{\argmin}{\arg\min}
\DeclareMathOperator{\argmax}{\arg\max}
\usepackage[capitalize,noabbrev]{cleveref}
\hypersetup{colorlinks={true},linkcolor={blue},citecolor={blue}}

\newcommand{\compilehidecomments}{false} 
\ifthenelse{ \equal{\compilehidecomments}{false} }{%
	\newcommand{\zongqi}[1]{{\color{red}  [\text{Zongqi:} #1]}}
}{
	\newcommand{\zongqi}[1]{}
}

\allowdisplaybreaks

\begin{document}

\title{Boosting Gradient Ascent for Continuous DR-submodular Maximization\thanks{Preliminary results of this paper were presented in part at the 2022 International Conference on Machine Learning~\citep{pmlr-v162-zhang22e}. $\dagger$. Equal Contribution. $\diamondsuit$. Corresponding Authors.}}

\author{\name Qixin Zhang\textsuperscript{$1, \dagger$}\email qxzhang4-c@my.cityu.edu.hk\\
       \name Zongqi Wan\textsuperscript{$3, 4, \dagger$}\email wanzongqi20s@ict.ac.cn 
     \\
       \name Zengde Deng\textsuperscript{$2$} \email dengzengde@gmail.com
       \\
       \name Zaiyi Chen\textsuperscript{$2$}\email zaiyi.chen@outlook.com\\
       \name Xiaoming Sun\textsuperscript{$3,4$} \email sunxiaoming@ict.ac.cn\\
       \name Jialin Zhang\textsuperscript{$3,4,\diamondsuit$} \email zhangjialin@ict.ac.cn \\
       \name Yu Yang\textsuperscript{$1,\diamondsuit$} \email yuyang@cityu.edu.hk
      \\ \addr $^1$ School of Data Science, City University of Hong Kong; $^2$ ByteDance;
      \\ $^3$ State Key Lab of Processors, Institute of Computing Technology, Chinese Academy of Sciences;
      \\ $^4$ School of Computer Science and Technology, University of Chinese Academy of Sciences.
      }
\editor{Anonymous}
\maketitle
\begin{abstract}
    Projected Gradient Ascent (PGA) is the most commonly used optimization scheme in machine learning and operations research areas. Nevertheless, numerous studies and examples have shown that the PGA methods may fail to achieve the tight approximation ratio for continuous DR-submodular maximization problems. To address this challenge,  we present a boosting technique in this paper, which can efficiently improve the approximation guarantee of the standard PGA to \emph{optimal} with only small modifications on the objective function. The fundamental idea of our boosting technique is to exploit non-oblivious search to derive a novel auxiliary function $F$, whose stationary points are excellent approximations to the global maximum of the original DR-submodular objective $f$. Specifically, when $f$ is monotone and $\gamma$-weakly DR-submodular, we propose an auxiliary function $F$ whose stationary points can provide a better $(1-e^{-\gamma})$-approximation than the $(\gamma^2/(1+\gamma^2))$-approximation guaranteed by the stationary points of $f$ itself. Similarly, for the non-monotone case, we devise another auxiliary function $F$ whose stationary points can achieve an optimal $\frac{1-\min_{\bx\in\mathcal{C}}\|\bx\|_{\infty}}{4}$-approximation guarantee where $\mathcal{C}$ is a convex constraint set. In contrast, the stationary points of the original non-monotone DR-submodular function can be arbitrarily bad~\citep{chen2023continuous}. Furthermore, we demonstrate the scalability of our boosting technique on four problems, i.e., offline stochastic DR-submodular maximization, online DR-submodular maximization, bandit DR-submodular maximization, and minimax optimization of convex-submodular function. In all of these four problems, our resulting variants of boosting PGA algorithm beat the previous standard PGA in several aspects such as approximation ratio and efficiency. Finally, we corroborate our theoretical findings with numerical experiments, which demonstrate the effectiveness of our boosting PGA methods.\\
\end{abstract}

\thispagestyle{empty}
\newpage
\setcounter{page}{1}


\section{Introduction}
Due to the relatively low computational complexity, first-order optimization methods are widely used in machine learning, operations research, and statistics communities. Especially for convex objectives, there is an enormous literature~\citep{nesterov2013introductory,bertsekas2015convex} deriving the corresponding convergence rate of first-order methods. Recent studies have shown that first-order optimization methods also can achieve the global minimum for some special non-convex problems~\citep{NIPS2014_443cb001,arora2016computing,ge2016matrix, du2018gradient,liu2020nonconvex}, although it is in general NP-hard to find a global minima of a non-convex objective function~\citep{murty1987some}. Motivated by this, massive research focused on the structures and conditions under which non-convex optimization is tractable~\citep{bian2017guaranteed,hazan2016graduated}. In this paper, we investigate a subclass of tractable non-convex problems, that is, stochastic continuous DR-submodular maximization. 

Continuous DR-submodular Maximization has drawn much attention recently due to that it admits efficient approximate maximization routines. 
For instance, under the deterministic monotone setting, \citet{bian2017guaranteed} proposed a variant of the Frank-Wolfe method achieving the optimal $(1-1/e)$-approximation guarantee. Although this Frank-Wolfe method plays an important role in achieving the tight approximation ratio, it is not easy to extend it to other settings such as stochastic optimization and online learning. Adapting the Frank-Wolfe method to more complicated settings usually requires some new and customized technical components and assumptions. When the stochastic estimates of the gradient are available, \cite{hassani2017gradient} pointed out that the Frank-Wolfe method \citep{bian2017guaranteed} performs poorly and can produce arbitrarily bad solutions. To tackle this challenge, \citet{mokhtari2018conditional} merged the variance reduction techniques into the previous Frank-Wolfe method~\citep{bian2017guaranteed}. Assuming the Lipschitz continuity of stochastic Hessian, an accelerated Frank-Wolfe algorithm is proposed by \cite{hassani2020stochastic} with the optimal stochastic first-order oracle complexity. Similarly, some other tricks should be involved to generalize Frank-Wolfe methods to the online setting, which makes the algorithm design more complicated. For example, \citet{chen2018online} and \citet{zhang2019online} took the idea of meta actions~\citep{streeter2008online} and blocking procedure to design online Frank-Wolfe algorithms. To achieve $O(\sqrt{T})$-regret, these algorithms require querying a significant number of gradients of the online function in each round, which triggers an efficiency concern. Moreover, in these aforementioned studies of online settings, the environment/adversary reveals the reward and stochastic first-order information immediately after the action is chosen by the learner/algorithm. In practice, the assumption of immediate feedback might be too restrictive. The feedback delays widely exist in many real-world applications, e.g., online advertising~\citep{mehta2007adwords}, influence maximization problem~\citep{chen2012time,yang2016continuous}. Also, the Frank-Wolfe methods for the general non-monotone DR-submodular maximization suffer similar issues \citep{hassani2020stochastic,zhang2023online,mualem2023resolving} when applied to the stochastic and online scenarios.

To address these issues, a natural algorithmic candidate is the Projected Gradient Ascent(PGA) algorithm, whose framework is not only simple to execute but more robust to the fluctuations of optimization environments, compared with Frank-Wolfe algorithms. However, for the DR-Submodular maximization problems, PGA algorithm only can guarantee a sub-optimal approximation ratio. Specifically, in contrast with the tight $(1-e^{-\gamma})$-approximation ratio for monotone $\gamma$-weakly DR-submodular function, PGA only produces a suboptimal $(\frac{\gamma^{2}}{1+\gamma^{2}})$-approximation to the global maximum~\citep{hassani2017gradient}. Not to mention that running PGA on a non-monotone DR-submodular function may produce an arbitrarily bad solution~\citep{chen2023continuous}. Thus, this article aims at revolving around the following question:

\vspace{1pt}
\begin{center}
    \emph{Can we boost the PGA methods to achieve the optimal approximation ratio for continuous DR-submodular maximization problems?}
\end{center}

Our answer to this question is affirmative. According to \citet{hassani2017gradient}, the standard projected gradient ascent method can converge to a stationary point of the continuous DR-submodular objective $f$ under mild assumptions. It is the unsatisfied performance of stationary points of $f$ that severely deteriorates the approximation guarantee of the standard PGA method. To overcome this drawback, we technically hope to devise an auxiliary function whose stationary points provide a better approximation guarantee than those of $f$ itself. Then we can obtain a better solution by running PGA on the auxiliary function. To be specific, for the monotone $\gamma$-weakly DR-submodular objective $f$, we first consider a family of auxiliary functions whose gradient at point $\boldsymbol{x}$ allocates different weight to the gradient $\nabla f(z\cdot\boldsymbol{x})$ where $z\in[0,1]$. By solving a factor-revealing optimization problem, we select the optimal auxiliary function whose stationary points provide a tight $(1-e^{-\gamma})$-approximation to the global maximum of the original function $f$. Then, based on this optimal auxiliary function, we boost the projected gradient ascent method to $(1-e^{-\gamma})$-approximation guarantee under both offline and online settings. 
When the objective function $f$ is non-monotone and DR-submodular, we consider another form of auxiliary functions whose gradient at point $\bx$ allocations different weight to the gradient $\nabla f(z\alpha\cdot\bx+(1-z\alpha)\cdot\underline{\bx})$, where $\alpha\in (0,1)$ is a fixed parameter to be determined and $\underline{\bx}$ is the feasible solution with lowest infinity norm, that is, $\underline{\bx} := \argmin_{\bx\in \mathcal{C}} \|\bx\|_{\infty}$ where $\mathcal{C}$ is the constraint of the problem. By selecting weights and $\alpha$ carefully, we construct an auxiliary function whose stationary point indicates a $\frac{1-\|\underline{\bx}\|_{\infty}}{4}$-approximation solution to the global maximum of $f$. Specially, if $\bx$ is a stationary point of the auxiliary function, then $\frac{\bx+\underline{\bx}}{2}$ is  a $\frac{1-\|\underline{\bx}\|_{\infty}}{4}$-approximation solution to $f$. Then we can boost both offline and online gradient ascent algorithms to $\frac{1-\|\underline{\bx}\|_{\infty}}{4}$-approximation which has been proved optimal~\citep{mualem2023resolving}. 

Furthermore, the auxiliary function elaborately designed by us can also be applied to the minimax optimization of convex-submodular functions. This problem is coined by \cite{adibi2022minimax} where only the case when the submodular part is monotone is considered. We improve the approximation ratio and extend the result for the case where the submodular part of the objective is non-monotone.

\paragraph{Contributions} To summarize, we make the following contributions:
\begin{enumerate} 
	\item 
	We design non-oblivious (auxiliary) functions for both monotone $\gamma$-weakly DR-submodular functions and general non-monotone DR-submodular functions. Any stationary point of the non-oblivious function indicates a $(1-e^{-\gamma})$-approximation solution for monotone $\gamma$-weakly function and $\frac{1-\|\underline{\bx}\|_{\infty}}{4}$-approximation for non-monotone function, respectively. As a comparison, a stationary point of the original objective function itself only provides a $(\frac{\gamma^2}{1+\gamma^2})$-approximation for the monotone $\gamma$-weakly function. Moreover, there is no approximation ratio guarantee so far for the stationary points of a non-monotone function~\citep{chen2023continuous}. Our non-oblivious functions make it possible to boost PGA method to attain tight approximation ratios.

	\item For offline stochastic DR-submodular maximization over a general convex set constraint, we propose the boosting gradient ascent method using the non-oblivious technique. 
	Our algorithm achieves a $(1-e^{-\gamma})$-approximation for monotone $\gamma$-weakly functions, which improves the $(\frac{\gamma^2}{1+\gamma^2})$-approximation of the classical projected gradient ascent algorithm and weakens the assumption of high order smoothness on the objective functions \citep{hassani2020stochastic}. For general non-monotone functions, our algorithm achieves the optimal $\frac{1-\|\underline{\bx}\|_{\infty}}{4}$-approximation, which is in accord with the best-known approximation ratio of deterministic non-monotone Frank-Wolfe variants~\citep{du2022improved, mualem2023resolving} over a general convex constraint.
 
	
	\item Next, we consider an online submodular maximization setting with adversarial feedback delays. 
	When an unbiased stochastic gradient estimation is available, we propose an online boosting gradient ascent algorithm that theoretically achieves the optimal $(1-e^{-\gamma})$-regret of $O(\sqrt{D})$ for monotone functions and $\frac{1-\|\underline{\bx}\|_{\infty}}{4}$-regret of $O(\sqrt{D})$ for non-monotone functions. Here $D=\sum_{t=1}^{T}d_{t}$ and $d_{t}$ is a positive integer delay for round $t$. To the best of our knowledge, our work is the first to investigate adversarial delays in online submodular maximization problems. 
 Remarkably, when $D=T$ for the standard no-delay setting, our proposed online boosting gradient ascent algorithm yields the first result to achieve a tight approximation ratio of $O(\sqrt{T})$ regret with only $O(1)$ stochastic gradient estimate at each round. We also extend our result to the \emph{bandit feedback} model where the algorithm can only observe the function value of the selected action rather than the entire function. Under this feedback model, we boost the bandit gradient ascent method via non-oblivious functions and obtain $(1-e^{-\gamma})$-regret and $\frac{1-\|\underline{\bx}\|_{\infty}}{4}$-regret of $O(T^{4/5})$ for the monotone case and general non-monotone case respectively. Specially, our regret bound of the monotone case improves the results by \cite{zhang2019online}, \cite{niazadeh2022online} as well as \cite{zhang2023online}. Moreover, we are the first result to study the online bandit learning of the non-monotone DR-submodular function over a convex set constraint. 

	\item We also apply our non-oblivious technique on the minimax optimization of convex-submodular functions, where we consider a general matroid constraint for the submodular part. When the submodular part is monotone, our algorithm achieves $(1-1/e)$-approximation, which improves the previous $\frac{1}{2}$-approximation result~\citep{adibi2022minimax}. As for non-monotone settings, our algorithm achieves $\frac{1}{4}$-approximation while we do not recognize any other algorithm that can achieve a constant approximation ratio under the same setting.
	
	\item Finally, we empirically evaluate our proposed boosting methods using the special examples of coverage maximization \citep{hassani2017gradient,chen2023continuous}, the simulated non-convex/non-concave quadratic programming, and movie recommendation. Our algorithms achieve superior performance in all experiments.
\end{enumerate}

\begin{table}[t]
	\renewcommand\arraystretch{1.35}
	\centering
	\caption{\small Comparison of convergence guarantees for continuous DR-submodular function maximization. Note that `\textbf{Mono.}' means the monotonicity of the objective. Especially, `mono.' and `general' means that the object function is monotone and general non-monotone respectively. `\textbf{Cons.}' means the constraint set and `d.c.' represents the downward closed convex set. `det.' and `sto.' represent the deterministic and stochastic setting, respectively. `\textbf{Hess Lip}' means whether the Hessian of functions needs to be Lipschitz continuous, `OPT' is the function value at the global optimum, `\textbf{Complexity}' is the total number of queries to the gradient oracle. }
	\vspace{0.5em}
	\resizebox{\textwidth}{!}{
		\setlength{\tabcolsep}{1.0mm}{
			\begin{tabular}{c|c|c|c|c|c|c}
				\toprule[1.5pt]
				\textbf{Method} &\textbf{Mono.} &\textbf{Cons.} & \textbf{Setting}& \textbf{Hess Lip.} &\textbf{Utility} & \textbf{Complexity} \\
				\hline
				\hline 
				\makecell{Submodular FW\\ \citep{bian2017guaranteed}}&mono. & d.c. & det. & No & $(1-1/e)\rm{OPT}-\epsilon$ & $O(1/\epsilon)$ \\
				\hline
				\makecell{SGA \\ \citep{hassani2017gradient}}&mono. &convex & sto. &No & $(1/2)\rm{OPT}-\epsilon$ & $O(1/\epsilon^2)$ \\
				\hline 
				\makecell{Classical FW \\ \citep{bian2020continuous}} &mono. &convex & det. & No & $(1/2)\rm{OPT}-\epsilon$ & $O(1/\epsilon^2)$ \\
				\hline
				\makecell{SCG \\ \citep{mokhtari2018conditional}} &mono. &convex & sto. & No & $(1-1/e)\rm{OPT}-\epsilon$ & $O(1/\epsilon^3)$ \\
				\hline
				\makecell{SCG++\\ \citep{hassani2020stochastic}}&mono. &convex & sto. & Yes & $(1-1/e)\rm{OPT}-\epsilon$ & $O(1/\epsilon^2)$ \\
				\hline
				\makecell{Non-Oblivious FW\\ \citep{mitra2021submodular+}} &mono. &convex & det. & No & $(1-1/e-\epsilon)\rm{OPT}-\epsilon$ & $O(1/\epsilon^3)$ \\
				\hline
				\hline
				\makecell{Non-monotone FW \\ \citep{bian2017continuous}} &general &d.c. &det. &No & $(1/e)\rm{OPT}-\epsilon$ & $O(1/\epsilon)$\\
				\hline
				\makecell{SMCG++\\ \citep{hassani2020stochastic}} &general &d.c. &sto. &Yes & $(1/e)\rm{OPT}-\epsilon$ & $O(1/\epsilon^2)$\\
				\hline
				\makecell{Non-mon. FW\\
\citep{du2022lyapunov}\\
\citep{mualem2023resolving}}&general &convex &det. &No & $\frac{1-\min_{\bx\in\mathcal{C}}\|\bx\|_{\infty}}{4}\rm{OPT}-\epsilon$ & $O(1/\epsilon)$\\
                \hline
				\makecell{\citep{pedramfar2023unified}}&general &convex &sto.& No & $\frac{1-\min_{\bx\in\mathcal{C}}\|\bx\|_{\infty}}{4}\rm{OPT}-\epsilon$ & $O(1/\epsilon^3)$ \\
				\hline
				\hline
				\multirow{2}{*}{\makecell{Boosting GA\\(\cref{thm:4},\cref{thm:4 nonmonotone})}} &mono. &convex & sto. & No & $(1-1/e)\rm{OPT}-\epsilon$ & $O(1/\epsilon^2)$ \\
				&general &convex &sto. &No &$\frac{1-\min_{\bx\in\mathcal{C}}\|\bx\|_{\infty}}{4}\rm{OPT}-\epsilon$ & $O(1/\epsilon^2)$\\
				\midrule[1.5pt]
			\end{tabular}
	}}
	\vspace{-1.5em}
	\label{tab:offline_convergence}
\end{table}

\subsection{Related Works}
In this section, we review the work related to this paper. We also present comparisons between this work and previous studies in \cref{tab:offline_convergence},  \cref{tab:online_convergence}, \cref{tab:bandit_online_convergence}, \cref{tab:minimax_convergence} and \cref{tab:stationary_points} for offline optimization, online learning, bandit online learning, minimax setting and approximation guarantee of stationary points, respectively.

\paragraph{Submodular Set Functions} Submodular set functions originate from combinatorial optimization problems \citep{nemhauser1978analysis,fisher1978analysis,fujishige2005submodular}, which could be either exactly minimized via Lov\'{a}sz extension~\citep{lovasz1983submodular} or approximately maximized~\citep{chekuri2014submodular,buchbinder2019constrained}. Submodular set functions find numerous applications in machine learning and other related areas, including viral marketing \citep{kempe2003maximizing}, document summarization \citep{lin2011class}, network monitoring \citep{leskovec2007cost}, and variable selection \citep{das2011submodular,elenberg2018restricted}.

\paragraph{Continuous Submodular Maximization} 
Submodularity can be naturally extended to continuous domains. For monotone functions, \citet{bian2017guaranteed} first proposed a variant of Frank-Wolfe~(Submodular FW) for continuous DR-submodular maximization problem with $(1-1/e)$-approximation guarantee after $O(1/\epsilon)$ iterations under deterministic gradient oracle. When considering the stochastic gradient oracle, \citet{hassani2017gradient} proved that the stochastic gradient ascent~(SGA) guarantees a $(1/2)$-approximation after $O(1/\epsilon^{2})$ iterations. Then, \citet{mokhtari2018conditional} proposed the stochastic continuous greedy algorithm~(SCG), which achieves a $(1-1/e)$-approximation after $O(1/\epsilon^{3})$ iterations. Moreover, by assuming the Hessian of objective is Lipschitz continuous, \citet{hassani2020stochastic} proposed the stochastic continuous greedy++~(SCG++), which guarantees a $(1-1/e)$-approximation after $O(1/\epsilon^{2})$ iterations. For non-monotone functions, the maximization problem becomes more challenging, and the state-of-the-art approximation ratios highly depend on the structure of constraint set. \cite{bian2019optimal} and \cite{niazadeh2020optimal} proposed similar $1/2$-approximation algorithms over the hypercube constraint. Under the downward-closed convex constraint, \cite{bian2017continuous} proposed the deterministic Two-Phase Frank-Wolfe and nonmonotone Frank-Wolfe with $1/4$-approximation and $1/e$-approximation guarantee respectively. The above results require the deterministic gradient oracle. As for stochastic gradient oracle, \cite{hassani2020stochastic} improved the nonmonotone Frank-Wolfe by variance reduction technique, which yields a result with $1/e$-approximation ratio. Under general convex constraints, \cite{vondrak2013symmetry} pointed out that any algorithm with a constant-factor approximation ratio requires exponential many queries. Luckily, \cite{durr2021non} found that the approximation ratio can be written in terms of the minimal $\ell_{\infty}$-norm of the vectors in the constraint set. To be specific, they proposed an algorithm with a $\frac{1-\min_{\bx\in\mathcal{C}}\|\bx\|_{\infty}}{3\sqrt{3}}$ approximation ratio. \cite{du2022lyapunov} improved the approximation to $\frac{1}{4}(1-\min_{\bx\in\mathcal{C}}\|\bx\|_{\infty})$, which was shown optimal by \cite{mualem2023resolving}. \cite{pedramfar2023unified} extended this result and proposed the algorithm for stochastic setting, which achieves the $\frac{1-\min_{\bx\in\mathcal{C}}\|\bx\|_{\infty}}{4}$ approximation ratio with $O(1/\epsilon^3)$ queries to the stochastic gradient oracle.

\paragraph{Online Continuous Submodular Maximization} For monotone case, \citet{chen2018online} first investigated the online (stochastic) gradient ascent~(OGA) with a $(1/2)$-regret of $O(\sqrt{T})$. Then, inspired by the meta-action technique~\citep{streeter2008online}, \citet{chen2018online} also proposed the Meta-Frank-Wolfe algorithm with a $(1-1/e)$-regret bound of $O(\sqrt{T})$ under the deterministic setting. Assuming that an unbiased estimation of the gradient is available, \citet{chen2018projection} proposed a variant of the Meta-Frank-Wolfe algorithm~(Meta-FW-VR), having a $(1-1/e)$-regret bound of $O(T^{1/2})$ and requiring $O(T^{3/2})$ stochastic gradient queries for each function. Then, in order to reduce the number of gradient evaluations, \citet{zhang2019online} presented the Mono-Frank-Wolfe taking the blocking procedure, which achieves a $(1-1/e)$-regret bound of $O(T^{4/5})$ with only one stochastic gradient evaluation in each round. This result is improved by \cite{niazadeh2022online}, \citet{liao2023improved} and \cite{pedramfar2024unified}. Additionally, \citet{liao2023improved} also extended their results to the distributed optimization setting.

For non-monotone functions, \cite{thang2021online} first explored the sublinear-regret online algorithm over a downward-closed set, where they devised an algorithm achieving $1/e$-regret of $O(T^{3/4})$ with access to non-convex online maximization oracle and $O(T^{3/4})$ gradient queries per round. \cite{zhang2023online} improved the result,
their algorithm only requires access to linear online maximization oracle and can trade off between the regret and the query complexity, achieving $O(\sqrt{T})$ regret with $O(T^{2/3})$ queries per round and $O(T^{4/5})$ regret with $O(1)$ queries per round. Furthermore, they extend the result to the bandit feedback model. Under the general convex set constraint, \cite{mualem2023resolving} proposed the algorithm with $\frac{1-\min_{\bx\in\mathcal{C}}{\|\bx\|_{\infty}}}{4}$-regret of $O(\sqrt{T})$. \cite{pedramfar2024unified} improved on the previous results for both the downward-closed constraint and the convex-set constraint. They achieved $O(T^{2/3})$ regret with the same approximation ratio, using $O(1)$ queries per round. 

\begin{table}[t]
	\renewcommand\arraystretch{1.35}
	\centering
	\caption{Comparison of regrets for stochastic online continuous DR-submodular function maximization with full-information feedback. Note that `\textbf{\# Grad. Evaluations}' means the number of stochastic gradient evaluations at each round, `\textbf{Ratio}' means approximation ratio, and `\textbf{Delay}' indicates whether the adversarial delayed feedback is considered. $D=T$ means no delay exists.}
	\vspace{0.5em}
	\resizebox{\textwidth}{!}{
		\setlength{\tabcolsep}{1.0mm}{
			\begin{tabular}{c|c|c|c|c|c|c}
				\toprule[1.5pt]
				\textbf{Method}&\textbf{Mono.}&\textbf{Cons.} & \textbf{\# Grad. Evaluations } &\textbf{Ratio} & \textbf{Regret} &\textbf{Delay} \\
				\hline
				\makecell{OGA\\\citep{chen2018online}} &mono. &convex & $O(1)$ & $1/2$ & $O(\sqrt{T})$ & No \\
				\hline
				\makecell{Meta-FW-VR\\ \citep{chen2018projection}}&mono. &convex & $O(T^{3/2})$ & $1-1/e$ & $O(\sqrt{T})$ & No\\
				\hline
				\makecell{Mono-FW\\ \citep{zhang2019online}} &mono. &convex & $O(1)$ & $1-1/e$ & $O(T^{4/5})$ & No\\
                \hline
    			\makecell{\citep{niazadeh2022online}}& mono. &d.c. & $O(\sqrt{T})$ & $1-1/e$ & $O(\sqrt{T})$ & No\\
				\hline
                \makecell{\citep{liao2023improved}} &mono. &convex & $O(1)$ & $1-1/e$ & $O(T^{3/4})$ & No\\
                \hline
                \makecell{\citep{pedramfar2024unified}} &mono. &convex & $O(1)$ & $1-1/e$ & $O(T^{2/3})$ & No\\
                \hline
				\hline
				\makecell{ODC\\ \citep{thang2021online}}& general &d.c. & $O(T^{3/4})$ & $1/e$ & $O(T^{3/4})$ & No\\
     \hline
				\makecell{Meta-MFW\\ \citep{zhang2023online}}&general &d.c. &$O(T^{3/2})$ & $1/e$ & $O(\sqrt{T})$& No\\ \hline
				\makecell{Non-monotone Meta-FW\\ \citep{mualem2023resolving}}&general &convex &$O(\sqrt{T})$& $\frac{1-\min_{\bx\in\mathcal{C}}\|\bx\|_{\infty}}{4}$ &$O(\sqrt{T})$&No \\
                \hline
                \multirow{2}{*}{\citep{pedramfar2024unified}} &general &d.c. & $O(1)$ & $1/e$ & $O(T^{2/3})$ & No\\
                &general &convex & $O(1)$ & $\frac{1-\min_{\bx\in\mathcal{C}}\|\bx\|_{\infty}}{4}$ & $O(T^{2/3})$ & No\\
                \hline
				\hline
				\multirow{2}{*}{\makecell{Boosting OGA\\ (\cref{thm:5},\cref{thm:5 nonmonotone})}}&mono. &convex & $O(1)$ & $1-1/e$ & $O(\sqrt{T})$ & Yes\\
				&general &convex & $O(1)$ & $\frac{1-\min_{\bx\in\mathcal{C}}\|\bx\|_{\infty}}{4}$ &$O(\sqrt{T})$ & Yes
				\\
				\midrule[1.5pt]
			\end{tabular}
	}}
	\vspace{-1.5em}
	\label{tab:online_convergence}
\end{table}

\paragraph{Bandit Continuous Submodular Maximization} \cite{zhang2019online} first studied the continuous submodular maximization problem under the bandit feedback model. The algorithm Bandit-FW they proposed achieves $O(T^{8/9})$ $(1-1/e)$-regret for monotone DR-submodular function and downward-closed convex set constraint. \cite{niazadeh2022online} further improved this result to $O(T^{5/6})$ $(1-1/e)$-regret. \cite{pmlr-v202-wan23e} improved the regret bound to $\widetilde{O}(T^{2/3})$ while assuming multi-linearity of the online functions. Furthermore, they applied this result to the discrete submodular bandit via a special continuous extension. As for the non-monotone case, \cite{zhang2023online} proposed the Bandit-MFW algorithm which achieves $O(T^{8/9})$ of $1/e$-regret over a downward-closed convex set constraint. \citet{pedramfar2024unified} improved this result to $O(T^{5/6})$ regret and designed the algorithm for convex set constraint, which achieved $O(T^{5/6})$ regret with tight approximation ratio. 
Another relevant work is by \cite{pedramfar2023unified}, where they investigated the stochastic bandit setting where the online objective functions are randomly sampled from an unknown distribution.

\begin{table}[t]
	\renewcommand\arraystretch{1.35}
	\centering
	\caption{Comparison of regrets for bandit continuous DR-submodular function maximization. `\textbf{Ratio}' means approximation ratio. For simplicity, we set $\gamma=1$ for our results which reduces to the standard monotone DR-submodular setting. \cite{pmlr-v202-wan23e} also proposed the algorithm which achieves $(1-1/e)$-regret of $\widetilde{O}(T^{3/4})$ when the reward function is monotone and constraint is convex set containing $\boldsymbol{0}$. However, their results make an extra assumption that $f_t(\boldsymbol{0})=0$ compared to other results, so we do not compare our result with this result in the table.}
	\vspace{0.5em}
	\resizebox{0.85\textwidth}{!}{
		\setlength{\tabcolsep}{1.0mm}{
			\begin{tabular}{c|c|c|c|c}
				\toprule[1.5pt]
				\textbf{Method}&\textbf{Monotonicity}&\textbf{Constraint} &\textbf{Ratio} & \textbf{Regret} \\
				\hline
				\makecell{Bandit-FW\\\citep{zhang2019online}} &monotone &downward closed & $1-1/e$ & $O(T^{8/9})$ \\
				\hline
                \makecell{\citep{niazadeh2022online}} &monotone &downward closed & $1-1/e$ & $O(T^{5/6})$ \\
				\hline
				\makecell{Bandit-MFW\\ \citep{zhang2023online}}&general &downward closed & $1/e$ & $O(T^{8/9})$\\
				\hline
                \multirow{2}{*}{\makecell{\citep{pedramfar2024unified}}}&general &downward closed & $1/e$ & $O(T^{5/6})$ \\
				&general &convex & $\frac{1-\min_{\bx\in\mathcal{C}}\|\bx\|_{\infty}}{4}$ &$O(T^{5/6})$
				\\
				\hline
				\hline
				\multirow{2}{*}{\makecell{Boosting BGA\\ (\cref{thm:monotone bandit},\cref{thm:nonmonotone bandit})}}&monotone &convex & $1-1/e$ & $O(T^{4/5})$ \\
				&general &convex & $\frac{1-\min_{\bx\in\mathcal{C}}\|\bx\|_{\infty}}{4}$ &$O(T^{4/5})$
				\\
				\midrule[1.5pt]
			\end{tabular}
	}}
	\vspace{-1.5em}
	\label{tab:bandit_online_convergence}
\end{table}

\paragraph{Minimax Optimization of Convex-Submodular Functions} The problem is formulated by \cite{adibi2022minimax} in the form of $\argmin_{\bx\in\mathcal{K}}\max_{S\in \mathcal{I}} f(\bx,S)$. Here $f$ is convex w.r.t. $\bx$ and monotone submodular w.r.t. $S$. They defined the notion of approximation solution to this minimax problem and proposed the algorithms that achieve $(1-1/e)$-approximation when the set system constraint $\mathcal{I} $ is a cardinality constraint and $\frac{1}{2}$-approximation when $\mathcal{I}$ is a general matroid. Besides, they also show a $1-1/e$ approximation ratio upper bound to this problem. We focus on the situation where $\mathcal{I}$ is a general matroid and the submodular part of the objective function is either monotone or general non-monotone. For monotone case, we improve the approximation ratio from $1/2$ to optimal $1-1/e$. As for the non-monotone case, our proposed algorithm achieves $\frac{1}{4}$-approximation guarantee.

\begin{table}[t]
	\renewcommand\arraystretch{1.35}
	\centering
	\caption{Comparison of the convergence guarantee for minimax optimization of convex-submodular function. Note that `\textbf{Mono.}' means the monotonicity of the submodular part of the convex-submodular function. A rigorous definition of the \textbf{Approximation Ratio} refers to \cref{sec:minimax}. `\textbf{Unbounded Grad}' indicates if the corresponding method applies to the situation where the gradient is not uniformly bounded.}
	\vspace{0.5em}
	\resizebox{\textwidth}{!}{
		\setlength{\tabcolsep}{1.0mm}{
			\begin{tabular}{c|c|c|c|c|c|c}
				\toprule[1.5pt]
				\textbf{Method}&\textbf{Mono.}&\textbf{$\mathcal{I}$}&\textbf{Setting} &\makecell{\textbf{Approximation}\\ \textbf{Ratio}} & \textbf{Complexity}&\makecell{\textbf{Unbounded }\\\textbf{Grad.}}\\
				\hline
				\multirow{2}{*}{\makecell{GG\\\citep{adibi2022minimax}}}&mono. &cardinality&det. & $(1-1/e,\epsilon)$ & $O(1/\epsilon^2)$& No\\
				& mono.&matroid &det. & $(1/2,\epsilon)$ & $O(1/\epsilon^2)$ & No\\
				\hline
				\multirow{2}{*}{\makecell{EGG\\\citep{adibi2022minimax}}}&mono. &cardinality &det. & $(1-1/e,\epsilon)$ & $O(1/\epsilon^2)$& Yes\\
				& mono.&matroid & det. &$(1/2,\epsilon)$ & $O(1/\epsilon^2)$ & Yes\\
				\hline
				\makecell{EGCE\\\citep{adibi2022minimax}}&mono. &matroid& det. & $(1/2,\epsilon)$ & $O(1/\epsilon)$ & Yes\\
				\hline
				\hline
				\multirow{2}{*}{\makecell{Boosting GG\\ (\cref{thm:minimax},\cref{thm:minimax nonmonotone})}}&mono. &matroid& sto. & $(1-1/e,\epsilon)$ & $O(1/\epsilon^2)$& No \\
				&general &matroid &sto. & $\left(1/4,\epsilon\right)$ & $O(1/\epsilon^2)$& No\\
				\midrule[1.5pt]
			\end{tabular}
	}}
	\vspace{-1.5em}
	\label{tab:minimax_convergence}
\end{table}

\paragraph{Stationary Points of Continuous Submodular Function} Stationary points are of independent interest because they characterize the fixed points of the classical gradient ascent method \citep{nesterov2013introductory} and Frank-Wolfe algorithm~\citep{lacoste2016convergence}. \cite{hassani2017gradient} first showed that the value of a monotone DR-Submodular function at stationary points is at least $(1/2)$-approximation to the global maximum. As for the non-monotone case, \cite{chen2023continuous} constructed a simple instance whose stationary points can have arbitrarily bad approximation ratios such that there is no approximation guarantee for stationary points of the general DR-submodular function.
\begin{table}[t]
	\renewcommand\arraystretch{1.35}
	\centering
	\caption{Comparison of the approximation guarantee for the solution related to stationary points of different functions. Note that `\textbf{Solution}' indicates the composition of the target point; `\textbf{Approximation Guarantee}' means the ratio between $f(\text{\textbf{Solution}})$ and $\text{OPT}:=\max_{\bx\in\mathcal{C}}f(\bx)$; when $f$ is monotone and DR-submodular, we set $F(\boldsymbol{x})=\int_{0}^{1}\frac{e^{z-1}}{z}f(z\cdot\boldsymbol{x})dz$; as for the non-monotone DR-submodular case, $F(\bx)= \int_0^1 \frac{1}{4z(1-\frac{z}{2})^3}\left(f\left(\frac{z}{2} \cdot (\bx-\underline{\bx})+\underline{\bx}\right) - f(\underline{\bx})\right)dz$ and
 $\underline{\bx}:=\argmin_{\bx\in \mathcal{C}} \|\bx\|_{\infty}$; ‘d.c.’ represents the down-closed convex set.}
	\vspace{0.5em}
	\resizebox{\textwidth}{!}{
		\setlength{\tabcolsep}{1.0mm}{
			\begin{tabular}{c|c|c|c}
				\toprule[1.5pt]
				\textbf{Solution}&\textbf{Mono.}&\textbf{Constraint} &\textbf{Approximation Guarantee} \\
				\hline
	\makecell{Stationary Point on $f$\\\citep{hassani2017gradient}} &mono. &convex & $1/2$ \\
				\hline
				\makecell{Stationary Point on $f$\\ \citep{chen2023continuous}}&general &d.c.& $0$\\
				\hline
				\hline
				\makecell{Stationary Point on $F$\\ (\cref{corollary: monotone stationary point})}& mono. &convex &$1-1/e$\\ \hline
				\multirow{2}{*}{\makecell{Average of Stationary Point on $F$ and $\underline{\bx}$\\(\cref{corollary: nonmonotone stationary point})}}&general &d.c.&$1/4$\\ 
        &general &convex&$\frac{1 -\min_{\bx\in\mathcal{C}}\|\bx\|_{\infty}}{4}$\\ 
				\midrule[1.5pt]
			\end{tabular}
	}}
	\vspace{-1.5em}
	\label{tab:stationary_points}
\end{table}
\paragraph{Non-Oblivious Search} In many cases, classical local search, e.g., the greedy method, may return a solution with a poor approximation ratio to the global maximum. To avoid this issue, \citet{khanna1998syntactic} and \citet{alimonti1994new} first proposed a technique named \textsl{Non-Oblivious Search} that leverages an auxiliary function to guide the search. After carefully choosing the auxiliary function, the new solution generated by the non-oblivious search may have a better performance than the previous solution found by the classical local search. Inspired by this idea, for the maximum coverage problem over a matroid, \citet{filmus2012power} proposed a $(1-1/e)$-approximation algorithm via a non-oblivious set function allocating extra weights to the solutions that cover some element more than once, which efficiently improves the traditional $(1/2)$-approximation greedy method. After that, \citet{filmus2014monotone} extended this idea to improve the $(1/2)$-approximation greedy method for the general submodular set maximization problem over a matroid. Recently, for the continuous submodular maximization problem with concave regularization, a variant of Frank-Wolfe algorithm~(Non-Oblivious FW) based on a special auxiliary function was proposed for boosting the approximation ratio of the submodular part from $1/2$ to $(1-1/e)$ in \citep{mitra2021submodular+}.
Compared to the proposed algorithm in this paper, i) 
The Non-Oblivious Frank-Wolfe method needs $O(1/\epsilon)$ gradient evaluations at each round under the deterministic setting, 
while our method only needs $O(1)$ evaluations per iteration under the stochastic setting;
ii) The Non-Oblivious Frank-Wolfe method is designed only for the deterministic monotone offline setting, while we present a boosting framework covering the stochastic gradient ascent in both monotone and non-monotone cases under several optimization scenarios.


\section{Preliminaries}\label{sec:pre}
In this section, we define some concepts and notations that we will frequently use.
\subsection{Continuous Submodularity}
\noindent{\textbf{Continuous Submodular Functions:}} A function $f:\mathcal{X}\rightarrow \mathbb{R}_{+}$ is a {\it continuous submodular} function if for any $ \boldsymbol{x} ,\boldsymbol{y}\in\mathcal{X}$,
\begin{align*}
    f(\boldsymbol{x})+f(\boldsymbol{y})\ge f(\boldsymbol{x}\land\boldsymbol{y})+f(\boldsymbol{x}\lor\boldsymbol{y}).
\end{align*}
Here, $\boldsymbol{x}\land\boldsymbol{y}=\min(\boldsymbol{x},\boldsymbol{y})$ and $\boldsymbol{x}\lor\boldsymbol{y}=\max(\boldsymbol{x},\boldsymbol{y})$ are component-wise minimum and component-wise maximum, respectively. $\mathcal{X}=\prod_{i=1}^{n}\mathcal{X}_{i}$ where each $\mathcal{X}_{i}$ is a compact interval in $\mathbb{R}_{+}$. Without loss of generality, we assume $\mathcal{X}_{i}=[0,1]$. If $f$ is twice differentiable, the continuous submodularity is equivalent to 
\begin{align*}
    \forall i\neq j,\forall\boldsymbol{x}\in\mathcal{X},\frac{\partial^{2}f(\boldsymbol{x})}{\partial x_{i}\partial x_{j}}\le 0.
\end{align*}

\noindent{\textbf{DR-Submodularity:}} 
A continuous submodular function $f$ is {\it DR-submodular} if
\begin{align*}
    f(\boldsymbol{x}+z\boldsymbol{e}_{i})-f(\boldsymbol{x})\le f(\boldsymbol{y}+z\boldsymbol{e}_{i})-f(\boldsymbol{y}),
\end{align*} where $\boldsymbol{e}_{i}$ is the $i$-th basic vector, $\boldsymbol{x}\ge\boldsymbol{y}$ and $z\in \mathbb{R}_{+}$ such that $\boldsymbol{x}+z\boldsymbol{e}_{i}, \boldsymbol{y}+z\boldsymbol{e}_{i}\in\mathcal{X}$. 
When the DR-submodular function $f$ is differentiable, we have $\nabla f(\boldsymbol{x})\le\nabla f(\boldsymbol{y})$ if $\boldsymbol{x}\ge\boldsymbol{y}$ \citep{bian2020continuous}. When $f$ is twice differentiable, the DR-submodularity is also equivalent to
\begin{align*}
    \forall i,j\in[n],\forall\boldsymbol{x}\in\mathcal{X},\frac{\partial^{2}f(\boldsymbol{x})}{\partial x_{i}\partial x_{j}}\le 0.
\end{align*} 

\noindent{\textbf{Monotonicity:}} We say $f$ is {\it monotone} if $f(\boldsymbol{x})\ge f(\boldsymbol{y})$ when $\boldsymbol{x}\ge\boldsymbol{y}$. Here the inequality of vectors is component-wise.

\noindent{\textbf{Weak DR-submodularity}: }
We call a \emph{monotone} function $f$ {\it \textbf{weakly} DR-submodular} with parameter $\gamma$, if
\begin{align*}
    \gamma=\inf_{\boldsymbol{x}\le\boldsymbol{y}}\inf_{i\in[n]}\frac{[\nabla f(\boldsymbol{x})]_{i}}{[\nabla f(\boldsymbol{y})]_{i}}.
\end{align*} Note that $\gamma=1$ indicates a differentiable and monotone DR-submodular function.
\subsection{Notations and Concepts}
\noindent{\textbf{Norm:}} $\|\cdot\|$ is the $\ell_{2}$ norm in Euclidean space. $\|\cdot\|_{\infty}$ is the $\ell_{\infty}$-norm in Euclidean space.

\noindent{\textbf{Radius and Diameter:}} For any bounded domain $\mathcal{C}\in\mathcal{X}$, the radius $r(\mathcal{C})=\max_{\boldsymbol{x}\in\mathcal{C}} \left\|\boldsymbol{x}\right\|$ and the diameter $\mathrm{diam}(\mathcal{C})=\max_{\boldsymbol{x},\boldsymbol{y}\in\mathcal{C}} \left\|\boldsymbol{x}-\boldsymbol{y}\right\|$.

\noindent{\textbf{Projection:}} We define the projection to the domain $\mathcal{C}$ as $\mathcal{P}_{\mathcal{C}}(\boldsymbol{x})=\arg\min_{\boldsymbol{z}\in\mathcal{C}}\left\|\boldsymbol{x}-\boldsymbol{z}\right\|$.

\noindent{\textbf{Smoothness:}} A differentiable function $f$ is called $L$-$smooth$ if for any $\boldsymbol{x},\boldsymbol{y}\in\mathcal{X}$,
\begin{equation*}
  \left\|\nabla f(\boldsymbol{x})-\nabla f(\boldsymbol{y})\right\|\le L\left\|\boldsymbol{x}-\boldsymbol{y}\right\|.
\end{equation*}

\noindent{\textbf{$\alpha$-Regret:}} When considering the online learning of DR-submodular functions, people usually use $\alpha$-regret~\citep{streeter2008online, chen2018online} to measure the performance of an algorithm. Online learning can be formulated by a $T$-$round$ game between the algorithm and an adversary. Each round, after the algorithm $\mathcal{A}$ chooses an action $\boldsymbol{x}_{t}\in\mathcal{X}$, the adversary reveals the utility function $f_{t}$. The objective of the algorithm $\mathcal{A}$ is to minimize the $\alpha$-regret, namely, the gap between the accumulative reward and that of the best-fixed action in hindsight with scale parameter $\alpha$, i.e.,
\begin{align*}
    \mathcal{R}_{\alpha}(\mathcal{A},T)=\alpha\max_{\boldsymbol{x}\in\mathcal{X}}\sum_{t=1}^{T}f_{t}(\boldsymbol{x})-\sum_{t=1}^{T}f_{t}(\boldsymbol{x}_{t}).
\end{align*}

\section{Derivation of the Non-oblivious Function}\label{sec:non-oblivious}
In this section, we present in detail how to derive non-oblivious functions for both monotone and non-monotone DR-submodular functions, which play an important role in our boosting framework.
To begin, we recall the definition of stationary points.
\begin{definition}\label{def:1}
A point $\boldsymbol{x}\in\mathcal{C}$ is called a stationary point for function $f:\mathcal{X}\rightarrow \mathbb{R}_{+}$ over the domain $\mathcal{C}\subseteq\mathcal{X}$ if
\begin{equation*}
    \max_{\boldsymbol{y}\in\mathcal{C}}\langle\nabla f(\boldsymbol{x}),\boldsymbol{y}-\boldsymbol{x}\rangle\le 0.
\end{equation*}
\end{definition}
\begin{remark}
Stationary points are crucial for the projected gradient ascent(PGA) methods since they characterize the fixed point of the iterative sequence~\citep{nesterov2013introductory}. Generally speaking, if falling into a stationary point $\bx$,  PGA will not make any progress and will be stuck in $\bx$ because any feasible update $\by-\bx$ has a non-positive correlation with the steepest direction $\nabla f(\bx)$. 
\end{remark}

To our regret, some stationary points of general DR-submodular objectives can only provide a conservative approximation guarantee or even behave extremely badly to the global maxima. To circumvent these issues, we present a boosting technique to avoid these bad stationary points.

\subsection{Non-oblivious Function for Monotone DR-Submodular Function}

We make the following assumptions throughout this paper when we are considering monotone objectives.

\begin{assumption}\label{assumption1}\
\begin{enumerate}
   \item[(i)] The $f:\mathcal{X}\rightarrow \mathbb{R}_{+}$ is a monotone, differentiable, weakly DR-submodular function with parameter $\gamma$. So is each $f_{t}$ in the online settings..
   \item[(ii)] We also assume the knowledge of parameter $\gamma$.
    \item[(iii)] Without loss of generality, $f(\boldsymbol{0})=0$. Also, in online settings,  $f_{t}(\boldsymbol{0})=0$ for $t=1,2,\dots,T$.
\end{enumerate}
\end{assumption} 
With this assumption, we have the following result. 
\begin{lemma}[\cite{hassani2017gradient}]\label{lemma:2} 
If $f$ is a differentiable monotone $\gamma$-weakly DR-submodular function, then for any stationary point $\boldsymbol{x}\in\mathcal{C}$ of $f$, we have 
\begin{equation}
    f(\boldsymbol{x})\ge\frac{\gamma^{2}}{\gamma^{2}+1}\max_{\boldsymbol{y}\in\mathcal{C}}f(\boldsymbol{y}).
\end{equation}
\end{lemma}
We provide the proof of this lemma in \cref{proof:lem1}.
\begin{remark}
The ratio $\frac{\gamma^{2}}{1+\gamma^{2}}$-approximation guarantee is tight for the stationary points of $f$ itself, since a simple DR-submodular instance with a $(1/2+\epsilon)$-approximation local maximum is presented in \citet{hassani2017gradient} for any $\epsilon>0$. As a result, Lemma~\ref{lemma:2} implies that any stationary point of a $\gamma$-weakly DR-submodular function $f$ provides a
$(\frac{\gamma^{2}}{1+\gamma^{2}})$-approximation to the global maximum. 
\end{remark}

As far as we know, projected gradient ascent method~\citep{hassani2017gradient} with small step size usually converges to a stationary point of $f$, resulting in a limited $(\frac{\gamma^{2}}{1+\gamma^{2}})$ approximation guarantee. In order to boost this algorithm, a natural idea is to design some auxiliary functions whose stationary points achieve better approximation to the global maximum.
To be specific, we want to find $F:\mathcal{X}\rightarrow \mathbb{R}_{+}$ based on $f$ such that $\langle\boldsymbol{y}-\boldsymbol{x}, \nabla F(\boldsymbol{x})\rangle\ge \beta_{1}f(\boldsymbol{y})-\beta_{2}f(\boldsymbol{x})$,  
where $\beta_{1}/\beta_{2}\ge\frac{\gamma^{2}}{1+\gamma^{2}}$.

Motivated by \citep{feldman2011unified,filmus2012power,filmus2014monotone,harshaw2019submodular,feldman2021guess,mitra2021submodular+}, we consider the function $F(\boldsymbol{x}):\mathcal{X}\rightarrow \mathbb{R}_{+}$ whose gradient at point $\boldsymbol{x}$ allocates different weights to the gradient $\nabla f(z\cdot\boldsymbol{x})$, i.e.,
$\nabla F(\boldsymbol{x})=\int_{0}^{1} w(z)\nabla f(z\cdot\boldsymbol{x})\mathrm{d}z$, assuming that $\nabla f(z\cdot\boldsymbol{x})$ is Lebesgue integrable w.r.t. $z\in[0,1]$, the weight function $w(z)\in C^{1}[0,1]$, and $w(z)\ge 0$. Then, we investigate a property of $\langle \boldsymbol{y}-\boldsymbol{x}, \nabla F(\boldsymbol{x})\rangle$ in the following lemma.

\begin{lemma}[Proof in \cref{proof:lem2}]\label{lemma:3}
For all $\boldsymbol{x},\boldsymbol{y} \in \mathcal{X}$, we have
    \begin{align*}
     \langle \boldsymbol{y}-\boldsymbol{x}, \nabla F(\boldsymbol{x})\rangle \ge\left(\gamma\int_{0}^{1}w(z)\mathrm{d}z\right)\left(f(\boldsymbol{y})-\theta(w)f(\boldsymbol{x})\right),
\end{align*}
where $\theta(w)=\max_{f,\boldsymbol{x}}\theta(w,f,\boldsymbol{x})$, $\theta(w,f,\boldsymbol{x})=\frac{w(1)+\int^{1}_{0} (\gamma w(z)-w'(z))\frac{f(z\cdot\boldsymbol{x})}{f(\boldsymbol{x})}\mathrm{d}z}{\gamma\int_{0}^{1}w(z)\mathrm{d}z}$ for any $f(\boldsymbol{x})>0$. 
\end{lemma}

Fixing a weight function $w(z)$, Lemma~\ref{lemma:3} indicates that the stationary points of auxiliary function $F$ achieve at least $\frac{1}{\theta(w)}$-approximation guarantee. To maximize the approximation ratio, we consider the following factor-revealing optimization problem:
\begin{equation}\label{frop}
   \begin{aligned}
    \min_{w}\theta(w)=\min_{w}&\max_{f,\boldsymbol{x}} \frac{w(1)+\int^{1}_{0} (\gamma w(z)-w'(z))\frac{f(z\cdot\boldsymbol{x})}{f(\boldsymbol{x})}\mathrm{d}z}{\gamma\int_{0}^{1}w(z)\mathrm{d}z}\\
          \rm{s.t.} \ &w(z)\ge 0,\\
          &w(z)\in C^{1}[0,1],\\
          &f(\boldsymbol{x})>0, \\
          &\nabla f(\boldsymbol{x}_1)\ge\gamma\nabla f(\boldsymbol{y}_1)\ge\boldsymbol{0},  \forall \boldsymbol{x}_1\le \boldsymbol{y}_1\in\mathcal{X}.
    \end{aligned} 
\end{equation}
At first glance, problem~\eqref{frop} looks challenging to solve. Fortunately, we could directly find the optimal solution, which is provided in the following theorem.  
\begin{theorem}[Proof in the \cref{proof:thm1}]\label{thm:1} For problem~\eqref{frop}, we have $\hat{w}(z)=e^{\gamma(z-1)}\in\arg\min_{w}\theta(w)$ and $\min_{w}\max_{f,\boldsymbol{x}}\theta(w,f,\boldsymbol{x})=\frac{1}{1-e^{-\gamma}}$. 
\end{theorem}

In the following sections, we consider this optimal auxiliary function $F$ with $\nabla F(\boldsymbol{x})=\int_{0}^{1}\hat{w}(z)\nabla f(z\cdot\boldsymbol{x})\mathrm{d}z$, and $\hat{w}(z)=e^{\gamma(z-1)}$. According to the definition of $\theta(w,f,\boldsymbol{x})$ in \cref{lemma:3}, we could derive that $\theta(\hat{w},f,\boldsymbol{x})=\hat{w}(1)/(\gamma\int_{0}^{1}\hat{w}(z)\mathrm{d}z)=1/(1-e^{-\gamma})$ such that we have $\langle\boldsymbol{y}-\boldsymbol{x}, \nabla F(\boldsymbol{x})\rangle \ge (1-e^{-\gamma})f(\boldsymbol{y})-f(\boldsymbol{x})$ which immediately implies the following
corollary.
\begin{corollary}\label{corollary: monotone stationary point}
    Let $F$ be defined by its gradient $\nabla F(\bx) = \int_0^1 e^{\gamma(z-1)}\nabla f(z\cdot\boldsymbol{x})\mathrm{d}z$,  then for any $\bx,\by\in \mathcal{X}$, we have
    \begin{equation*}
        \langle \nabla F(\bx), \by-\bx\rangle \geq \left(1-e^{-\gamma}\right)f(\by)-f(\bx).
    \end{equation*}
    As a result, If $\bx$ is a stationary point of $F$, then $\bx$ is a $(1-e^{-\gamma})$-approximation solution to the original monotone  $\gamma$-weakly DR-submodular objective $f$.
\end{corollary}
\begin{remark}
\cref{corollary: monotone stationary point} sheds light on the possibility of utilizing $F$ to obtain a better approximation than the classical gradient ascent method, which motivates our boosting methods in the following sections.
\end{remark}

Next, we investigate some properties of this optimal auxiliary function $F(\boldsymbol{x})$. Following the same terminology in \citep{filmus2012power,filmus2014monotone,mitra2021submodular+}, we also call this $F$ the {\it Non-Oblivious Function}. 
\subsubsection{Properties about the Non-Oblivious Function of Monotone Case}
The following theorem establishes some key properties about the boundness and smoothness of the non-oblivious function $F(\boldsymbol{x})$.

\begin{theorem}[Proof in \cref{proof:thm2}]\label{thm:2} 
If $f$ is $L$-smooth, $L_1$-lipschitz continuous, and Assumption~\ref{assumption1} holds, we have 
\begin{enumerate}
\item[(i)] $F$ is well-defined and $F(\boldsymbol{x})=\int_{0}^{1}\frac{e^{\gamma(z-1)}}{z}f(z\cdot\boldsymbol{x})\mathrm{d}z$. Moreover, $F(\boldsymbol{x})\le(1+\ln(\tau))(f(\boldsymbol{x})+c)$ for any positive $c\le Lr^{2}(\mathcal{X})$, where $\tau=\max(\frac{1}{\gamma},\frac{Lr^{2}(\mathcal{X})}{c})$. 
\item [(ii)]
$F$ is $L_{\gamma}$-smooth and $\frac{1-e^{-\gamma}}{\gamma}L_1$-lipschitz continuous where $L_{\gamma}=L\frac{\gamma+e^{-\gamma}-1}{\gamma^2}$.
\end{enumerate}
\end{theorem}
\begin{remark}
Note that the integral $\int_{0}^{1}\frac{e^{\gamma(z-1)}}{z}f(z\cdot\boldsymbol{x})\mathrm{d}z$ in \cref{thm:2}.(i) is well-defined when $f(\boldsymbol{0})=0($Assumption~\ref{assumption1}$)$ and the limit $\lim_{\bx\rightarrow\boldsymbol{0}^{+}}\nabla f(\bx)$ exists. If $f(\boldsymbol{0})\neq0$, we can re-define $f(\bx):=f(\bx)-f(\boldsymbol{0})$ due to the monotone assumption. Furthermore, the existence of $\lim_{\bx\rightarrow\boldsymbol{0}^{+}}\nabla f(\bx)$ follows from the $L$-smoothness and the monotone of $\nabla f(\bx)$.
\end{remark}

Previously, \citet{filmus2014monotone} designed an auxiliary discrete function to improve the standard $1/2$-approximation greedy method for the submodular set maximization problem over a matroid. Next, we unveil the connection between our proposed non-oblivious function and the auxiliary discrete function in \citet{filmus2014monotone}. Roughly speaking, they considered a monotone submodular set function $\bar{f}:2^{\Omega}\rightarrow\mathbf{R}_{+}$ and defined its related auxiliary set function as $\bar{g}(A)=\sum_{B\subset A}m_{|A|-1,|B|-1}\bar{f}(B)$ for any $A\subseteq\Omega$ where $\Omega=\{1,2,\dots,n\}$ and $m_{a,b}=\int_{0}^{1}\frac{e^{p}}{e-1}p^{b}(1-p)^{a-b}\mathrm{d}p$. To maximize $\bar{f}$ over a matroid, \citet{filmus2014monotone} provided an improved greedy method based on $\bar{g}$ instead of the original objective $\bar{f}$ itself with the optimal $(1-1/e)$-approximation ratio. Throughout the multi-linear relaxation~\citep{calinescu2011maximizing}, we can obtain a corresponding monotone continuous DR-submodular function $\bar{F}(\boldsymbol{x})=\sum_{S\in 2^{\Omega}}\bar{f}(S)\prod_{i\in S}x_{i}\prod_{j\in\Omega\setminus S}(1-x_{j})$ where $\bx\in[0,1]^{n}$. If taking the same boosting policy for $\bar{F}$, we could obtain a non-oblivious function $\bar{G}(\boldsymbol{x})=\int_{0}^{1}\frac{e^{z-1}}{z}\bar{F}(z\cdot\boldsymbol{x})\mathrm{d}z$ from Theorem~\ref{thm:2}.(i). After careful reformulations, we could find 
\begin{theorem}\label{qx_add_thm:1} 
If ignoring a constant factor $\frac{e-1}{e}$, we could regard the non-oblivious function $\bar{G}$ as the multi-linear extension of the submodular set function $\bar{g}$.
\end{theorem}
\begin{proof}
Firstly, for any $S\subseteq\Omega$ and $\bx\in[0,1]^{n}$, we set $\triangle(S,\boldsymbol{x})=\prod_{i\in S}x_{i}\prod_{j\in\Omega\setminus S}(1-x_{j})$. Then,
\begin{align*}
     &\bar{G}(\boldsymbol{x})=\int_{0}^{1}\frac{e^{z-1}}{z}\bar{F}(z\cdot\boldsymbol{x})\mathrm{d}z\\
        &=\int_{0}^{1}\frac{e^{z-1}}{z}\sum_{S\in 2^{\Omega}}\bar{f}(S)\prod_{i\in S}z\cdot x_{i}\prod_{j\in\Omega\setminus S}(1-z\cdot x_{j})\mathrm{d}z\\
        &=\sum_{S\in 2^{\Omega}}\bar{f}(S)\int_{0}^{1}e^{z-1}z^{|S|-1}\prod_{i\in S}x_{i}\prod_{j\in\Omega\setminus S}(1-z\cdot x_{j})\mathrm{d}z\\
        &=\sum_{S\in 2^{\Omega}}\bar{f}(S)\int_{0}^{1}e^{z-1}z^{|S|-1}\prod_{i\in S}x_{i}\prod_{j\in\Omega\setminus S}(1-x_{j}+x_{j}(1-z))\mathrm{d}z\\
        &=\sum_{S\in 2^{\Omega}}\sum_{K\subset\Omega\setminus S}\bar{f}(S)\int_{0}^{1}e^{z-1}z^{|S|-1}(1-z)^{|K|}\triangle(S\cup K,\boldsymbol{x})\mathrm{d}z\\
        &=\sum_{M\in 2^{\Omega}}\sum_{S\subset M}\bar{f}(S)\int_{0}^{1}e^{z-1}z^{|S|-1}(1-z)^{|M|-|S|}\triangle(M,\boldsymbol{x})\mathrm{d}z\\
         &=\sum_{M\in 2^{\Omega}}\sum_{S\subset M}\bar{f}(S)\frac{(e-1)m_{|M|-1,|S|-1}}{e}\triangle(M,\boldsymbol{x})\\
         &=\frac{e-1}{e}\sum_{M\in 2^{\Omega}}\triangle(M,\boldsymbol{x})\bar{g}(M)\\
         &= \frac{e-1}{e}\sum_{M\in 2^{\Omega}}\bar{g}(M)\prod_{i\in M}x_{i}\prod_{j\in\Omega\setminus M}(1-x_{j}).
\end{align*}
\end{proof}
\subsection{Non-oblivious Function for Non-monotone DR-Submodular Function}


Notably, the monotonicity of DR-submodular objectives plays an indispensable role in deriving the previous auxiliary function. However, a large body of real-world applications can be cast into non-monotone DR-submodular maximization problems, such as the Determinantal Point Processes~\citep{kulesza2012determinantal,bian2017continuous} and Revenue Maximization~\citep{bian2017guaranteed}, which motivates our curiosity on how to design a non-oblivious function
for non-monotone counterparts to avoid the bad stationary points.

Before going into the detail, we first recall the results about the stationary points of general non-monotone DR-submodular maximization, i.e., 
\begin{lemma}[\cite{chen2023continuous}]\label{qx_add_lemma:1}
For any $\epsilon>0$, there exists a general continuous DR-submodular function $f$, whose ratio $\frac{f(\bx)}{f(\bx^{\star})}$ is not greater than $\epsilon$ where $\bx$ is the worst stationary point of $f$ itself and $\bx^*$ is the maximum solution over a convex set. 
\end{lemma}
\begin{remark} Lemma~\ref{qx_add_lemma:1} implies that a stationary point of a non-monotone continuous DR-submodular function may be arbitrarily bad. Similarly, there is no approximation guarantee for the standard PGA on a general continuous DR-submodular function $f$, since it may approach a bad stationary point of $f$.
\end{remark}

To avoid these bad stationary points, we also hope to design an auxiliary function $F$ whose stationary points can provide a significant approximation guarantee for the continuous non-monotone DR-submodular function $f$. We first specify some assumptions about the objective function $f$.
\begin{assumption}\label{assumption2}The $f:\X\rightarrow \R_+$ is differentiable and DR-submodular. So is each $f_t$ in the online settings.
\end{assumption}
\begin{remark}
    We do not assume that $f(\boldsymbol{0}) = 0$ here since the reformulation $f(\bx):=f(\bx)-f(\boldsymbol{0})$ may violate the non-negative assumption about the objective function, when $f$ is non-monotone. However, this modification can be done without violating any assumption when we consider the monotone functions.
\end{remark}
Let $\underline{\bx}:=\argmin_{\bx\in \mathcal{C}} \|\bx\|_{\infty}$. Different from the monotone case, we consider a new form of the non-oblivious function $F(\bx): \X\rightarrow \R_+$ whose gradient at $\bx$ allocation different weights to the gradient $\nabla f(z\alpha \cdot\bx + (1-z\alpha)\underline{\bx})$ for $z\in [0,1]$. Here $\alpha\in [0,1]$ is a parameter to be determined. Rigorously, $\nabla F(\bx) = \int_0^1 \omega(z) \nabla f(z\alpha\cdot\bx+(1-z\alpha)\underline{\bx}) \mathrm{d}z $. Then, we show the following property of $\langle \by-\bx, \nabla F(\bx)\rangle$.

\begin{lemma}[Proof in \cref{proof:nonmontone auxilliary}]\label{nonmonotone auxiliary}
    For all $\by,\bx\in\X$, we have,
    \begin{align}
        \langle \by-\bx, \nabla F(\bx)\rangle \geq \left((1-\|\underline{\bx}\|_{\infty})\int_0^1 (1-\alpha z)\omega(z) \mathrm{d}z\right)\left(f(\by)- \theta(\omega)f(\alpha\cdot\bx+ (1-\alpha)\cdot\underline{\bx})\right),
    \end{align}
    where $\theta(\omega) = \max_{f,\bx} \theta(\omega, f, \bx)$ and
    \begin{equation}
        \theta(\omega,f,\bx) = \frac{\frac{(1-\alpha)\omega(1)}{\alpha}+\int_0^1\left(3\omega(z)-\frac{1-\alpha z}{\alpha}\omega'(z)\right)\frac{f(\alpha z\cdot\bx+(1-\alpha z)\cdot\underline{\bx})}{f(z\cdot\bx + (1-z)\cdot\underline{\bx})} \mathrm{d}z}{(1-\|\underline{\bx}\|_{\infty})\int_0^1(1-\alpha z)\omega(z)\mathrm{d}z}.
    \end{equation}
\end{lemma}

Let $\omega(z)$ satisfy $3\omega(z) = \frac{1-\alpha z}{\alpha}\omega'(z)$. The solution of this ODE is 
\begin{align}\label{omega}
    \omega(z) =  \frac{C}{(1-\alpha z)^{3}},
\end{align}
where $C$ is an arbitrary constant. Then 
\begin{equation}
    \begin{aligned}
        \theta(\omega,f,\bx) &= \frac{(1-\alpha) \omega(1)}{\alpha (1-\|\underline{\bx}\|_{\infty})\int_0^1(1-\alpha z) \omega(z) \mathrm{d}z}
        \\&=\frac{1}{(1-\|\underline{\bx}\|_{\infty})\alpha(1-\alpha)}.
    \end{aligned}
\end{equation}
Note that $\theta(\omega,f,\bx)$ is independent of $f$ and $\bx$, thus $\theta(\omega) = \frac{1}{(1-\|\underline{\bx}\|_{\infty})\alpha(1-\alpha)}$.
The minimum value of $\theta(\omega)$ is attained at $\alpha=\frac{1}{2}$ which leads to the corresponding $\theta(\omega)= \frac{4}{1-\|\underline{\bx}\|_{\infty}}$. Furthermore, to make our analysis of the subsequent section more concise, we let $C = \frac{1}{8}$ in Eq.\eqref{omega}. Therefore, our choice of the weighting function is $\omega(z) = \frac{1}{8(1-\frac{z}{2})^3}$. The above argument immediately implies the following corollary.
\begin{corollary}\label{corollary: nonmonotone stationary point}
    Let $F$ be defined by its gradient $\nabla F(\bx) = \int_0^1\frac{1}{8(1-\frac{z}{2})^3}\nabla f\left(\frac{z}{2}(\bx-\underline{\bx})+\underline{\bx}\right) \mathrm{d}z$,
    then for any $\bx,\by\in \mathcal{X}$, 
    \begin{equation*}
        \langle \nabla F(\bx), \by-\bx\rangle \geq \frac{1-\|\underline{\bx}\|_{\infty}}{4}f(\by)-f\left(\frac{\bx+\underline{\bx}}{2}\right).
    \end{equation*}
    As a result, if $\bx$ is a stationary point of $F$ over convex domain $\mathcal{C}$, 
    then $\frac{\bx+\underline{\bx}}{2}$ is a $\frac{1-\|\underline{\bx}\|_{\infty}}{4}$-approximation solution to the maximum value $\max_{\by\in\mathcal{C}}f(\by)$.
\end{corollary}
\begin{remark}
    Same as the monotone case, \cref{corollary: nonmonotone stationary point} shows the possibility of boosting the projected gradient ascent to find a $\frac{1-\|\underline{\bx}\|_{\infty}}{4}$-approximation solution of a non-monotone DR-submodular function. We will prove this in the subsequent section. Since $\frac{1-\|\underline{\bx}\|_{\infty}}{4}$ is the optimal approximation ratio when maximizing the non-monotone DR-submodular function over a general convex set constraint if one assumes $P\neq NP$ \citep{mualem2023resolving}. Thus, the weight function we choose is optimal unless $P=NP$.
\end{remark}
\subsubsection{Properties about the Non-Oblivious Function of Non-monotone Case}
Like the monotone case, we also care about the properties of the auxiliary function $F(\bx)$ satisfying \cref{corollary: nonmonotone stationary point}. 
The following theorem establishes its boundness and smoothness.
\begin{theorem}[Proof in \cref{proof: nonmontone nonoblivious properties}]\label{nonmonotone nonoblivious properties}
    If $f$ is $L$-smooth, $L_1$-lipschitz and $f$ satisfies Assumption \ref{assumption2}. Let $F$ be defined according to \cref{corollary: nonmonotone stationary point}, then the following holds.
    \begin{itemize}
        \item[(i)] F is well defined and $F(\bx)= \int_0^1 \frac{1}{4 z(1-\frac{z}{2})^3}\left(f\left(\frac{z}{2} \cdot (\bx-\underline{\bx})+\underline{\bx}\right) - f(\underline{\bx})\right)\mathrm{d}z$. 
        \item[(ii)] $F(\bx)$ is $\frac{1}{8}L$-smooth and $\frac{3}{8}L_1$-lipschitz continuous.
    \end{itemize}
\end{theorem}
\begin{remark}
From Theorem~\ref{nonmonotone nonoblivious properties}.(i), the non-oblivious function of the non-monotone case is not only related to objective $f$ itself but relies on the selection of constraint set $\mathcal{C}$. Note that the difference with $f(\underline{\bx})$ makes sure the integral $\int_0^1 \frac{1}{4 z(1-\frac{z}{2})^3}\left(f\left(\frac{z}{2} \cdot (\bx-\underline{\bx})+\underline{\bx}\right) - f(\underline{\bx})\right)\mathrm{d}z$ is well-defined if $f$ satisfies Assumption~\ref{assumption2}.
\end{remark}

In the monotone case, Theorem~\ref{qx_add_thm:1} shows that we could view our proposed auxiliary function as a multi-linear extension of the non-oblivious set function in \citep{filmus2014monotone} when $f$ is a multi-linear extension of a set function. Next, we verify a similar result for the non-monotone case. 

\begin{theorem}\label{qx_add_thm:2} 
Considering a submodular set function $\bar{f}:2^{\Omega}\rightarrow\mathbf{R}_{+}$ and its multi-linear extension $\bar{F}(\boldsymbol{x})=\sum_{S\in 2^{\Omega}}\bar{f}(S)\prod_{i\in S}x_{i}\prod_{j\in\Omega\setminus S}(1-x_{j})$ where $\Omega=\{1,2,\dots,n\}$, if we set $\underline{\bx}=\boldsymbol{0}$, we could verify the non-oblivious function of continuous DR-submodular function $\bar{F}$ is the multi-linear extension of set function $\bar{g}(A)=\sum_{B\subset A}m_{|A|-1,|B|-1}\left(\bar{f}(B)-\bar{f}(\varnothing)\right)$ for any $A\subseteq\Omega$ where $m_{a,b}=\frac{1}{8}\int_{0}^{1}(\frac{p}{2})^{b}(1-\frac{p}{2})^{a-b-3}\mathrm{d}p$. 
\end{theorem}
\begin{proof}
First, for any $S\subseteq\Omega$ and $\bx\in[0,1]^{n}$, we set $\triangle(S,\boldsymbol{x})=\prod_{i\in S}x_{i}\prod_{j\in\Omega\setminus S}(1-x_{j})$. Also, we use the symbol $\bar{G}$ to represent the non-oblivious function of $\bar{F}$.
Then, according to Theorem~\ref{nonmonotone nonoblivious properties}.(i), we have
\begin{align*}
    &\bar{G}(\boldsymbol{x})=\int_{0}^{1}\frac{1}{4z(1-\frac{z}{2})^3}\left(\bar{F}(\frac{z}{2}\cdot\boldsymbol{x})-\bar{F}(\boldsymbol{0})\right)\mathrm{d}z\\
        &=\int_{0}^{1}\frac{1}{4z(1-\frac{z}{2})^3}\left(\sum_{S\in 2^{\Omega}}\bar{f}(S)\prod_{i\in S}\frac{z}{2}\cdot x_{i}\prod_{j\in\Omega\setminus S}(1-\frac{z}{2}\cdot x_{j}) - \bar{f}(\varnothing)\right)\mathrm{d}z\\
        &=\frac{1}{8}\sum_{S\in 2^{\Omega}}\left(\bar{f}(S) -\bar{f}(\varnothing)\right)\int_{0}^{1}(\frac{z}{2})^{|S|-1}(1-\frac{z}{2})^{-3}\prod_{i\in S}x_{i}\prod_{j\in\Omega\setminus S}(1-\frac{z}{2}\cdot x_{j})\mathrm{d}z\\
        &=\frac{1}{8}\sum_{S\in 2^{\Omega}}\left(\bar{f}(S) -\bar{f}(\varnothing)\right)\int_{0}^{1}(\frac{z}{2})^{|S|-1}(1-\frac{z}{2})^{-3}\prod_{i\in S}x_{i}\prod_{j\in\Omega\setminus S}(1-x_{j}+x_{j}(1-\frac{z}{2}))\mathrm{d}z\\
        &=\frac{1}{8}\sum_{S\in 2^{\Omega}}\sum_{K\subset\Omega\setminus S}\left(\bar{f}(S) -\bar{f}(\varnothing)\right)\int_{0}^{1}(\frac{z}{2})^{|S|-1}(1-\frac{z}{2})^{|K|-3}\triangle(S\cup K,\boldsymbol{x})\mathrm{d}z\\
        &=\frac{1}{8}\sum_{M\in 2^{\Omega}}\sum_{S\subset M}\left(\bar{f}(S) -\bar{f}(\varnothing)\right)\int_{0}^{1}(\frac{z}{2})^{|S|-1}(1-\frac{z}{2})^{|M|-|S|-3}\triangle(M,\boldsymbol{x})\mathrm{d}z\\
        &=\sum_{M\in 2^{\Omega}}\sum_{S\subset M}\left(\bar{f}(S) -\bar{f}(\varnothing)\right)m_{|M|-1,|S|-1}\triangle(M,\boldsymbol{x})\\
         &= \sum_{M\in 2^{\Omega}}\bar{g}(M)\prod_{i\in M}x_{i}\prod_{j\in\Omega\setminus M}(1-x_{j}).
\end{align*}

\end{proof}
\subsection{Unbiased Gradient Estimator of Non-oblivious Functions}\label{sec:unbiased}
In this subsection, we address the urgent problem: Given the gradient oracle of the original objective, how do we construct an unbiased estimator to the gradient of the corresponding non-oblivious function? For the sake of generality, we assume that we have access to an unbiased stochastic oracle $\widetilde{\nabla}f(\boldsymbol{x})$, i.e., $\mathbb{E}(\widetilde{\nabla}f(\boldsymbol{x})|\boldsymbol{x})=\nabla f(\boldsymbol{x})$. 
We first introduce two random variables $\mathbf{Z}_{\uparrow}$ and $\mathbf{Z}_{\sim}$ where Pr$(\mathbf{Z}_{\uparrow}\le z)=\int_{0}^{z}\frac{\gamma e^{\gamma(u-1)}}{1-e^{-\gamma}}\mathrm{d}u$ and $\Pr(\mathbf{Z}_{\sim}\leq z) = \int_0^z \frac{1}{3(1-\frac{u}{2})^3}\mathrm{d}u$. 

When the number $z$ is sampled from r.v. $\mathbf{Z}_{\uparrow}$, we consider $\frac{1-e^{-\gamma}}{\gamma}\widetilde{\nabla}f(z\cdot\boldsymbol{x})$ as an estimator of $\nabla F_{\uparrow}(\boldsymbol{x}):= \int_0^1 e^{\gamma(z-1)}\nabla f(z\cdot\bx) \mathrm{d}z$ with statistical properties given in the following proposition. 

\begin{proposition}[Proof in \cref{proof:prop1}]\label{prop:1} \
\begin{enumerate}
    \item[(i)] If $z$ is sampled from r.v. $\mathbf{Z}_{\uparrow}$ and $\mathbb{E}(\widetilde{\nabla}f(\boldsymbol{x})|\boldsymbol{x})=\nabla f(\boldsymbol{x})$, we have 
    \begin{equation*}
        \mathbb{E}\left(\left.\frac{1-e^{-\gamma}}{\gamma}\widetilde{\nabla}f(z\cdot\boldsymbol{x})\right|\boldsymbol{x}\right)=\nabla F_{\uparrow}(\boldsymbol{x}).
    \end{equation*}
    \item[(ii)] If $z$ is sampled from r.v. $\mathbf{Z}_{\uparrow}$, $\mathbb{E}(\widetilde{\nabla}f(\boldsymbol{x})|\boldsymbol{x})=\nabla f(\boldsymbol{x})$, and $\mathbb{E}(\|\widetilde{\nabla}f(\boldsymbol{x})-\nabla f(\boldsymbol{x})\|^{2}|\boldsymbol{x})\le\sigma^{2}$, we have
       \begin{equation*}
        \mathbb{E}\left(\bigg\|\frac{1-e^{-\gamma}}{\gamma}\widetilde{\nabla}f(z\cdot\boldsymbol{x})-\nabla F_{\uparrow}(\boldsymbol{x})\bigg\|^{2}\bigg|\boldsymbol{x}\right)\le\sigma^{2}_{\gamma},
        \end{equation*}
     where $\sigma^{2}_{\gamma}=2\frac{(1-e^{-\gamma})^{2}\sigma^{2}}{\gamma^{2}}+\frac{2L^{2}r^{2}(\mathcal{X})(1-e^{-2\gamma})}{3\gamma}$.
\end{enumerate}
\end{proposition}

\cref{prop:1} indicates that $\frac{1-e^{-\gamma}}{\gamma}\widetilde{\nabla}f(z\cdot\boldsymbol{x})$ is an unbiased estimator of $\nabla F_{\uparrow}(\boldsymbol{x})$ with a bounded variance. Similarly, we can sample number $z$ from r.v. $\mathbf{Z}_{\sim}$, and consider $\frac{3}{8}\widetilde{\nabla} f\left(\frac{z}{2} (\bx-\underline{\bx}) + \underline{\bx}\right)$ as an estimator of $\nabla F_{\sim}(\bx):= \int_0^1 \frac{1}{8(1-\frac{z}{2})^3}\nabla f\left(\frac{z}{2}(\bx-\underline{\bx}) + \underline{\bx}\right)\mathrm{d}z$. It also has bounded variances, as we proved in the following proposition.

\begin{proposition}[Proof in \cref{proof:prop2}]\label{prop:2} \
    \begin{itemize}
    \item[(i)] If $z$ is sampled from r.v. $\mathbf{Z}_{\sim}$ and $\mathbb{E}(\widetilde{\nabla}f(\boldsymbol{x})|\boldsymbol{x})=\nabla f(\boldsymbol{x})$, we have 
    \begin{equation*}
        \mathbb{E}\left(\left.\frac{3}{8}\widetilde{\nabla}f\left(\frac{z}{2}(\bx-\underline{\bx}) + \underline{\bx}\right)\right|\boldsymbol{x}\right)=\nabla F_{\sim}(\boldsymbol{x}).
    \end{equation*}
    \item[(ii)] If $z$ is sampled from r.v. $\mathbf{Z}_{\sim}$, $\mathbb{E}(\widetilde{\nabla}f(\boldsymbol{x})|\boldsymbol{x})=\nabla f(\boldsymbol{x})$, and $\mathbb{E}(\|\widetilde{\nabla}f(\boldsymbol{x})-\nabla f(\boldsymbol{x})\|^{2}|\boldsymbol{x})\le\sigma^{2}$, we have
       \begin{equation*}
        \mathbb{E}\left(\bigg\|\widetilde{\nabla}f\left(\frac{z}{2}(\bx-\underline{\bx}) + \underline{\bx}\right)-\nabla F_{\sim}(\boldsymbol{x})\bigg\|^{2}\bigg|\boldsymbol{x}\right)\le\frac{3}{8}\sigma^2 + \frac{\ln (64)-4}{12} L^2 \mathrm{diam}^2(\X).
        \end{equation*}
    \end{itemize}
\end{proposition}



\section{Applications}\label{sec:applications}
The non-oblivious function we designed in \cref{sec:non-oblivious} makes it possible to boost the gradient ascent(PGA) methods for several DR-submodular function-related optimization problems. We select four scenarios to explain how to boost the gradient methods via non-oblivious functions. They are offline stochastic optimization, online optimization of DR-submodular function, bandit optimization, and minimax optimization of convex-submodular function.

Before we investigate these problems separately, we point out that the core ideas of all these applications are the same. That is, we leverage the stochastic gradient estimator of $F_{\uparrow}$ or $F_{\sim}$(depends on the monotonicity of the objective function) when executing PGA, instead of the stochastic gradient $\widetilde{\nabla}f$ of the original DR-submodular function $f$. 

In the following sections, we use the symbol $\widetilde{\nabla} F(\bx)$ to denote the gradient estimates of both $\nabla F_{\uparrow}$ and $\nabla F_{\sim}$ in Section~\ref{sec:unbiased}. These two estimators are generated by the sampling method described in~\cref{prop:1} and \cref{prop:2}.

\subsection{Offline Optimization}\label{sec:gradient_ascent}
\begin{algorithm}[t]
	\caption{Boosting Gradient Ascent}\label{alg:1}
	\hspace*{0.02in} {\bf Input:} $T$, $\eta_{t}$ , $\gamma$, $L$, $r(\mathcal{X})$
	\begin{algorithmic}[1]
		\STATE \textbf{Initialize} any  $\boldsymbol{x}_{1}\in\mathcal{C}$.
		\FOR{$t\in [T]$}
		\STATE
		Option I~(monotone): Sample $z_{t}$ from $\mathbf{Z}_{\uparrow}$ and set $\widetilde{\nabla} F(\bx_t) = \frac{1-e^{-\gamma}}{\gamma} \widetilde{\nabla} f(z_t\cdot\bx_t)$
		\STATE
		Option II~(non-monotone): Sample $z_{t}$ from $\mathbf{Z}_{\sim}$ and set $\widetilde{\nabla} F(\bx_t) = \frac{3}{8}\widetilde{\nabla} f(\frac{z_t}{2}\cdot\bx_t + (1-\frac{z_t}{2})\cdot\underline{\bx})$
		\STATE Set $\boldsymbol{y}_{t+1}=\boldsymbol{x}_{t}+\eta_{t}\widetilde{\nabla}F(\boldsymbol{x}_{t})$
		\STATE $\boldsymbol{x}_{t+1}\gets \mathcal{P}_{\mathcal{C}}(\boldsymbol{y}_{t+1})$
		\ENDFOR
		\STATE
          Option I~(monotone): Choose a number $l\in [T-1]$ with the distribution $\Pr(l=t) = \frac{1}{T-1}$ and \textbf{output} $\bx_l$
		\STATE
		Option II~(non-monotone): Choose a number $l\in [T-1]$ with the distribution $\Pr(l=t) = \frac{1}{T-1}$ and \textbf{output} $\frac{\bx_l + \underline{\bx}}{2}$ where $\underline{\bx}:=\argmin_{\bx\in \mathcal{C}} \|\bx\|_{\infty}$
	\end{algorithmic}
\end{algorithm}



\noindent In this subsection, we propose Boosting Gradient Ascent for the offline stochastic  submodular maximization problem, namely, $\max_{\boldsymbol{x}\in\mathcal{C}}f(\boldsymbol{x})$ where $f$ is a continuous DR-submodular function and $\mathcal{C}\subseteq[0,1]^{n}$ is a convex set. The pseudocode is shown in \cref{alg:1}.

As demonstrated in \cref{alg:1}, in each iteration, after calculating the estimated gradient $\widetilde{\nabla} F(\boldsymbol{x})$, we make the standard projected gradient step to update $\boldsymbol{x}$. Finally, according to the history of the trajectory of $\bx_t$, the algorithm randomly selects $l\in [T-1]$ and outputs $\bx_l$ or $\frac{\bx_l+\underline{\bx}}{2}$ depending on the monotonicity of the online function.
For both monotone and non-monotone objective functions, we establish the convergence results of \cref{alg:1}.


\begin{theorem}[Proof in \cref{Appendix:B1}]\label{thm:4}Assume $\mathcal{C}\subseteq\mathcal{X}$ is a bounded convex set, $f$ satisfies Assumption \ref{assumption1} and $f$ is $L$-smooth, the gradient oracle $\widetilde{\nabla}f(\boldsymbol{x})$ is unbiased with $\mathbb{E}(\|\widetilde{\nabla}f(\boldsymbol{x})-\nabla f(\boldsymbol{x})\|^{2}|\boldsymbol{x})\le\sigma^{2}$. If we take Option I and let $\eta_{t}=\frac{1}{\frac{\sigma_{\gamma}\sqrt{t}}{\mathrm{diam}(\mathcal{C})}+L_{\gamma}}$ in Algorithm~\ref{alg:1}, then we have
 \begin{align*}
    \mathbb{E}(f(\boldsymbol{x}_{l}))\ge\big(1-e^{-\gamma}\big)OPT-O\Big(\dfrac{1}{\sqrt{T}}\Big),
 \end{align*}
 where $OPT=\max_{\boldsymbol{x}\in\mathcal{C}}f(\boldsymbol{x})$.
\end{theorem}
\begin{remark}
\cref{thm:4} shows that after $O(1/\epsilon^{2})$ iterations, the boosting stochastic gradient ascent achieves $(1-1/e)OPT-\epsilon$, which efficiently improves the $(1/2)$-approximation guarantee of classical stochastic gradient ascent~\citep{hassani2017gradient} for continuous DR-submodular maximization. Moreover, we highlight that the overall gradient complexity is $O(1/\epsilon^2)$ which is optimal~\citep{hassani2020stochastic} under the stochastic setting.
\end{remark}

\noindent Similarly, we can conclude that
\begin{theorem}[Proof in \cref{Appendix:B2}]\label{thm:4 nonmonotone}
    Assume $\mathcal{C}\subseteq\mathcal{X}$ is a bounded convex set, $f$ satisfies Assumption \ref{assumption2} and $f$ is $L$-smooth, the gradient oracle $\widetilde{\nabla}f(\bx)$ is unbiased with $\mathbb{E}(\|\widetilde{\nabla} f(\bx) -\nabla f(\bx)\|^2 \mid \bx )\leq \sigma^2$. If we take Option II and let $\eta_t = \frac{1}{L\sqrt{t}}$ in Algorithm \ref{alg:1}, then we have
    \[\E\left(f\left(\frac{\bx_l+\underline{\bx}}{2}\right)\right)\geq \frac{1-\|\underline{\bx}\|_{\infty}}{4}OPT- O\left(\frac{1}{\sqrt{T}}\right).\]
\end{theorem}


\subsection{Online Learning with Delayed Feedback}\label{sec:delay}

In this section, we consider the online setting with delayed feedbacks. To begin, recall the process of classical online optimization. In round $t$, after picking an action $\boldsymbol{x}_{t}\in\mathcal{C}$, the environment~(adversary) gives a utility $f_{t}(\boldsymbol{x}_{t})$ and permits the access to the stochastic gradient of $f_{t}$.
The objective is to minimize the $\alpha$-regret for $T$ planned rounds.
Then, we turn to the (adversarial) feedback delays phenomenon \citep{quanrud2015online} in our online stochastic submodular maximization problem. That is, instead of the prompt feedback, the information about the stochastic gradient of $f_{t}$ could be delivered at the end of round $(t+d_{t}-1)$, where $d_{t}\in\mathbb{Z}_{+}$ is a positive integer delay for round $t$. For instance, the standard online setting sets all $d_{t}=1$~\citep{hazan2019introduction}. 

Next, we introduce some useful notations. We denote the feedback given at the end of round $t$ as $\mathcal{F}_{t}=\{u\in[T]: u+d_{u}-1=t\}$ and $D=\sum_{t=1}^{T}d_{t}$. Hence, at the end of round $t$, we only have access to the stochastic gradients of past $f_{s}$ where $s\in\mathcal{F}_{t}$. 

To improve the suboptimal $1/2$ approximation ratio of online gradient ascent and tackle the adversarial delays simultaneously, we employ the online delayed gradient algorithm~\citep{quanrud2015online} with the stochastic gradient of the non-oblivious function $F$. 
As shown in \cref{alg:2}, at each round $t$, after querying the stochastic gradient $\widetilde{\nabla}F_{t}(\boldsymbol{x}_{t})$, we apply the received stochastic gradients feedback $\widetilde{\nabla}F_{s}(\boldsymbol{x}_{s})$ $( s\in\mathcal{F}_{t})$ in a standard projection gradient step to update $\boldsymbol{x}_{t}$.

We provide the regret bound of Algorithm~\ref{alg:2} while applying on the monotone and non-monotone objectives.
\begin{theorem}[Proof in \cref{Appendix:D1}]\label{thm:5}
Assume that $\mathcal{C}\subseteq\mathcal{X}$ is a bounded convex set, each $f_{t}$ satisfies Assumption \ref{assumption1}. Meanwhile, the gradient oracle is unbiased $\mathbb{E}(\widetilde{\nabla}f_{t}(\boldsymbol{x})|\boldsymbol{x})=\nabla f_{t}(\boldsymbol{x})$ and $\E(\|\widetilde{\nabla}f_t(\bx)\|^2\mid \bx)\leq \widetilde{G}^2$. If we select Option I and set $\eta=\frac{diam(\mathcal{C})}{\widetilde{G}\sqrt{D}}$ in Algorithm~\ref{alg:2}, then we have
\begin{align*}
  (1-e^{-\gamma})\max_{\boldsymbol{x}\in\mathcal{C}}\sum_{t=1}^{T}f_{t}(\boldsymbol{x})-\mathbb{E}\left(\sum_{t=1}^{T}f_{t}(\boldsymbol{x}_{t})\right)\leq O(\sqrt{D}),  
\end{align*}
where $D=\sum_{i=1}^{T}d_{t}$ and $d_{t}\in\mathbb{Z}_{+}$ is a positive delay for the information about $f_{t}$.
\end{theorem}

\begin{remark}
When no delay exists, i.e., $d_{t}=1$ for all $t$, \cref{thm:5} says that the online boosting gradient ascent achieves a ($1-e^{-\gamma}$)-regret of $O(\sqrt{T})$. 
To the best of our knowledge, this is the first result achieving a $(1-e^{-\gamma})$-regret of $O(\sqrt{T})$ with $O(1)$ stochastic gradient queries for each submodular function $f_{t}$. 
\end{remark}

\begin{remark}
Under the delays of stochastic gradients, \cref{thm:5} gives the first regret analysis for the online stochastic submodular maximization problem. 
It is worth mentioning that the $(1-e^{-\gamma})$-regret of $O(\sqrt{D})$ result not only achieves the optimal $(1-e^{-\gamma})$ approximation ratio, but also matches the $O(\sqrt{D})$ regret of online convex optimization with adversarial delays~\citep{quanrud2015online}.
\end{remark}
\begin{theorem}[Proof in \cref{Appendix:D2}]\label{thm:5 nonmonotone}
    Assume that $\mathcal{C}\subseteq\mathcal{X}$ is a bounded convex set and each $f_{t}$ satisfies Assumption \ref{assumption2}. Meanwhile, the gradient oracle is unbiased $\mathbb{E}(\widetilde{\nabla}f_{t}(\boldsymbol{x})|\boldsymbol{x})=\nabla f_{t}(\boldsymbol{x})$ and $\E(\|\widetilde{\nabla}f_t(\bx)\|^2\mid \bx)\leq \widetilde{G}^2$. If we select Option II and set $\eta=\frac{diam(\mathcal{C})}{\widetilde{G}\sqrt{D}}$ in Algorithm~\ref{alg:2}, then we have
    \begin{align*}
        \frac{1-\|\underline{\bx}\|_{\infty}}{4} \max_{\bx\in\mathcal{C}} \sum_{t=1}^T f_t\left(\bx\right) - \E\left(\sum_{t=1}^T f_t\left(\frac{\bx_t+\underline{\bx}}{2}\right)\right)\leq O(\sqrt{D})
    \end{align*}
    where $D=\sum_{i=1}^{T}d_{t}$ and $d_{t}\in\mathbb{Z}_{+}$ is a positive delay for the information about $f_{t}$.
\end{theorem}
\begin{algorithm}[t]
	\caption{Online Boosting Delayed Gradient Ascent}\label{alg:2}
	\hspace*{0.02in} {\bf Input:} $T$, $\eta$, $\gamma$
	\begin{algorithmic}[1]
		\STATE \textbf{Initialize:} any  $\boldsymbol{x}_{1}\in\mathcal{C}$.
		\FOR{$t\in [T]$}
		\STATE Option I~(monotone):
        \STATE \hspace{\algorithmicindent} Play $\boldsymbol{x}_{t}$
        \STATE \hspace{\algorithmicindent} Sample $z_{t}$ from r.v. $\mathbf{Z}_{\uparrow}$ and query $\widetilde{\nabla}F_t(\boldsymbol{x}_{t})=\frac{1-e^{-\gamma}}{\gamma}\widetilde{\nabla}f_{t}(z_{t}*\boldsymbol{x}_{t})$
        \STATE Option II~(non-monotone): 
        \STATE \hspace{\algorithmicindent} Play $\frac{\bx_t+\underline{\bx}}{2}$ where $\underline{\bx}:=\argmin_{\bx\in \mathcal{C}} \|\bx\|_{\infty}$
        \STATE \hspace{\algorithmicindent} Sample $z_t$ from r.v. $\mathbf{Z}_{\sim}$ and query $\widetilde{\nabla}F_t(\bx_t)=\frac{3}{8}\widetilde{\nabla}f_t\left(\frac{z_t}{2}*\bx_t+(1-\frac{z_t}{2})*\underline{\bx}\right)$
		\STATE Receive feedback $\widetilde{\nabla}F_s(\x_s)$, where $s\in \mathcal{F}_t$
		\STATE $\boldsymbol{y}_{t+1}=\boldsymbol{x}_{t}+\eta\sum_{s\in\mathcal{F}_{t}}\widetilde{\nabla}F_s(\boldsymbol{x}_{s})$
		\STATE $\boldsymbol{x}_{t+1}=\mathcal{P}_{\mathcal{C}}(\boldsymbol{y}_{t+1})$
		\ENDFOR
	\end{algorithmic}
\end{algorithm}


\subsection{Bandit Optimization}\label{sec:bandit}
In the bandit setting, the learning agent first picks an action $\bx_t\in\mathcal{C}$ in each round $t$, then the adversary reveals a utility value $f_t(\bx_t)$ to the agent. Different from the standard online learning, the learning agent is not permitted to query the gradient of $f_t$, and the only accessible information is $f_t(\bx_t)$.

\begin{algorithm}[t]
	\caption{Boosting Bandit Gradient Ascent}
	\label{algo:BBGA}
	\hspace*{0.02in}{\bf Input:} smoothing radius $\delta$, $\by$ and $R$ such that $\mathbb{B}(\by,R)\subseteq \mathcal{C}$, weakly DR-submodular parameter $\gamma$ for OPTION I, exploration rate $\lambda \in (0,1)$, learning rate $\eta$
	\begin{algorithmic}[1]
	   \STATE Initialize $x_1\in \mathcal{C}$ arbitrarily
          \STATE $\delta'\leftarrow \frac{\delta}{R-\delta}$
          \STATE Construct Minkowsky set $\mathcal{C}_{\delta',\by}$
          \STATE $\boldsymbol{0}_{\delta'}\leftarrow \mathcal{P}_{\mathcal{C}_{\delta',\by}}(\boldsymbol{0})$
          \STATE $\underline{\bx}\leftarrow \argmin_{\bx\in \mathcal{C}} \|\bx\|_{\infty}$
          \STATE $\underline{\bx}_{\delta'}\leftarrow \mathcal{P}_{\mathcal{C}_{\delta',\by}}\left(\underline{\bx}\right)$
	   \FOR{$t=1,2,\ldots,T$}
            \STATE With probability $\lambda$, set $\Upsilon_t = \mathrm{explore}$ and set $\Upsilon_t = \mathrm{exploit}$ with probability $1-\lambda$.
            \IF{$\Upsilon_t = \mathrm{explore}$}
                    \STATE draw $\bv_t\sim \mathbb{S}_{d-1}$
                    \STATE Option I~(monotone):
                    \STATE \hspace{\algorithmicindent} sample $z_t$ from r.v. $\mathbf{Z}_{\uparrow}$
	                \STATE \hspace{\algorithmicindent} play $\bx_t = z_t\cdot \by_t+(1-z_t)\boldsymbol{0}_{\delta'}+\delta \cdot \bv_t$ and observe $f_t(\bx_t)$
	                \STATE \hspace{\algorithmicindent} $\widetilde{\nabla}F_t(\by_t) \gets \frac{1-e^{-\gamma}}{\gamma}\frac{d}{\lambda\delta}f_t(\bx_t)\bv_t$
                    \STATE Option II~(non-monotone):
                    \STATE \hspace{\algorithmicindent} sample $z_t$ from r.v. $\mathbf{Z}_{\sim}$
                    \STATE \hspace{\algorithmicindent} play $\bx_t = \frac{z_t}{2} (\by_t - \underline{\bx}_{\delta'})+\underline{\bx}_{\delta'} +\delta \cdot \bv_t$ and observe $f_t(\bx_t)$
	                \STATE \hspace{\algorithmicindent} $\widetilde{\nabla}F_t(\by_t) \gets \frac{3}{8}\frac{d}{\lambda\delta}f_t(\bx_t)\bv_t$
            \ENDIF
            \IF{$\Upsilon_t = \mathrm{exploit}$}
                \STATE Option I~(monotone):
                \STATE \hspace{\algorithmicindent} play $\bx_t=\by_t$
                \STATE Option II~(non-monotone):
                \STATE \hspace{\algorithmicindent} play $\bx_t = \frac{\by_t+\underline{\bx}}{2}$
                \STATE $\widetilde{\nabla} F_t(\by_t) \leftarrow \boldsymbol{0}$
	        \ENDIF
            \STATE $\by_{t+1} \gets \mathcal{P}_{\mathcal{C_{\delta',\by}}}\left(\by_t+\eta \widetilde{\nabla}F_t(\by_t)\right)$
	   \ENDFOR
	\end{algorithmic}
\end{algorithm}

Since we are not able to query the gradient of $f_t(\bx_t)$ directly, a natural idea is to estimate the gradient using the zeroth-order information of $f_t$ and then plug this estimate of gradient into online boosting gradient ascent method. \emph{How to realize this high-level idea?} Generally speaking, it is challenging to estimate the gradient of a continuous function  throughout the function value at a single point.
To circumvent this technical obstacle, \cite{flaxman2005online} move their focus on the $\delta$-smoothed version $\hat{f}_{t}^{\delta}(\bx)$ of $f_t(\bx)$, which is defined by the averaging of $f_t$ over a ball of radius $\delta$ and centered at $\bx$. That is, $\hat{f}^{\delta}_t(\bx) = \E_{\boldsymbol{u}\sim \mathbb{B}_d}\left(f_t(\bx+\delta\boldsymbol{u})\right)$ where $\boldsymbol{u}\sim \mathbb{B}_{d}$ indicates that $\boldsymbol{u}$ is selected uniformly at random from a $d$-dimensional unit ball. Surprisingly, \cite{flaxman2005online} find that it is possible to construct an unbiased gradient estimator of $\hat{f}^{\delta}_t(\bx)$ throughout one-point function value. In formal, they prove:
\begin{lemma}[\cite{flaxman2005online}]\label{lem:fkm estimator}
    Let $f_t$ be a continuous function and $\bv$ be a random vector uniformly sampled from the $d-1$ dimensional unit sphere $\mathbb{S}_{d-1}$, then for any $\bx$ and $\delta >0$, we have 
    \[\E_{\bv\sim \mathbb{S}_{d-1}}\left(\frac{d}{\delta}f_t(\bx + \delta\cdot \bv)\bv\right) = \nabla \hat{f}_t^{\delta}(\bx),\]
    where $\hat{f}_t^{\delta}$ is the $\delta$-smoothed version of $f_t$, defined by $\hat{f}_t^{\delta}(\bx) = \E_{\boldsymbol{u}\sim \mathbb{B}_{d}}\left(f_t(\bx+\delta\boldsymbol{u})\right)$.
\end{lemma}

This lemma shows that one can query the function value of a random point on the sphere centered at $\bx$ with radius $\delta$, and next use this value to get an unbiased estimate of $\nabla \hat{f}_{t}^{\delta}(\bx)$.
 Furthermore, we can prove that the $\delta$-smoothed version $\hat{f}_t^{\delta}$ is also dr-submodular and has the same monotonicity as $f_t$ (\cref{lem:submodular smooth} in \cref{Appendix:bandit}), which makes it possible to run the boosting online gradient ascent algorithm (Algoritm~\ref{alg:2}) to the sequence $\{\hat{f}_t^{\delta}\}_{t=1}^T$.  When $\hat{f}_t^{\delta}(\bx)$ is close to $f_t(\bx)$, we can transform the regret bound of sequence $\{f_t\}_{t=1}^T$ to the regret bound with respect to $\{\hat{f}_t^{\delta}\}_{t=1}^T$ and a few additive regret loss. 
 
 Nevertheless, the one-sample gradient estimator may not be applied directly since $\bx+\delta \bv$ can fall outside the constraint. To fix this flaw, we need to find a \emph{$\delta$-interior} $\mathcal{C}^{\circ}_{\delta}$ of $\mathcal{C}$ such that, for any $\bx \in \mathcal{C}^{\circ}_{\delta}$ and $\bv\in \mathbb{S}_{d-1}$, we have $\bx+\delta\bv \in \mathcal{C}$. Also, $\mathcal{C}^{\circ}_{\delta}$ should be large enough so that the optimal revenue of fixed action in the $\delta$-interior is close to the optimal revenue of the fixed action in the original constraint. That is, we need that $\left|\max_{\bx\in \mathcal{C}^{\circ}_{\delta}}\sum_{t=1}f_t(\bx)-\max_{\bx\in \mathcal{C}}\sum_{t=1}f_t(\bx)\right|$ is small enough. Such a $\delta$-interior can be constructed through the Minkowsky set~\citep{abernethy2008competing} of $\mathcal{C}$ if $\mathcal{C}$ is compact and convex.
\begin{definition}[Minkowsky set~\citep{abernethy2008competing}]\label{def:Minkowsky}
    Let $\mathcal{C}$ be a compact convex set, the Minkowsky function $\pi_{\by}:\mathcal{C} \rightarrow \mathbb{R}$ parameterized by a pole $\by\in \mathrm{int} (\mathcal{C})$ is defined as $\pi_{\by}(\bx)\triangleq \mathrm{inf}\{t\geq 0 \mid \by+t^{-1}(\bx-\by)\in \mathcal{C}\}$. Given $\delta'\in \mathbb{R}^+$ and $\by_1\in \mathrm{int}(\mathcal{C})$, we define the Minkowsky set $\mathcal{C}_{\delta',\by_1}\triangleq\{\by\in \mathcal{C}\mid \pi_{\by_1}(\by)\leq (1+\delta')^{-1}\}$.
\end{definition}
Under mild assumption, we can construct a $\delta$-interior via Minkowsky set by selecting an appropriate $\delta'$.
\begin{assumption}\label{assum:inner ball}
    There exist a $R>d^{1/3}T^{-1/5}$ and $\by\in \mathcal{C}$, such that $\mathbb{B}(\by, R)\subseteq\mathcal{C}$. Here $\mathbb{B}(\by, R)$ denotes the ball centered at $\by$ with a radius of $R$.
\end{assumption}

\begin{lemma}[Proof in \cref{sect:proof of lemma 6}]\label{lem:interior}
    Under Assumption \ref{assum:inner ball}, the Minkowsky set $\mathcal{C}_{\delta',\by}$ is convex and for $\bx \in \mathcal{C}_{\delta',\by}$, $\mathbb{B}(\bx,\frac{\delta'}{1+\delta'} R)\subseteq \mathcal{C}$. In another word, $\mathcal{C}_{\delta',\by}$ is a $\frac{\delta'}{1+\delta'} R$-interior of $\mathcal{C}$.
\end{lemma}

Given $\delta<R$ and Assumption \ref{assum:inner ball}, we can construct $\mathcal{C}_{\delta',\by}$ with $\delta':=\frac{\delta}{R-\delta}$ as a $\delta$-interior. All the above techniques are quite standard in the bandit optimization literature. But the presence of the non-oblivious function introduces extra feasibility issue. Taking the monotonic case as an example: to estimate the boosting gradient of the $\delta$-smoothed objective $\hat{f}_t^{\delta}$ at a specific point $\bx\in\mathcal{C}_{\delta', \by}$, we need to play $z \cdot \bx + \delta \bv$ to get its function value for some $z\in [0,1]$ and $\bv\in \mathbb{S}_{d-1}$, which can jump out $\mathcal{C}$ since $z\cdot \bx$ may not lie in $\mathcal{C}_{\delta', \by}$. To overcome this issue, we play $\widetilde{\bx} = z \cdot \bx + (1-z)\boldsymbol{0}_{\delta'} + \delta \bv$ and use the value $f_t(\widetilde{\bx})$ to replace $f_t(z \cdot \bx + \delta \bv)$ in our gradient estimator, where $\boldsymbol{0}_{\delta'}:=\mathcal{P}_{\mathcal{C}_{\delta',\by}}(\boldsymbol{0})$. Then it's easy to see $z \cdot \bx + (1-z)\boldsymbol{0}_{\delta'} \in \mathcal{C}_{\delta',\by}$. Therefore $\widetilde{\bx}$ is feasible. This idea can also apply to the non-monotone case. In formal, we present the pseudo-code of our bandit algorithm in \cref{algo:BBGA}. Next, we can verify the feasibility of the algorithm, namely,

\begin{lemma}
    In our \cref{algo:BBGA}, for both OPTION I and OPTION II,
    \begin{equation*}
        \bx_t\in \mathcal{C}, \forall t\in [T].
    \end{equation*}
\end{lemma}
\begin{proof}
    When $\Upsilon_t=\mbox{explore}$, we have proved the case of OPTION I above. For OPTION II, note that $\frac{z_t}{2} (\by_t - \underline{\bx}_{\delta'})+ \underline{\bx}_{\delta'}$ is a convex combination of $\by_t$ and $\underline{\bx}_{\delta'}$. By the definition of $\by_t$ and $\underline{\bx}_{\delta'}$, they are both in $\mathcal{C}_{\delta',\by}$. Since $\mathcal{C}_{\delta',\by}$ is convex by \cref{lem:interior}, $\frac{z_t}{2} (\by_t - \underline{\bx}_{\delta'})+ \underline{\bx}_{\delta'}\in \mathcal{C}_{\delta',\by}$. Then $\bx_t = \frac{z_t}{2} (\by_t - \underline{\bx}_{\delta'})+\underline{\bx}_{\delta'} +\delta \cdot \bv_t\in \mathcal{C}$ since $\mathcal{C}_{\delta',\by}$ is a $\delta$-interior.

    When $\Upsilon_t = \mbox{exploit}$, $\bx_t=\by_t\in \mathcal{C}_{\delta',\by}\subseteq \mathcal{C}$ or $\bx_t =\frac{\by_t+\underline{\bx}}{2}$. Since $\by_t,\underline{\bx}\in \mathcal{C}$, their convex combination $\frac{\by_t+\underline{\bx}}{2}\in\mathcal{C}$.  
\end{proof}

Besides the infeasibility problem, \cref{algo:BBGA} also takes the exploration-exploitation trade-off to tackle another technical issue brought by the the non-oblivious functions. Let's consider the monotone case, if we want to obtain as much reward as possible to minimize regret, we need to select an action near the $\by_t$ which denotes the actions recommended via the full-information boosting online gradient ascent of sequence $\{\hat{f}_t^{\delta}\}_{t=1}^T$. However, to estimate the gradient of the non-oblivious function by a one-sample gradient estimator, we must query the function value near the point $z_t \cdot \by_t + (1-z_t)\boldsymbol{0}_{\delta'}$ where $z_{t}\sim\mathbf{Z}_{\uparrow}$, which may be far away from $\by_t$. In the bandit optimization literature, the former is often referred to as exploitation, while the latter is typically referred to as exploration. Given the exploration probability $\lambda\in (0,1)$, at each round, we execute exploration with probability $\lambda$, that is, selecting the point far away from $\by_t$ to obtain the gradient estimate of the non-oblivious function at $\by_t$. With probability $1-\lambda$, we execute exploitation to select $\by_t$ (monotone case) or $\frac{\by_t+\underline{\bx}}{2}$ (non-monotone case) to accumulate reward and set the gradient estimate to $\boldsymbol{0}$.

Before presenting the regret bound of the \cref{algo:BBGA}, we make the following assumption, which is standard in the bandit literature.
\begin{assumption}\label{assumption4}
    There exists a constant $M$ such that, for any $t$ and $\bx\in\mathcal{C}$, $|f_t(\bx)|\leq M$.
\end{assumption}
As a result, we can verify that:
\begin{theorem}[Proof in \cref{Appendix:monotone_bandit}]\label{thm:monotone bandit}
    Assume that $\mathcal{C}\subseteq\mathcal{X}$ is a bounded convex set containing $\boldsymbol{0}$ and satisfies Assumption \ref{assum:inner ball}. Each $f_{t}$ is $L_1$-Lipschitz continuous, $L_2$-smooth and satisfies Assumption \ref{assumption1} and \ref{assumption4}. If we set $\lambda = d^{1/3}T^{-1/5}, \delta = d^{1/3}T^{-1/5}, \eta =d^{-1/3}T^{-4/5}$ and select Option I in \cref{algo:BBGA}, then we have
    \[(1-e^{-\gamma})\max_{\bx\in \mathcal{C}}\sum_{t=1}^T f_t(\bx) - \E \left(\sum_{t=1}^T f_t(\bx_t)\right)= O(d^{1/3} T^{4/5}).\]
\end{theorem}
\begin{theorem}[Proof in \cref{Appendix:non_monotone_bandit}]\label{thm:nonmonotone bandit}
    Assume that $\mathcal{C}\subseteq\mathcal{X}$ is a bounded convex set and satisfies Assumption \ref{assum:inner ball}. Each $f_{t}$ is $L_1$-Lipschitz continuous, $L_2$-smooth and satisfies Assumption \ref{assumption2} and \ref{assumption4}. If we set $\lambda = d^{1/3}T^{-1/5}, \delta = d^{1/3}T^{-1/5}, \eta = d^{-1/3}T^{-4/5}$ and select Option II in \cref{algo:BBGA} and select Option II in \cref{algo:BBGA}, then we have
    \[\frac{1-\|\underline{\bx}\|_{\infty}}{4}\max_{\bx\in \mathcal{C}}\sum_{t=1}^T f_t(\bx) - \E \left(\sum_{t=1}^T f_t(\bx_t)\right)= O(d^{1/3} T^{4/5}).\]
\end{theorem}

\subsection{Minimax Optimization of Convex-Submodular Functions}\label{sec:minimax}
Minimax optimization appears in a wide range of domains such as robust optimization\citep{ben2009robust} and game theory\citep{osborne1994course}. In this section, we investigate a special non convex-concave minimax optimization, which is coined by \cite{adibi2022minimax}. Let $f(\bx,S)$ be a function defined on a continuous-discrete mixed constraint $\mathcal{K}\times \mathcal{I}$, where $\mathcal{K}\subseteq[0,1]^{n}$ is a convex body and $\mathcal{I}\subseteq 2^V$ is a collection of subset of a finite ground $V$. Moreover, we assume $f$ is convex-submodular, that is, $f(\bx,S)$ is convex w.r.t. $\bx$ and submodular w.r.t. $S$. Given this convex-submodular objective $f$, we usually consider the following minimax optimization problem:
\begin{equation}\label{eq:minimax}
    \min_{\bx\in \mathcal{K}}\max_{S\in \mathcal{I}} f(\bx,S).
\end{equation}
 According to \citet{adibi2022minimax}, this problem \eqref{eq:minimax} is NP-hard to solve accurately, so we hope to find an approximation solution as follows.
\begin{definition}[\cite{adibi2022minimax}]
    We call a point $\hat{\bx}$ an $(\alpha,\epsilon)$-approximation minimax solution of problem (\ref{eq:minimax}) if it satisfies
    \[\alpha \max_{S\in \mathcal{I}} f(\hat{\bx},S)\leq OPT+\epsilon.\]
\end{definition}

When $\mathcal{I}$ is a \emph{uniform} matroid, \citet{adibi2022minimax} propose several algorithms which can produce an optimal $(1-1/e, \epsilon)$-approximation solution. However, for general matroid $\mathcal{I}$, their algorithms only can guarantee a sub-optimal $(\frac{1}{2},\epsilon)$-approximation solution. To improve this flaw, we leverage our non-oblivious function to devise a tight $(1-1/e,\epsilon)$-approximation algorithm. Furthermore, our boosting technique can achieve a $(\frac{1}{4},\epsilon)$-approximation solution for the non-monotone cases over general matroid constraint. In contrast, all results of \citet{adibi2022minimax} are under the assumption that the submodular part of $f(\bx,S)$ is monotone w.r.t. $S$. Prior to introducing our algorithm, we turn to the continuous extension version of the convex-submodular minimax optimization problem.

\begin{definition}[\cite{adibi2022minimax}]
    The multi-linear extension of the convex-submodular function $f:\R^d \times 2^V\rightarrow \R_+$ is the function $\hat{f}:\R^d \times [0,1]^{|V|}\rightarrow \R_+$ defined as $\hat{f}(\bx,\by)=\E_{S\sim\by}\left(f(\bx,S)\right)$, where $S\sim\by$ indicates that each element $i\in V$ is included in $S$ with probability $y_i$ independently.
\end{definition}

\cite{adibi2022minimax} show that, the original problem is equivalent to its multi-linear extension version. Rigoriously, the following lemma holds.

\begin{lemma}[\cite{adibi2022minimax}]
    Let $\mathcal{C}$ be the convex hull of the matroid $\mathcal{I}$, then for any $\bx\in\mathcal{K}$, $\max_{S\in\mathcal{I}} f(\bx,S)=\max_{\by\in\mathcal{C}} \hat{f}(\bx,\by)$. As a result, any approximate solution $\hat{\bx}$ of the multi-linear version problem maintains its approximation ratio in the original problem.
\end{lemma}
As a corollary, to find a $(1-1/e,\epsilon)$-approximation solution of the original problem (\ref{eq:minimax}), we only need to find a $(1-1/e,\epsilon)$-approximation solution of the following multi-linear extension version of (\ref{eq:minimax}).
\begin{equation}
    min_{\bx\in \mathcal{K}}\max_{\by\in \mathcal{C}} \hat{f}(\bx,\by).
\end{equation}

Fixing $\bx$, it's well known that $\hat{f}(\bx,\by)$ is a DR-submodular function if $f(\bx,S)$ is convex-submodular. Futhermore, $\hat{f}(\bx,\by)$ has the same monotonicity of $f(\bx,S)$. We further make the following assumption about $\hat{f}$.
\begin{assumption}\label{assumption:minimax}
The gradient of $\hat{f}$ with respect to $\bx$ and $\by$ is uniformly bounded by a constant $G$. That is, for any $\bx\in\mathcal{K}$ and $\by\in\mathcal{C}$, we have $\|\nabla_{\bx} \hat{f}(\bx,\by)\|\le G$ and $\|\nabla_{\by} \hat{f}(\bx,\by)\|\le G$. Furthermore, the stochastic gradient oracle $\widetilde{\nabla}_{\bx} \hat{f}$ and $\widetilde{\nabla}_{\by} \hat{f}$ are unbiased and satisfy $\E\left(\|\widetilde{\nabla}_{\bx} \hat{f}(\bx,\by)\|^2\right)\leq \widetilde{G}^2$ and $\E\left(\|\widetilde{\nabla}_{\by} \hat{f}(\bx,\by)\|^2\right)\leq \widetilde{G}^2$ for a constant $\widetilde{G}$.
\end{assumption}

Our algorithm alternately executes the step of gradient descent or ascent by fixing $\bx$ or $\by$. Especially when we execute the gradient ascent step on the DR-submodular part, we use the gradient of its non-oblivious function. For details, see Algorithm \ref{minimax optimization montone}.

\begin{algorithm}[tbp]
    \caption{Boosting Gradient Descent Ascent}
    \label{minimax optimization montone}
    \hspace*{0.02in} {\bf Input:} $T, \eta$
    \begin{algorithmic}[1]
    \STATE \textbf{Initialize} any $\bx_1\in \mathcal{K}, \by_1\in \mathcal{C}$.
    \FOR{$t\in [T]$}
        \STATE 
        Option I~(monotone): 
        \STATE \hspace{\algorithmicindent} $\bx_{t+1} \gets \mathcal{P}_{\mathcal{K}}\left(\bx_t - \eta\widetilde{\nabla}_{\bx} \hat{f}\left(\bx_t,\by_t\right)\right)$
        \STATE \hspace{\algorithmicindent} Sample $z_t$ from $\boldsymbol{Z}_{\uparrow}$ 
        \STATE \hspace{\algorithmicindent} $\by_{t+1} = \mathcal{P}_{\mathcal{C}}\left(\by_t+\eta (1-e^{-1})\widetilde{\nabla} \hat{f}_{\by}(\bx_t,z_t\cdot \by_t)\right)$
        \STATE
        Option II~(non-monotone): 
        \STATE \hspace{\algorithmicindent} $\bx_{t+1} \gets \mathcal{P}_{\mathcal{K}}\left(\bx_t - \eta\widetilde{\nabla}_{\bx} \hat{f}\left(\bx_t,\frac{\by_t+\underline{\by}}{2}\right)\right)$
        \STATE \hspace{\algorithmicindent} Sample $z_t$ from $\boldsymbol{Z}_{\sim}$
        \STATE \hspace{\algorithmicindent} $\by_t \gets \mathcal{P}_{\mathcal{C}}\left(\by_t+\frac{3\eta}{8}\widetilde{\nabla} \hat{f}_{\by}\left(\bx_t,\frac{z_t}{2}\cdot \boldsymbol{y}_t+(1-\frac{z_t}{2})\cdot\underline{\boldsymbol{y}}\right)\right)$ where $\underline{\boldsymbol{y}}:=\argmin_{\boldsymbol{y}\in \mathcal{C}} \|\boldsymbol{y}\|_{\infty}$
    \ENDFOR
    \RETURN $\bx_{sol}=\sum_{t=1}^T \frac{1}{T}\bx_t$ 
    \end{algorithmic}
\end{algorithm}

Next, we show that our Algorithm~\ref{minimax optimization montone} is able to find a $(1-1/e,\epsilon)$-approximate solution after $O(\frac{1}{\epsilon^{2}})$ iterations, when and $f(\bx,S)$ is monotone submodular w.r.t. $S$.
    \begin{theorem}[Proof in Appendix~\ref{Appendix:mono_minimax}]\label{thm:minimax}
        Let $f(\bx,S)$ be a convex-submodular function, $\hat{f}(\bx,\by)$ be its multi-linear extension. Assuming $\hat{f}(\bx,\by)$ satisfies Assumption \ref{assumption:minimax} and $f(\bx,S)$ is monotone with respect to $S$.
        Let $T = \frac{(3-e^{-1})^2 
\widetilde{G}^2\left(\mathrm{diam}^2(\mathcal{C})+\mathrm{diam}^2(\mathcal{K})\right)}{4\epsilon^2}$, $\eta = \frac{\sqrt{\mathrm{diam}^2(\mathcal{C})+\mathrm{diam}^2(\mathcal{K})}}{\widetilde{G}\sqrt{T}}$ and take Option I in Algorithm \ref{minimax optimization montone}, then
        \[(1-e^{-1})\max_{\by \in \mathcal{C}}\E\left(\hat{f}\left(\bx_{sol},\by\right)\right)\leq OPT + \epsilon.\]
    \end{theorem}
    
Note that when applying Algorithm \ref{minimax optimization montone} on the multi-linear extension version of \eqref{eq:minimax}, $\mathcal{C}$ is the matroid convex hull, then $\underline{\by}=\boldsymbol{0}$. Thus, we have a similar result for non-monotone case:
    \begin{theorem}[Proof in Appendix~\ref{Appendix:non_mono_minimax}]\label{thm:minimax nonmonotone}
        Let $f(\bx,S)$ be a convex-submodular function, $\hat{f}(\bx,\by)$ be its multi-linear extension. Assuming $\hat{f}(\bx,\by)$ satisfies Assumption \ref{assumption:minimax}. Let $T = \frac{361 \widetilde{G}^2\left(\mathrm{diam}^2(\mathcal{C})+\mathrm{diam}^2(\mathcal{K})\right)}{256\epsilon^2}  $, $\eta = \frac{\sqrt{\mathrm{diam}^2(\mathcal{C})+\mathrm{diam}^2(\mathcal{K})}}{\widetilde{G}\sqrt{T}}$ and take Option II in Algorithm \ref{minimax optimization montone}, then
        \[\frac{1}{4}\max_{\by\in\mathcal{C}}\E\left(\hat{f}\left(\bx_{sol},\by\right)\right)\leq OPT+ \epsilon.\]
    \end{theorem}

\section{Numerical Experiments}\label{sec:experiments}
In this section, we empirically evaluate our proposed boosting projected gradient algorithms in three different optimization scenarios, namely, offline settings, online learning with different types of feedbacks and convex-submodular cases.  Note that i) all experiments are performed in Python 3.6.5 using CVX optimization tool~\citep{grant2014cvx} on a MacBook Pro with Apple M1 Pro and 16GB RAM; ii) To avoid the randomness of stochastic gradients, we repeat each trial $10$ times and report the average results; iii) For ease of exposition, this section only focuses on special coverage maximization~\citep{hassani2017gradient,chen2023continuous} and real-world movie recommendation. As for the rest experiments about the simulated quadratic programming, we present them in Appendix~\ref{Appendix:QP}.


\subsection{Offline Settings}
In this subsection, we consider offline continuous DR-submodular maximization problems and compare the following algorithms:
\begin{itemize}
    \item\textbf{Boosting Gradient Ascent~(BGA($B$))}: In the framework of Algorithm~\ref{alg:1}, we use the average of $B$ independent stochastic gradients to estimate $\nabla F$ in every iteration.
\item\textbf{Gradient Ascent~(GA($B$))}: We consider Algorithm 1 in \citet{hassani2017gradient}. We also use an average of $B$ independent stochastic gradients to estimate $\nabla f$ in every iteration.
\item\textbf{Continuous Greedy~(CG)}: Algorithm 1 in \cite{bian2017guaranteed} for monotone continuous DR-submodular maximization over general convex constraints.
\item\textbf{Stochastic Continuous Greedy~(SCG)}: Algorithm 2 in \citet{mokhtari2020stochastic} with $\rho_{t}=4/(t+8)^{2/3}$ for monotone DR-submodular maximization over general convex constraints.
\item\textbf{Non-monotone Frank-Wolfe~(Non-mono FW)} Algorithm 1 in \citet{mualem2023resolving} with error parameter $\epsilon=0.01$  for non-monotone DR-submodular maximization over general convex constraints.
\item\textbf{Measured Frank-Wolfe~(Measured FW)}: Algorithm 2 in \cite{mitra2021submodular+} for deterministic non-monotone DR-submodular maximization over down-closed convex constraints with $1/e$-approximation guarantee.
\item\textbf{Variance-reduced Measured Frank-Wolfe~(Measured FW-VR)}: A variant of Frank Wolfe algorithm merges the variance reduction technique\citep{mokhtari2018conditional,mokhtari2020stochastic} into Algorithm 2 of \cite{mitra2021submodular+}  with $\rho_{t}=4/(t+8)^{2/3}$. Notably, this algorithm is also designed for non-monotone stochastic DR-submodular maximization over down-closed convex constraints.

\end{itemize}
\begin{figure*}[t]
\vspace{-1.0em}
\centering
\subfigure[Special Monotone Case \label{graph1}]{\includegraphics[width=0.4\linewidth]{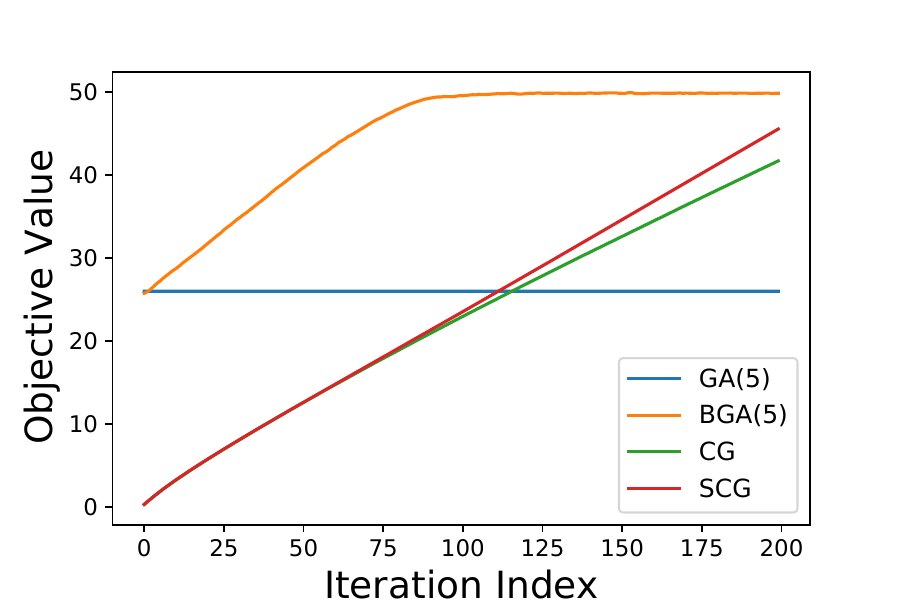}}
\subfigure[Special Monotone Case~(origin)\label{graph2}]{\includegraphics[width=0.4\linewidth]{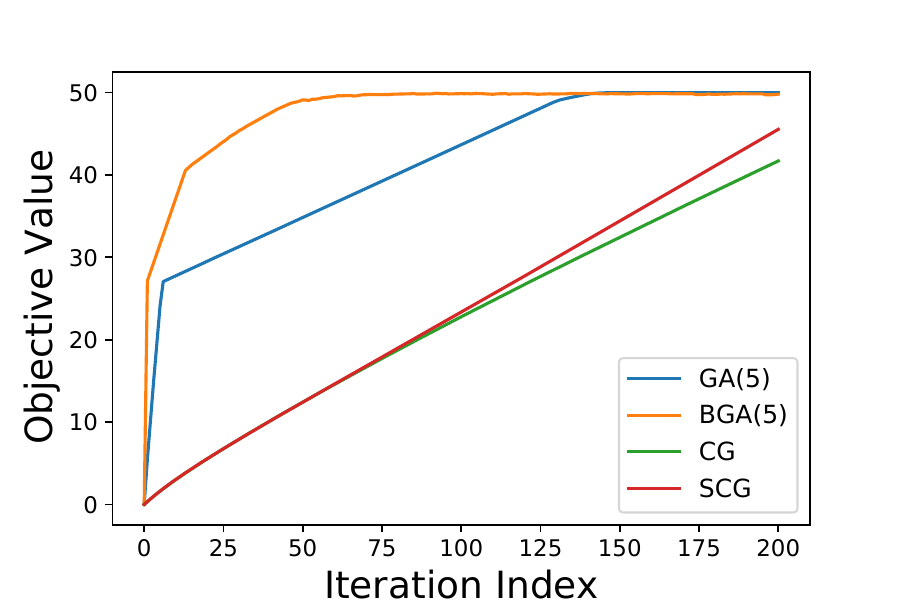}}

\subfigure[Special Non-Monotone Case\label{graph3}]{\includegraphics[width=0.4\linewidth]{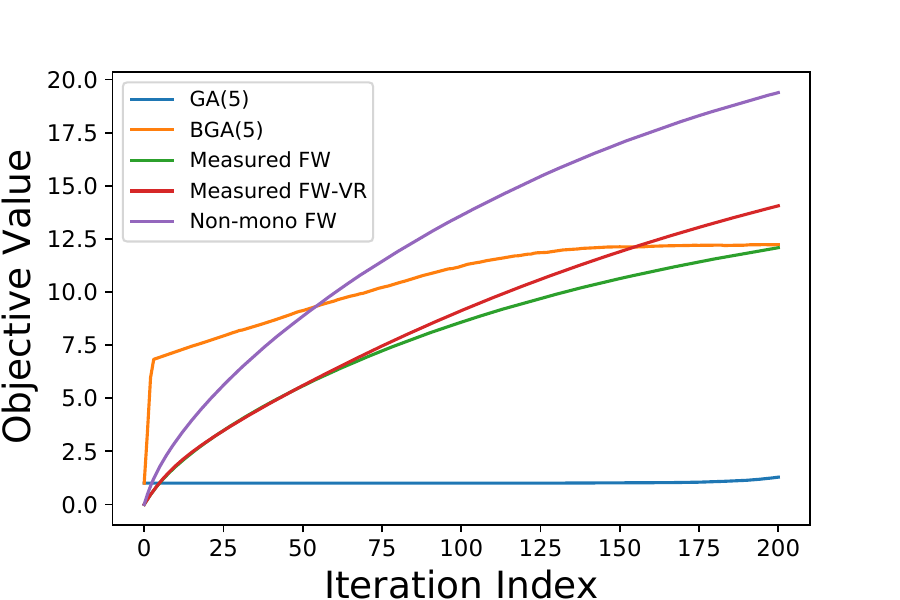}}
\subfigure[Special Non-Monotone Case~(origin)\label{graph4}]{\includegraphics[width=0.4\linewidth]{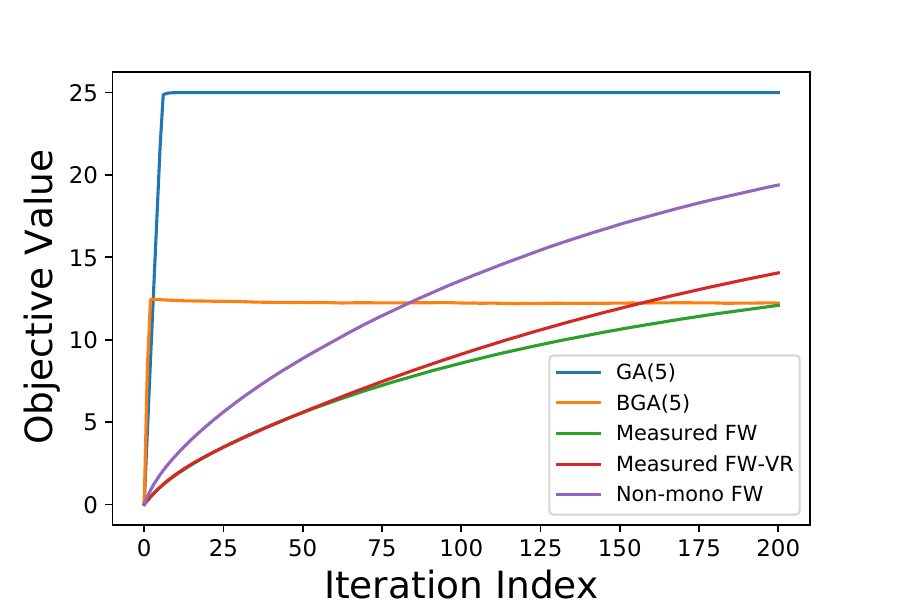}}
\caption{In ~\cref{graph1}, we test the performance of the four algorithms for the special \emph{monotone} submodular function in \cite{hassani2017gradient} where the GA(5) and BGA(5) start from $\boldsymbol{x}_{loc}$. Simultaneously, we present the results for all algorithm starting from the origin in ~\cref{graph2}. ~\cref{graph3} show the performance of four algorithms for the special \emph{non-monotone} submodular function in \cite{chen2023continuous} where the GA(5) and BGA(5) start from $\widetilde{\boldsymbol{x}}_{loc}$. Similarly, ~\cref{graph4} presents the results from the origin point.}
\vspace{-1.0em}
\end{figure*}
\subsubsection{Special Case}
\ \ 
\textbf{Monotone Setting:} \citet{hassani2017gradient} introduced a special \emph{monotone} continuous DR-submodular function $f_{k}$ coming from the multilinear extension of a set cover function. Here,  
$f_{k}(\boldsymbol{x})=k+1-(1-x_{2k+1})\prod_{i=1}^{k}(1-x_{i})-(1-x_{2k+1})(k-\sum_{i=1}^{k}x_{i})+\sum_{i=k+1}^{2k}x_{i}$, where $\boldsymbol{x}=(x_1,x_2,\dots,x_{2k+1})$.
Under the domain $\mathcal{C}=\{\boldsymbol{x}\in[0,1]^{2k+1}: \sum_{i=1}^{2k+1}x_{i}\le k\}$, \citet{hassani2017gradient} also
verified that $\boldsymbol{x}_{loc}=(\overbrace{1,1,\dots,1}^{k},0,\dots,0)$ is a local maximum with $(1/2+1/(2k))$-approximation to the global maximum. Thus, if start at $\boldsymbol{x}_{loc}$, theoretically Gradient Ascent \citep{hassani2017gradient} will get stuck at this local maximum point. In our experiment, we set $k=25$ and consider a  Gaussian noise, i.e., $[\widetilde{\nabla}f(x)]_{i}=[\nabla f(x)]_{i}+\delta\mathcal{N}(0,1)$ for any $i\in[2k]$ where $\delta=0.01$.  

 First, we set the initial point of GA(5) and BGA(5) to be $\boldsymbol{x}_{loc}$. From Figure~\ref{graph1}, we observe that GA(5) stays at $\boldsymbol{x}_{loc}$ as expected. Instead, BGA(5) escapes the local maximum $\boldsymbol{x}_{loc}$ and achieves near-optimal objective values. Then, we run both GA(5) and BGA(5) from the origin and present the results in Figure~\ref{graph2}. It shows that GA(5), starting from the origin, performs much better than the counterpart from a local maximum. Compared to GA(5), BGA(5) from origin converges to the optimal point $\boldsymbol{x}^{*}=(0,\dots,0,\overbrace{1,1,\dots,1}^{k+1})$ more rapidly. Both \cref{graph1} and \cref{graph2} show that BGA(5) also performs better than Frank-Wolfe-type algorithms with respect to the convergence rate and the objective value. 

\textbf{Non-Monotone Setting:} Recently, \citet{chen2023continuous} has presented a special \emph{non-monotone} continuous DR-submodular function $g_{k}$ , which follows from the multi-linear extension of a regularized coverage function, where   
$g_{k}(\boldsymbol{x})=k+1-(1-x_{2k+1})\prod_{i=1}^{k}(1-x_{i})-(1-x_{2k+1})(k-\sum_{i=1}^{k}x_{i})-\sum_{i=1}^{k}x_{i}-x_{2k+1}$. Moreover, \citet{chen2023continuous} showed that $\widetilde{\boldsymbol{x}}_{loc}=(\overbrace{1,1,\dots,1}^{2k},0)$ is a stationary point over the constraint  $\mathcal{C}=\{\boldsymbol{x}:\boldsymbol{x}\in[0,1]^{2k+1}\}$ and $\frac{g_{k}(\widetilde{\boldsymbol{x}}_{loc})}{\max_{\boldsymbol{x}\in\mathcal{C}}g_{k}(\boldsymbol{x})}\le\frac{1}{k}$. As a result, Gradient Ascent starting at $\widetilde{\boldsymbol{x}}_{loc}$ will be stuck at this point, resulting in a bad approximation to the global maximum when $k$ is large. Like the monotone case, we set $k=25$ and consider a standard Gaussian noise with $\sigma=0.01$ in our experiments. 

Firstly, we report the results of GA(5) starting from $\widetilde{\boldsymbol{x}}_{loc}$, BGA(5) from $\widetilde{\boldsymbol{x}}_{loc}$, Measured FW, Measured FW-VR, and Non-mono FW in Figure~\ref{graph3}. As we expect, GA(5) stays at the stationary point $\widetilde{\boldsymbol{x}}_{loc}$ with a  bad $0.04$-approximation guarantee. Instead, BGA(5) escapes the $\widetilde{\boldsymbol{x}}_{loc}$ and finally achieves $0.485$-approximation to the global maximum $\boldsymbol{x}^{*}=(\overbrace{0,\dots,0}^{2k},1)$. Then, we show the outcomes about GA(5) and BGA(5) from the origin in Figure~\ref{graph4}. Surprisingly, GA(5) from the origin 
approaches the optimal value, which is much better than GA(5) from $\widetilde{\boldsymbol{x}}_{loc}$. Due to the down-closed property of $\mathcal{C}$, it is foreseeable that Measured FW-VR achieves better function value than the BGA(5) at the final stage. Non-mono FW also performs better than BGA(5) in both Figure~\ref{graph3} and Figure~\ref{graph4}.
\begin{figure*}[t]
\vspace{-1.0em}
\centering
\subfigure[Monotone Movie Recommendation \label{graph31}]{\includegraphics[width=0.45\linewidth]{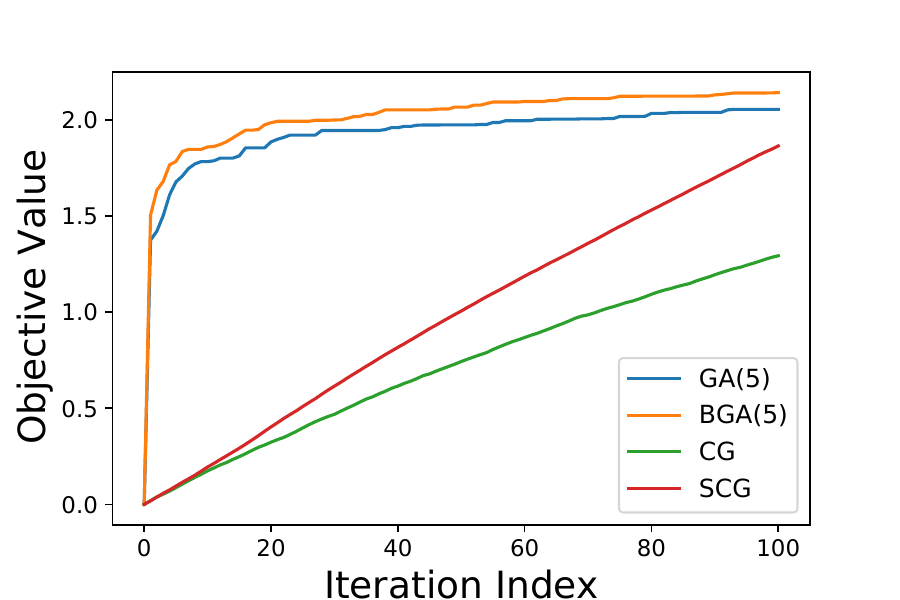}}
\subfigure[Non-Monotone Movie Recommendation\label{graph32}]{\includegraphics[width=0.45\linewidth]{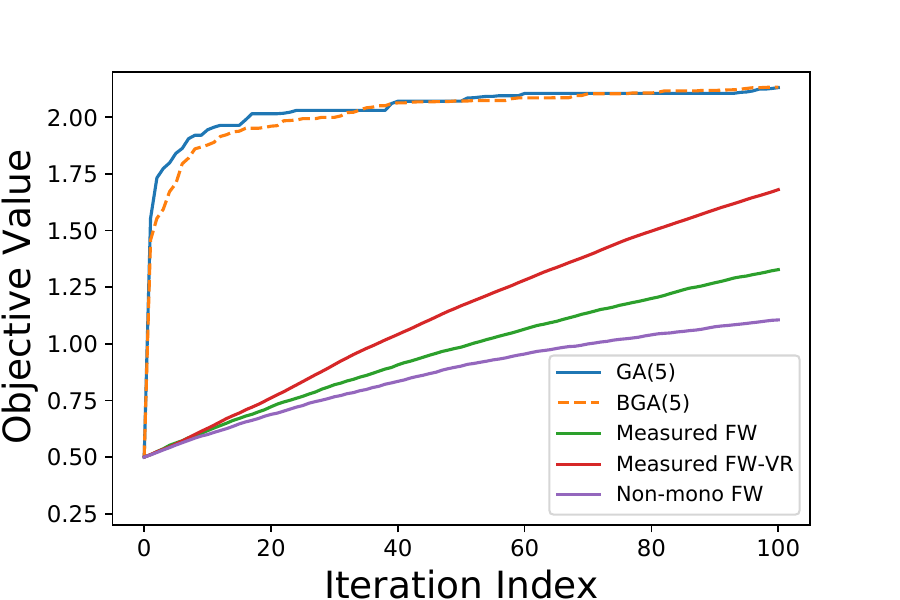}}
\caption{\cref{graph31} shows the performance of \textbf{GA(5)},\textbf{BGA(5)},\textbf{CG},and \textbf{SCG} in monotone movie recommendation task. In \cref{graph32}, we  report the results of \textbf{GA(5)},\textbf{BGA(5)},\textbf{Measured FW},\textbf{Measured FW-VR} and \textbf{Non-mono FW} in non-monotone
movie recommendation.} 
\vspace{-1.0em}
\end{figure*}
\subsubsection{Movie Recommendation}\label{sec:mov.offline}
\ \ \  \textbf{Monotone Setting:} We consider a movie recommendation task~\citep{stan2017probabilistic} with a part of MovieLens data set~\citep{harper2015movielens}. This dataset consists of 5-star ratings by $\mathcal{U}=1200$ users for $M=50$ movies. All Ratings are made with half-star increment.

 Let $r_{u,m}$ denote the rating of user $u$ for movie $m$. For each user $u$, we consider a well-motivated facility location objective function $f_{u}(S)=\max_{m\in S} r_{u,m}$ where $S$ is any subset of the movies with $f(\emptyset)=0$. Such a function shows how much user $u$  is satisfied by a subset $S$ of the movies. To quantify the satisfaction of all users for each set of movies $S$, we naturally investigate the average $f(S)=\frac{1}{|\mathcal{U}|}\sum_{u\in\mathcal{U}}f_{u}(S)$, where $\mathcal{U}$ represents the set of all users in dataset.  Like \citet{mokhtari2020stochastic} and \citet{zhang2023communication}, we consider the multi-linear extension of $f(S)$, that is, $F(\x)=\sum_{S}f(S)\prod_{i\in S}x_{i}\prod_{j\notin S}(1-x_{j})$ and constraint $\P=\{\boldsymbol{x}\in \mathbb{R}^{n}_{+} | \boldsymbol{A}\boldsymbol{x}\le\boldsymbol{b}, \boldsymbol{0}\le\boldsymbol{x}\le\boldsymbol{u}, \boldsymbol{A}\in \mathbb{R}^{m\times n}_{+}, \boldsymbol{b}\in \mathbb{R}_{+}^{m}\}$, where the matrix $\boldsymbol{A}$ is set as a random matrix with entries uniformly distributed in $[0,1]$, $\boldsymbol{b}=\boldsymbol{u}=\boldsymbol{1}$ and $m=\lfloor0.2M\rfloor$. It is easily verified that $F(\x)$ is a monotone continuous DR-submodular function. Our objective is to find the optimal allocation $\x$ over movies under the constraint $\P$, i.e., $\max_{\x\in\P}F(\x)$. In our experiment, we set a standard Gaussian noise for gradient.

\textbf{Non-Monotone Setting:} We investigate a different objective $G(\x)$, which adds a linear regularization in $F(\x)$. In other words, $G(\x)=F(\x)+\lambda(k-\sum_{i=1}^{M}x_{i})$ where $\x=(x_{1},\dots,x_{M})$ and $\lambda=0.1$. We can show $G(\x)$ is a continuous non-monotone DR-submodular function. To ensure $G(\x)\ge 0$, we consider a new constraint $\P_{1}=\{\boldsymbol{x}\in \mathbb{R}^{n}_{+} | \boldsymbol{A}\boldsymbol{x}\le\boldsymbol{b},\sum_{i=1}^{M}x_{i}\le k, \boldsymbol{0}\le\boldsymbol{x}\le\boldsymbol{u}, \boldsymbol{A}\in \mathbb{R}^{m\times n}_{+}, \boldsymbol{b}\in \mathbb{R}_{+}^{m}\}$, where the matrix $\boldsymbol{A}$ is set as a random matrix with entries uniformly distributed in $[0,1]$, $\boldsymbol{b}=\boldsymbol{u}=\boldsymbol{1}$ $m=\lfloor0.2M\rfloor$ and $k=5$. A standard Gaussian noise is also considered for gradient in solving $\max_{\x\in\P_{1}}G(\x)$.

As shown in \cref{graph31}, our BGA(5) performs better than both GA(5) and 
Frank-Wolfe-type algorithms with respect to the convergence rate and the objective value. Compared with CG, SCG is more robust to the gradient noise. In \cref{graph32}, our BGA(5) achieves nearly the same objective value with GA(5) after $60$-th iteration, both of which efficiently exceed Measured FW, Measured FW-VR as well as Non-mono FW. Among all Frank-Wolfe-type
algorithms, Non-mono FW shows the lowest objective value. Similarly, Measured FW-VR is more robust to Measured FW.
\subsection{Online Settings}

We also consider Online DR-submodular Maximization. Here, we present a list of algorithms to be compared:
\begin{itemize}
\item\textbf{Meta-Frank-Wolfe~($\alpha$-Meta-FW)}: We consider Algorithm 1 in \citep {chen2018online} and initialize $T^{\alpha}$ online gradient descent oracles~\citep{zinkevich2003online,hazan2019introduction} with step size $1/\sqrt{T}$.
\item\textbf{Variance-reduced Meta-Frank-Wolfe~($\alpha$-Meta-FW-VR)}: We consider Algorithm 1 in \citep{chen2018projection} with the $\rho_{t}=1/(t+3)^{2/3}$ and $T^{\alpha}$ online gradient descent oracles with step size $1/\sqrt{T}$.
\item\textbf{Mono-Frank-Wolfe~(Mono-FW)}: We consider Algorithm 1 in \citep{zhang2019online} with the $K=T^{3/5}$ and $Q=T^{2/5}$.
\item\textbf{Bandit-Frank-Wolfe~(Bandit-FW)}: We consider Algorithm 2 in \citep{zhang2019online} with the $L=T^{7/9}$ and $K=T^{2/3}$.
\item\textbf{Variance-reduced Measured-Meta-Frank-Wolfe~($\alpha$-Measured-MFW-VR)}: We consider Algorithm 1 in \citep{zhang2023online} with the $\rho_{t}=1/(t+3)^{2/3}$ and $T^{\alpha}$ online gradient descent oracles with step size $1/\sqrt{T}$.
\item\textbf{Measured-Meta-Frank-Wolfe~($\alpha$-Measured-MFW)}: We consider a variant of Algorithm 1 in \citep{zhang2023online} without variance reduction technique and initialize $T^{\alpha}$ online gradient descent oracles~\citep{zinkevich2003online,hazan2019introduction} with step size $1/\sqrt{T}$.
\item\textbf{Mono-Measured-Frank-Wolfe~(Mono-MFW)}: We consider Algorithm 2 in \citep{zhang2023online} with the $K=T^{3/5}$ and $Q=T^{2/5}$.
\item\textbf{Bandit-Measured-Frank-Wolfe~(Bandit-MFW)}: We consider Algorithm 3 in \citep{zhang2023online} with the $L=T^{7/9}$ and $K=T^{2/3}$.
\item\textbf{Non-monotone Meta-Frank-Wolfe~(Non-mono-MFW)}: We consider Algorithm 2 in \citep{mualem2023resolving} with the $L=50$ and $\epsilon=0.01$.
\item\textbf{Online Gradient Ascent~(OGA($B$))}: The delayed gradient ascent algorithm in \citep{quanrud2015online} with step size $1/\sqrt{T}$. We use $B$ independent samples to estimate $\nabla f_{t}(\boldsymbol{x}_{t})$ at each round.
\item\textbf{Online Boosting Gradient Ascent~(OBGA($B$))}:  We consider Algorithm~\ref{alg:2} with the step size $\eta_{t}=1/\sqrt{T}$ and use the average of $B$ independent samples to estimate the gradient at each round.
\item\textbf{Boosting Bandit Gradient Ascent~(Bandit-BGA)}: We consider Algorithm~\ref{algo:BBGA} with the step size $\eta=O(T^{-4/5})$ and $\lambda=O(T^{-1/5})$.
\end{itemize}
\begin{figure*}[t]
\vspace{-1.0em}
\centering
\subfigure[Monotone Case \label{graph41}]{\includegraphics[width=0.3\linewidth]{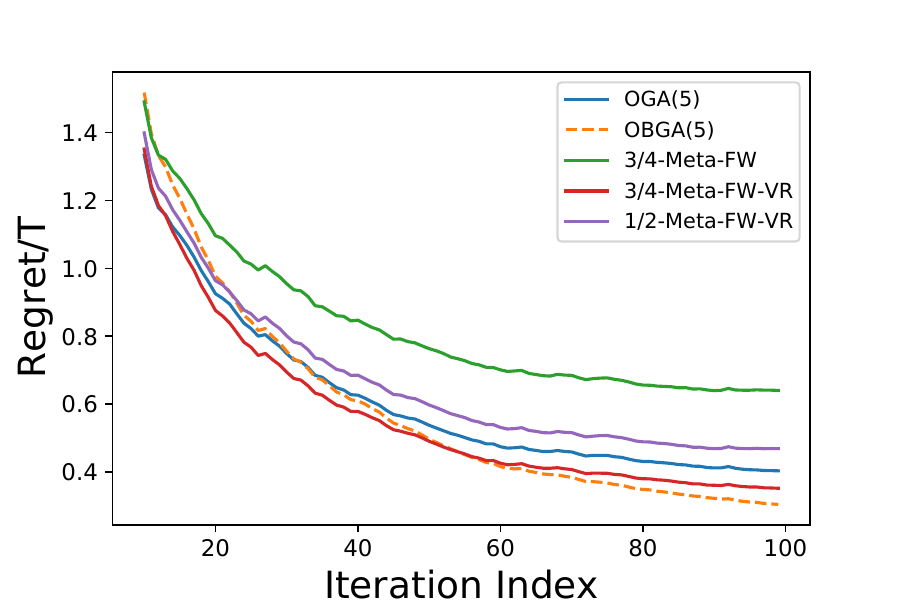}}
\subfigure[Delayed Monotone Case\label{graph42}]{\includegraphics[width=0.3\linewidth]{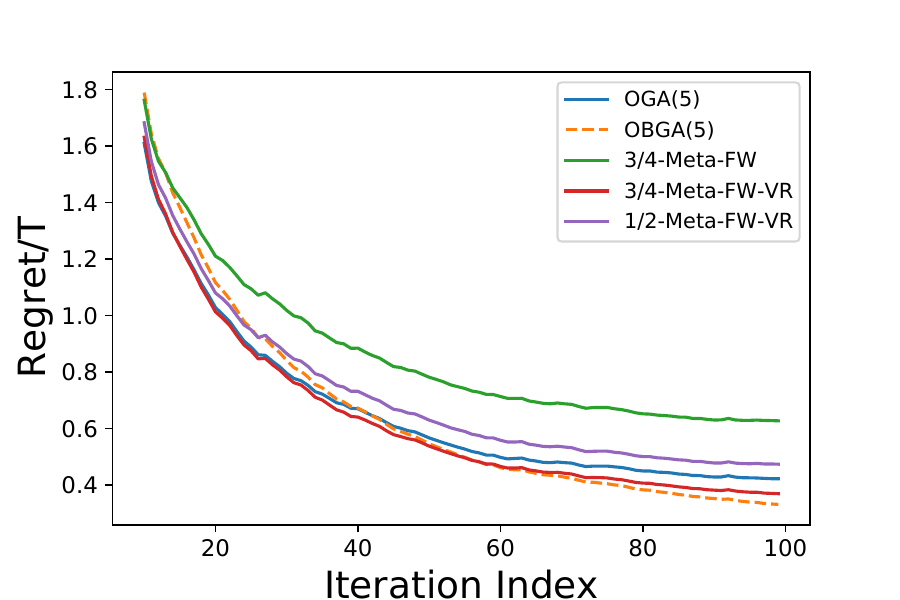}}
\subfigure[Bandit Monotone Case\label{graph43}]{\includegraphics[width=0.3\linewidth]{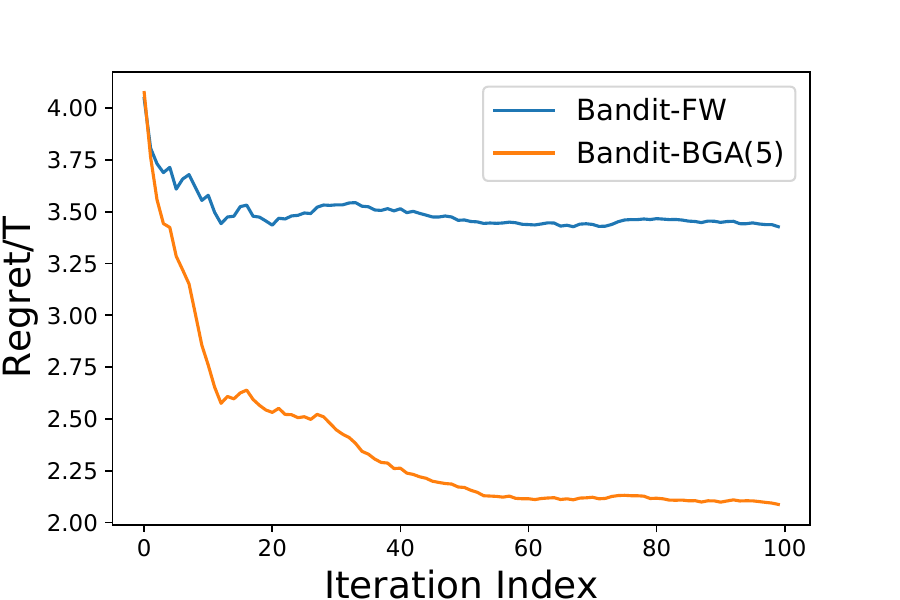}}

\subfigure[Non-Monotone Case\label{graph44}]{\includegraphics[width=0.3\linewidth]{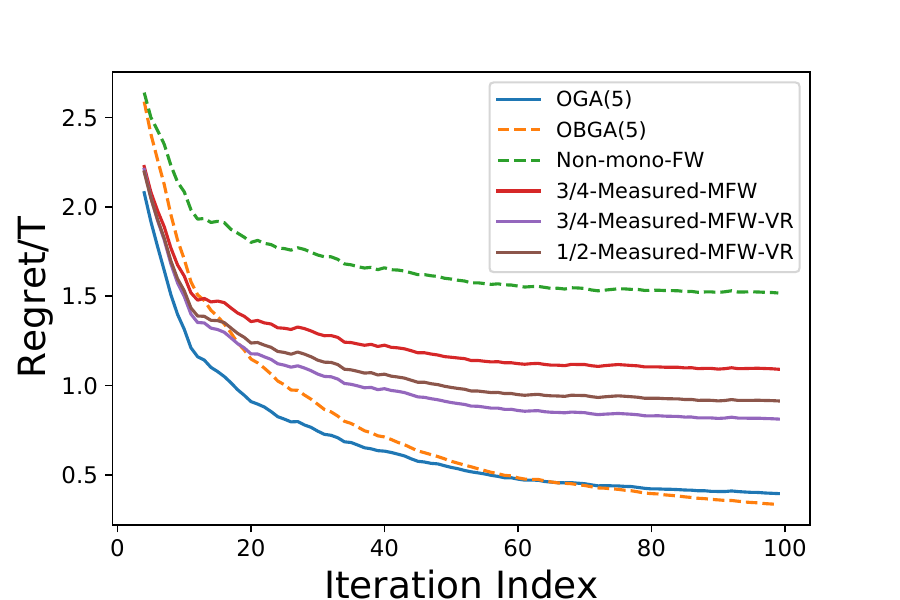}}
\subfigure[Delayed Non-Monotone Case~\label{graph45}]{\includegraphics[width=0.3\linewidth]{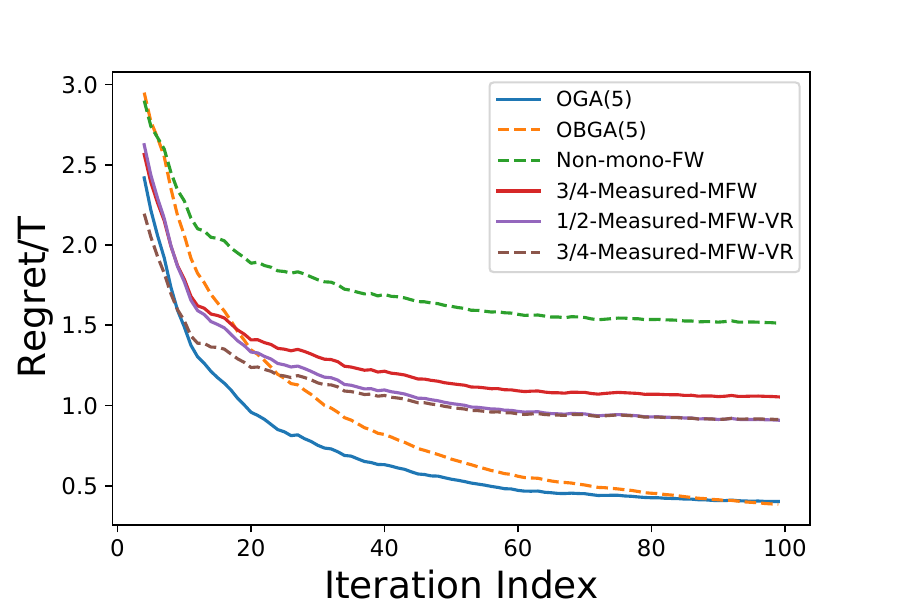}}
\subfigure[Bandit Non-Monotone Case\label{graph46}]{\includegraphics[width=0.3\linewidth]{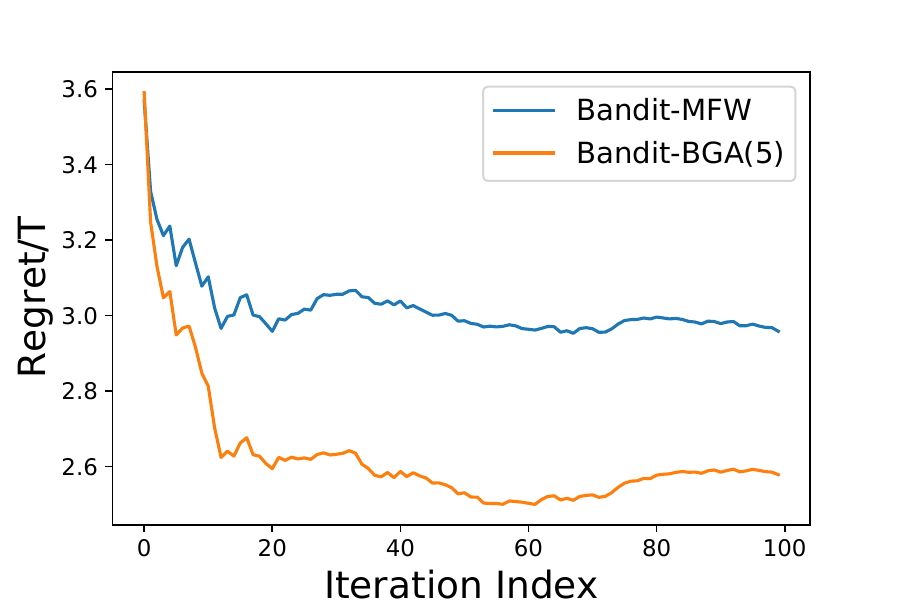}}
\caption{In ~\cref{graph41}-\ref{graph43}, we report the results for the online monotone movie recommendation task under full information, delayed feedback and bandit feedback. Similarly, ~\cref{graph44}-\ref{graph46} show the results of three different scenarios about online non-monotone movie recommendation tasks.}
\vspace{-1.0em}
\end{figure*}
\begin{table}[t]
	\renewcommand\arraystretch{1.35}
	\centering
	\caption{\small \cref{tab:online_movie_mono} shows the final $(1-1/e)$-Regret ratio and running time of online monotone movie recommendation. Note that `\textbf{Feedback Type}' means the form of objectives revealed by the environment during the process of online learning, `Full Feedback', `Delayed Feedback' and `Bandit Feedback' means that the object function is returned in full, delayed and bandit setting respectively. `\textbf{$(1-1/e)$-Regret Ratio}' means the ratio between $(1-1/e)$-Regret and timestamp at the $100$-th iteration, where we use a $500$-round continuous greedy method, namely, Algorithm 1 in \citep{bian2017guaranteed} as baseline to compute the $(1-1/e)$-regret.}
	\vspace{0.5em}
	\resizebox{0.8\textwidth}{!}{
		\setlength{\tabcolsep}{1.0mm}{
			\begin{tabular}{c|c|c|c}
				\toprule[1.5pt]
				\textbf{Feedback Type} &\textbf{Algorithm}& \textbf{$(1-1/e)$-Regret Ratio}&\textbf{Running time(seconds)}\\
				\hline
				\multirow{6}{*}{Full Feedback}& OGA(5)&0.404&18.14s\\
    \cline{2-4}
				~&\textbf{OBGA(5)}&\textbf{0.305}&\textbf{17.80s}\\
    \cline{2-4}
   ~&3/4-Meta-FW&0.641&113.29s\\
    \cline{2-4}
    ~&3/4-Meta-FW-VR&0.353&113.31s\\
    \cline{2-4}
 ~&Mono-FW&1.060&3.35s\\ 
    \cline{2-4}
 ~&1/2-Meta-FW-VR&0.469&36.60s\\
    \hline
     \hline
    \multirow{5}{*}{Delayed Feedback}&OGA(5)&0.422&18.12s\\
    \cline{2-4}
				~&\textbf{OBGA(5)}&\textbf{0.331}&\textbf{17.80s}\\
    \cline{2-4}
 ~&3/4-Meta-FW&0.627&113.26s\\
\cline{2-4}
~&3/4-Meta-FW-VR&0.369&113.34s\\
\cline{2-4}
~&1/2-Meta-FW-VR&0.473&36.60s\\
    \hline
     \hline
\multirow{2}{*}{Bandit Feedback}&\textbf{Bandit-BGA(5)}&\textbf{2.088}&\textbf{0.057s}\\
\cline{2-4}
~&Bandit-FW&3.428&0.116s\\
\midrule[1.5pt]
\end{tabular}}}
	\vspace{-1.5em}
	\label{tab:online_movie_mono}
\end{table}
\begin{table}[t]
	\renewcommand\arraystretch{1.35}
	\centering
 \caption{\small\cref{tab:online_movie_non_mono} shows the final regret ratio and running time of online non-monotone movie recommendation. Note that `\textbf{Feedback Type}' means the form of objectives revealed by the environment during the process of online learning, `Full Feedback', `Delayed Feedback' and `Bandit Feedback' means that the object function is returned in full, delayed and bandit setting respectively. `\textbf{Regret Ratio}' means the ratio between regret and time horizon at the $100$-th iteration, where we use a $500$-round deterministic Measured Frank Wolfe, namely, Algorithm 2 in \citep{mitra2021submodular+} as baseline to compute the regret.}
	\vspace{0.5em}
	\resizebox{0.77\textwidth}{!}{
		\setlength{\tabcolsep}{1.0mm}{
			\begin{tabular}{c|c|c|c}
				\toprule[1.5pt]
				\textbf{Feedback Type} &\textbf{Algorithm}& \textbf{Regret Ratio}&\textbf{Running time(seconds)}\\
				\hline
				\multirow{7}{*}{Full Feedback}& OGA(5)&0.394&19.67s\\
    \cline{2-4}
				~&\textbf{OBGA(5)}&\textbf{0.334}&\textbf{19.27s}\\
    \cline{2-4} 
    ~&Non-mono MFW&1.517&197.83s\\
    \cline{2-4}
   ~&3/4-Measured-MFW&1.090& 122.49s\\
    \cline{2-4}
    ~&1/2-Measured-MFW-VR&0.812&122.68s\\
    \cline{2-4}
 ~&Mono-FW&1.238&3.61s\\ 
    \cline{2-4}
 ~&3/4-Meta-FW-VR&0.913&39.66s\\
    \hline
     \hline
    \multirow{6}{*}{Delayed Feedback}&OGA(5)&0.402&19.71s\\
    \cline{2-4}
				~&\textbf{OBGA(5)}&\textbf{0.384}&\textbf{19.27s}\\
    \cline{2-4}
 ~&3/4-Measured-MFW&1.054&122.80s\\
\cline{2-4}
 ~&Non-mono MFW&1.513&197.63s\\
\cline{2-4}
~&3/4-Measured-MFW-VR&0.812&122.81s\\
\cline{2-4}
~&1/2-Measured-MFW-VR&0.909&39.66s\\
    \hline
     \hline
\multirow{2}{*}{Bandit Feedback}&\textbf{Bandit-BGA(5)}&\textbf{2.578}&\textbf{0.059s}\\
\cline{2-4}
~&Bandit-MFW&2.958&0.123s\\
\midrule[1.5pt]
\end{tabular}}}
	\vspace{-1.5em}
	\label{tab:online_movie_non_mono}
\end{table}

\textbf{Movie Recommendation:} Like section~\ref{sec:mov.offline}, we consider the  facility location objective function for each user $u$, i.e., $f_{u}(S)=\max_{m\in S}r_{u,m}$ where $r_{u,m}$ denote the rating of user $u$ for movie $m$. Then,  we split the first $T\times b$ users into disjoint and equally-sized sets $\mathcal{U}_{1},\dots,\mathcal{U}_{T}$, so $|\mathcal{U}_{i}|=b$ for any $i\in [T]$. At each round $t\in[T]$, the environment/adversary reveals the multi-linear extension of function $f_{t}(S)=\frac{1}{|\mathcal{U}_{t}|}\sum_{u\in\mathcal{U}_{t}}f_{u}(S)$ to the learner, that is, $F_{t}(\x)=\sum_{S}f_{t}(S)\prod_{i\in S}x_{i}\prod_{j\notin S}(1-x_{j})$. As for non-monotone cases, we also consider the $G_{t}(\x)=F_{t}(\x)+\lambda(k-\sum_{i=1}^{M}x_{i})$ where $b=15$, $k=5$ and $\lambda=0.1$. To efficiently find a solution for the $T$-round accumulative reward with theoretical guarantee, we consider the cardinality constraint $P=\{\boldsymbol{x}\in \mathbb{R}^{n}_{+} |\sum_{i=1}^{M}x_{i}\le 5, \boldsymbol{0}\le\boldsymbol{x}\le\boldsymbol{1}\}$. In the experiments, we impose the Gaussian noise to the gradient, i.e., $\widetilde{\nabla}f_{t}(\x)=\nabla f_{t}(\x)+0.01*\mathcal{N}(0,\mathbf{I})$ or $\widetilde{\nabla}g_{t}(\x)=\nabla g_{t}(\x)+0.01*\mathcal{N}(0,\mathbf{I})$ for any $t\in[T]$, $i\in[N]$ and $\x\in[0,1]^{n}$, where $\mathcal{N}(0,\mathbf{I})$ is standard multivariate normal distribution.  To simulate the feedback delays, we generate a uniform random number $d_{t}$ from $\{1,2,3,4,5\}$ for the $t$-th round stochastic gradient information. We present the trend of the ratio between regret and time horizon in the \cref{graph41}-\ref{graph46}, and report the running time and the ratio at $100$-th iteration in \cref{tab:online_movie_mono}-\ref{tab:online_movie_non_mono}, where we use the results of deterministic Frank Wolfe algorithms with $500$ iterations as a baseline to compute the regret at each time horizon.

As shown in \cref{graph41},\ref{graph42},\ref{graph44} and \ref{graph45}, OBGA(5) performs better than OGA(5) and all other Frank-Wolfe-type algorithms at the final stage. Moreover, 3/4-Meta-FW and Non-mono-MFW show the worst regret ratio in the monotone  and general movie recommendation respectively. As for the bandit settings, our Bandit-BGA(5) efficiently lower the regret ratio compared with Bandit-Frank-Wolfe-type algorithms, i.e., Bandit-FW and Bandit-MFW. From \cref{tab:online_movie_mono}, our OBGA(5) can be 6 times faster than the best Frank-Wolfe-tyle algorithm `3/4-Meta-FW-VR' in monotone cases. Similarly, our OBGA(5) is more effective than the best non-monotone Frank-Wolfe-tyle algorithm `3/4-Measured-MFW-VR' according to \cref{tab:online_movie_non_mono}.
\subsection{Minimax Settings}
We also consider minimax optimization of convex-submodular functions. Here, we present a list of algorithms to be compared:
\begin{itemize}
    \item \textbf{Extra-gradient on Continuous Extension~(EGCE)}: We consider Algorithm 3 in \citep{adibi2022minimax} and initialize the step size $\gamma_{t}=O(1/\sqrt{T})$ where $T$ is the predefined total iterations.
    \item \textbf{Boosting Gradient Descent Ascent~(BGDA)}: Algorithm~\ref{minimax optimization montone} in this paper and we initialize the step size $\eta=O(1/\sqrt{T})$ where $T$ is the predefined total iterations.
\end{itemize}
\begin{figure*}[t]
\vspace{-1.0em}
\centering
\subfigure[Convex-facility Location(Mono)\label{graph_minimax_1}]{\includegraphics[width=0.42\linewidth]{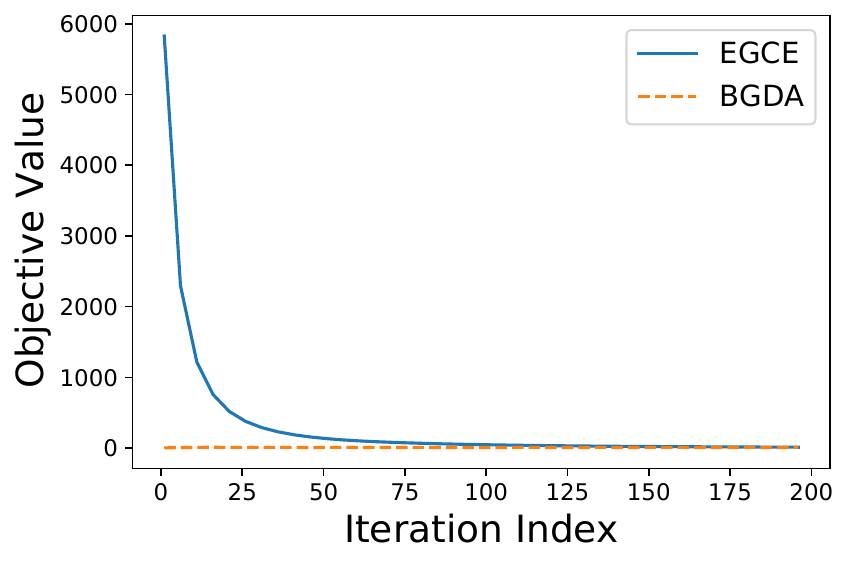}}
\subfigure[Convex-facility Location(Non-Mono)\label{graph_minimax_2}]{\includegraphics[width=0.42\linewidth]{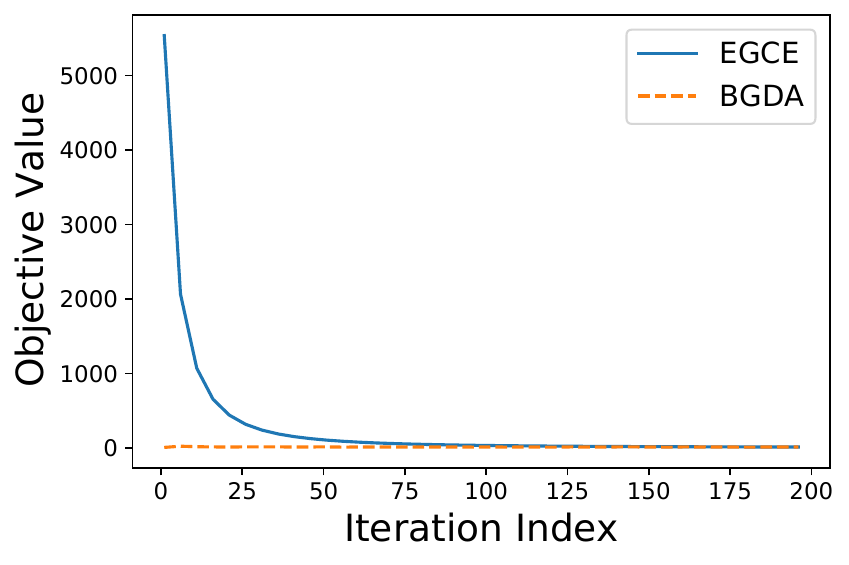}}

\subfigure[Item Recommendation(Mono) \label{graph_minimax_3}]{\includegraphics[width=0.42\linewidth]{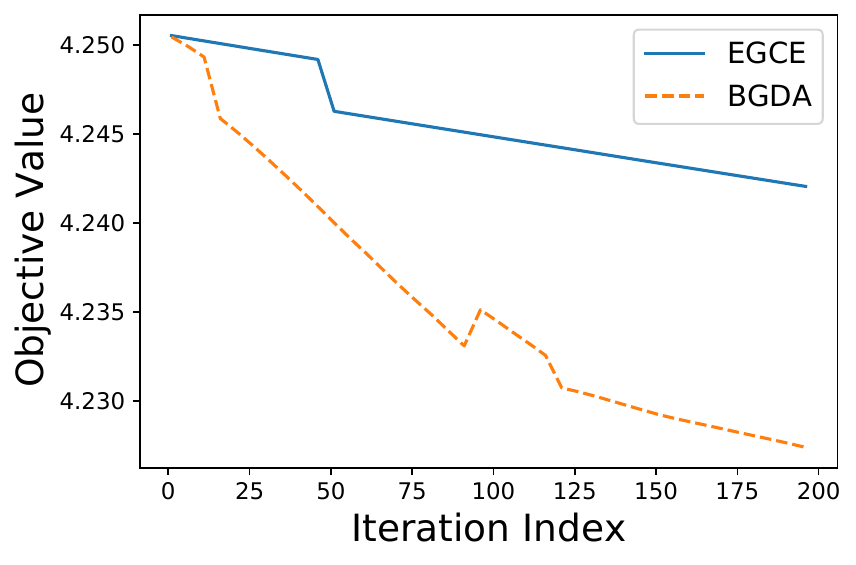}}
\subfigure[Item Recommendation(Non-Mono)\label{graph_minimax_4}]{\includegraphics[width=0.42\linewidth]{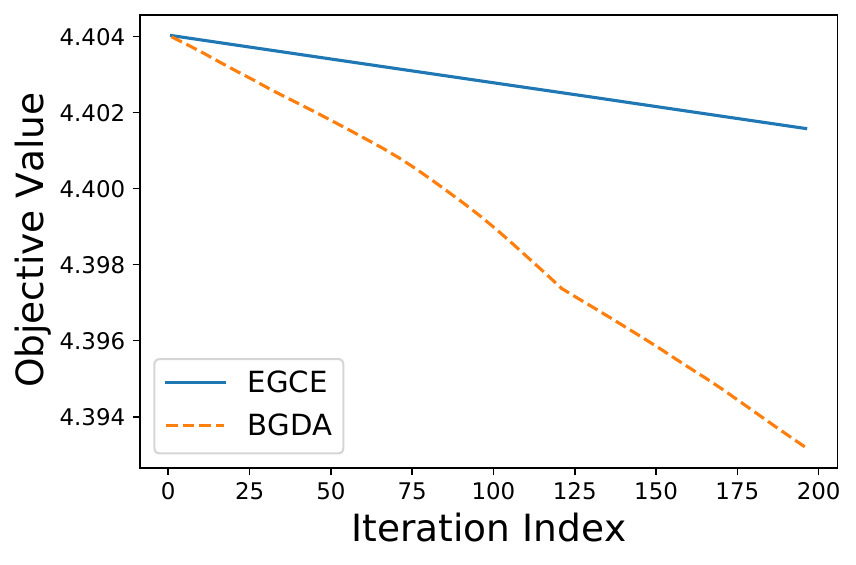}}
\caption{In \cref{graph_minimax_1}-\ref{graph_minimax_2}, we show the comparison of our proposed methods for Convex-facility Location. Similarly, the results about attack for item recommendation are presented in \cref{graph_minimax_3}-\ref{graph_minimax_4}.}
\vspace{-1.0em}
\end{figure*}
\textbf{Convex-facility Location:} In this setting, we consider an objective function $f:\R^{d}\times 2^{V}\rightarrow\R_{+}$ where $f(\x,S)=\sum_{i=1}^{n}\max_{j\in S}f_{i,j}(\x)+g(\x)$. if $f_{i,j}$ and $g$ is a convex function w.r.t. the continuous variable $\x$, we can easily verify that the $f$ is a convex-discrete monotone submodular function. Like what \cite{adibi2022minimax} do, we split the $d$-dimension vector $\x$ into n different parts, namely, $\x=[\x_{1};\dots;\x_{n}]$ where $\x_{i}\in\R^{m}$ and $m\times n=d$. In our experiments, we set $f_{i,j}(\x)=w_{i,j}\x_{i}^{T}\x_{j}$ where each $w_{ij}$ is randomly chosen from $[0,1]$. Furthermore, we consider the regularized term $g(\x)=\lambda(\sum_{i=1}^{n}\|\x_{i}\|^{2})^{-1}$ and set  the constraint about $\x$ as $\mathcal{C}=\{\x|\|\x_{i}\|\le 1,\forall i\in[n]\}$. Our objective is to optimize the multi-linear extension of $f$, i.e., 
\begin{equation*}
\min_{\x\in\C}\max_{\mathbf{y}\in\mathcal{K}}F(\x,\mathbf{y}),
\end{equation*} where $F(\x,\mathbf{y})=\sum_{S}f(\x,S)\prod_{i\in S}y_{i}\prod_{j\notin S}(1-y_{j})$ and $\mathcal{K}=\{\y\in[0,1]^{n}| \sum_{i=1}^{n}y_{i}\le k\}$. As for the non-monotone cases, we consider adding a linear term in $f(\x,S)$, namely, $g(\x,S)=f(\x,S)+k-|S|$. As a result, the multi-linear extension $G(\x,\y)$ of $g(\x,S)$ can be rewritten as $G(\x,\y)=F(\x,\y)+k-\sum_{i=1}^{n}y_{i}$. We then show the comparison of our BGDA and EGCE in \cref{graph_minimax_1}-\ref{graph_minimax_2}, where $m=10$, $n=30$,$k=5$. Note that it is hard to find the optimal $\y^{*}=\max_{\y\in\mathcal{K}}F(\x,\y)$ for any fixed $\x\in\mathcal{C}$. Thus, the reported objective value in \cref{graph_minimax_1} and \cref{graph_minimax_2} is exactly the value $f(\x,\text{GREEDY}(f,\x,k))$ and $g(\x,\text{DISTORTED-GREEDY}(g,\x,k))$ respectively, where `DISTORTED-GREEDY' is Algorithm 1 in \citet{harshaw2019submodular} and `GREEDY' is the classical greedy method. As we observe in \cref{graph_minimax_1} and \cref{graph_minimax_2}, our BGDA converges faster than EGCE.

\textbf{Adversarial Attack for Item Recommendation:} We consider designing an adversarial attack for a movie recommendation task, where there is a completed
rating matrix $R$ whose entry $r_{u,m}$ represents the estimated rating of user $u$ for movie $m$. Given a rating matrix $R$, we want to chooses $k$ movies via maximizing the well-motivated facility location objective function, namely, $\max_{|S|\le k}f(R,S)$ where $f(R,S)=\frac{1}{|\mathcal{U}|}\sum_{u\in\mathcal{U}}\max_{m\in S}r_{u,m}$ and $\mathcal{U}$ stands for the set of all users. The attacker’s goal is to slightly perturb the rating matrix $R$ to a matrix $R'$ such that the utility $\max_{|S|\le k}f(R',S)$ is minimized. That is, the attacker hope to tackle the following  minimax problem:
\begin{equation*}
    \min_{\|R'-R\|_{F}\le\epsilon}\max_{|S|\le k}f(R',S),
\end{equation*} where $\|\cdot\|_{F}$ is the Frobenius norm. Like convex-facility location, we run EGCE and BGDA on the multi-linear extension of $f(R',S)$. As for the non-monotone case, we also consider the multi-linear extension of $g(R',S)=f(R',S)+k-|S|$. In our experiments, we simulate a dataset about $100$ users for $50$ movies with each rating to be independently uniformly chosen from $[0,5]$ and set $k=10$ and $\epsilon=50*100*0.5*0.01$. As shown in \cref{graph_minimax_3}-\ref{graph_minimax_4}, our BGDA performs better than EGCE w.r.t. the convergence rate and objective value for both monotone and non-monotone recommendation attack.

\section{Conclusion}
In this paper, we design the non-oblivious function and leverage it to enhance the PGA method, thereby addressing the question posed at the outset of this article. By incorporating this innovative technical component, we obtain improved results across various settings pertaining to DR-Submodular functions. We believe that the non-oblivious function, together with the boosted PGA method, holds potential for wide-ranging applications in the realm of continuous submodular maximization, and can even be extended to discrete submodular maximization. As an illustration, \cite{pmlr-v202-wan23e} successfully applied the non-oblivious technique to submodular full-bandit problems through a specialized continuous DR-submodular extension.



\bibliography{references}

\appendix
\begin{appendices}
\section{Technical Lemmas}\label{Appendix:lemmas}
The following folklore lemma gives an upper bound and lower bound of $L$-smooth functions.
\begin{lemma}\label{smooth function bound}
    If $f$ is $L$-smooth, then for any $\bx$ and $\by$, we have
    \begin{align}
        f(\by)\leq f(\bx) + \langle\by -\bx,\nabla f(\bx)\rangle + \frac{L}{2}\|\by-\bx\|^2
    \end{align}
    and 
    \begin{align}
        f(\by)\geq f(\bx) + \langle\by -\bx,\nabla f(\bx)\rangle - \frac{L}{2}\|\by-\bx\|^2.
    \end{align}
\end{lemma}

Then we recall some lemmas about the projection operator and projected gradient ascent step.
\begin{lemma}[Bourbaki-Cheney-Goldstein inequality\citep{bertsekas2015convex}]\label{lemma:9}
For the projection $\mathcal{P}_{\mathcal{C}}(\boldsymbol{x})=\arg\min_{\boldsymbol{z}\in\mathcal{C}}\left\|\boldsymbol{z}-\boldsymbol{x}\right\|$,  we have
\begin{equation}\label{equ:41}
    \begin{aligned}
      \langle\mathcal{P}_{\mathcal{C}}(\boldsymbol{x})-\boldsymbol{x},\boldsymbol{z}-\mathcal{P}_{\mathcal{C}}(\boldsymbol{x})\rangle\ge 0, \forall \boldsymbol{z}\in\mathcal{C}.
    \end{aligned}
\end{equation}
\end{lemma}

\begin{lemma}[Gradient step]\label{lem:gradient step}
    Let $g(\bx)$ be any differentiable function, $\mathcal{C}$ be any convex body, $\eta\in \R$, $\bx\in \mathcal{C}$ and $\by \triangleq \mathcal{P}_{\mathcal{C}}(\bx-\eta \nabla g(\bx))$. For any $\bz\in \mathcal{C}$, it holds
    \begin{equation}\label{eq: gradient step}
        2\eta \langle \by-\bz, \nabla g(\bx)\rangle\leq \|\bx-\bz\|^2-\|\by-\bz\|^2-\|\by-\bx\|^2.
    \end{equation}
    Moreover, if $\by \triangleq \mathcal{P}_{\mathcal{C}}(\bx+\eta \nabla g(\bx))$. For any $\bz\in \mathcal{C}$, it holds
    \begin{equation}\label{eq: gradient step 2}
        2\eta \langle \by-\bz, \nabla g(\bx)\rangle\geq \|\by-\bz\|^2+\|\by-\bx\|^2-\|\bx-\bz\|^2.
    \end{equation}
\end{lemma}
\begin{proof}
    \begin{equation}\label{eq: gradient step 1}
       \begin{aligned}
           &\|\by-\bz\|^2
           \\&\leq \|\bx-\bz -\eta \nabla g(\bx)\|^2
           \\&= \|\bx-\bz\|^2 -2\eta\langle \bx-\bz, \nabla g(\bx)\rangle + \|\eta\nabla g(\bx)\|^2
           \\&=\|\bx-\bz\|^2 - 2\eta\langle \bx-\by, \nabla g(\bx)\rangle - 2\eta\langle \by-\bz, \nabla g(\bx)\rangle + \|\by-\bx\|^2
           \\&= \|\bx-\bz\|^2 - 2\langle \bx-\by, \bx-\by\rangle - 2\langle \bx-\by, \by-(\bx-\eta g(\bx))\rangle - 2\eta\langle \by-\bz, \nabla g(\bx)\rangle + \|\by-\bx\|^2
           \\&\leq \|\bx-\bz\|^2 - 2\|\by-\bx\|^2- 2\eta\langle \by-\bz, \nabla g(\bx)\rangle + \|\by-\bx\|^2
           \\&=\|\bx-\bz\|^2 - \|\by-\bx\|^2- 2\eta\langle \by-\bz, \nabla g(\bx)\rangle
       \end{aligned} 
    \end{equation}
    The second inequality is because of \cref{lemma:9}. (\ref{eq: gradient step}) comes immediately from (\ref{eq: gradient step 1}). \eqref{eq: gradient step 2} is obtained by substitute $\eta$ with $-\eta$ in \eqref{eq: gradient step}.
\end{proof}
\section{Proofs in Section~\ref{sec:non-oblivious}}\label{Appendix:A}
\subsection{Proof of \texorpdfstring{\cref{lemma:2}}{1}} \label{proof:lem1}
First, we review some basic inequalities for $\gamma$-weakly continuous DR-submodular function $f$.
\begin{lemma}\label{lemma:a1}
For a monotone, differentiable, and $\gamma$-weakly continuous DR-submodular function $f$, we have
\begin{enumerate}
\item For any $\boldsymbol{x}\le\boldsymbol{y}$, we have $\langle\boldsymbol{y}-\boldsymbol{x}, \nabla f(\boldsymbol{x})\rangle\ge \gamma(f(\boldsymbol{y})-f(\boldsymbol{x}))$ and $\langle \boldsymbol{y}-\boldsymbol{x}, \nabla f(\boldsymbol{y})\rangle\le\frac{1}{\gamma}(f(\boldsymbol{y})-f(\boldsymbol{x}))$.
\item For any $\boldsymbol{x},\boldsymbol{y}\in\mathcal{X}$, we also could derive $ \langle \boldsymbol{y}-\boldsymbol{x}, \nabla f(\boldsymbol{x})\rangle\ge\gamma f(\boldsymbol{x}\lor \boldsymbol{y})+\frac{1}{\gamma}f(\boldsymbol{x}\land \boldsymbol{y} )-(\gamma+\frac{1}{\gamma})f(\boldsymbol{x})$.
\end{enumerate}\end{lemma}
\begin{proof} First, according to the definition of DR-submodular function and monotone property in \cref{sec:pre}, we have $\nabla f(\boldsymbol{x})\ge\gamma\nabla f(\boldsymbol{y})$, if $\boldsymbol{x}\le\boldsymbol{y}$. Thus, for any $\boldsymbol{x}\le\boldsymbol{y}$, we have
\begin{equation}\label{equ:11}
    \begin{aligned}
    &f(\boldsymbol{y})-f(\boldsymbol{x})=\int_{0}^{1}\langle \boldsymbol{y}-\boldsymbol{x}, \nabla f(\boldsymbol{x}+z(\boldsymbol{y}-\boldsymbol{x}))\rangle \mathrm{d}z\le \frac{1}{\gamma}\langle \boldsymbol{y}-\boldsymbol{x}, \nabla f(\boldsymbol{x}))\rangle,\\
    &f(\boldsymbol{y})-f(\boldsymbol{x})=\int_{0}^{1}\langle \boldsymbol{y}-\boldsymbol{x}, \nabla f(\boldsymbol{x}+z(\boldsymbol{y}-\boldsymbol{x}))\rangle \mathrm{d}z\ge \gamma\langle \boldsymbol{y}-\boldsymbol{x}, \nabla f(\boldsymbol{y})\rangle,
    \end{aligned} 
\end{equation} where these two inequalities follow from  $\boldsymbol{y}\ge\boldsymbol{x}+z(\boldsymbol{y}-\boldsymbol{x})\ge\boldsymbol{x}$ such that $\frac{1}{\gamma}\nabla f(\boldsymbol{x})\ge\nabla f(\boldsymbol{x}+z(\boldsymbol{y}-\boldsymbol{x}))\ge\gamma\nabla f(\boldsymbol{y})$ for any $z\in[0,1]$. We finish the proof of the first inequality in \cref{lemma:a1}.

Then, from \eqref{equ:11}, we could derive that
\begin{equation}\label{equ:12}
\begin{aligned}
 &\langle \boldsymbol{y}\lor \boldsymbol{x} -\boldsymbol{x}, \nabla f(\boldsymbol{x})\rangle \ge \gamma f(\boldsymbol{y}\lor \boldsymbol{x})-\gamma f(\boldsymbol{x}), \\
 &\langle \boldsymbol{x}\land \boldsymbol{y} -\boldsymbol{x}, \nabla f(\boldsymbol{x})\rangle \ge \frac{1}{\gamma} (f(\boldsymbol{x}\land\boldsymbol{y})-f(\boldsymbol{x})),
 \end{aligned}
 \end{equation} where $ \boldsymbol{y}\lor \boldsymbol{x}\ge\boldsymbol{x}$ and $\boldsymbol{x}\land \boldsymbol{y}\le\boldsymbol{x}$.

Merging the two equations in \eqref{equ:12}, we have, for any $\boldsymbol{x}$ and $\boldsymbol{y}\in\mathcal{X}$,
\begin{equation}\label{equ:13}
\begin{aligned}
\langle\boldsymbol{y}-\boldsymbol{x}, \nabla f(\boldsymbol{x})\rangle
  &= \langle \boldsymbol{y}\lor \boldsymbol{x} -\boldsymbol{x}, \nabla f(\boldsymbol{x})\rangle+\langle \boldsymbol{x}\land \boldsymbol{y} -\boldsymbol{x}, \nabla f(\boldsymbol{x})\rangle\\
  &\ge\gamma f(\boldsymbol{x}\lor \boldsymbol{y})+\frac{1}{\gamma}f(\boldsymbol{x}\land \boldsymbol{y} )-(\gamma+\frac{1}{\gamma})f(\boldsymbol{x}),  
\end{aligned}
\end{equation} 
where $\boldsymbol{x}\land \boldsymbol{y}+ \boldsymbol{x}\lor \boldsymbol{y}=\boldsymbol{x}+\boldsymbol{y}$.
Thus, we prove the second inequality in \cref{lemma:a1}. 
\end{proof}

Next, with the~\cref{lemma:a1}, we prove the~\cref{lemma:2}. 

\begin{proof}From \cref{equ:13}, if $\boldsymbol{x}$ is a stationary point of $f$ in domain $\mathcal{C}$, we have $(\gamma+\frac{1}{\gamma})f(\boldsymbol{x})\ge\gamma f(\boldsymbol{x}\lor \boldsymbol{y})+\frac{1}{\gamma}f(\boldsymbol{x}\land \boldsymbol{y} )  $ for any $\boldsymbol{y}\in\mathcal{C}$. Due to the monotone and non-negative property, $f(\boldsymbol{x})\ge\frac{\gamma^{2}}{\gamma^{2}+1}\max_{\boldsymbol{y}\in\mathcal{C}}f(\boldsymbol{y})$.
\end{proof}
\subsection{Proof of \texorpdfstring{\cref{lemma:3}}{2}} \label{proof:lem2}
\begin{proof} First, we obtain an inequality about $\langle \boldsymbol{x},\nabla F(\boldsymbol{x})\rangle$, i.e., 
\begin{equation}\label{equ:appendix1}
    \begin{aligned}
    \langle \boldsymbol{x},\nabla F(\boldsymbol{x})\rangle&=\int_{0}^{1} w(z)\langle \boldsymbol{x},\nabla f(z\cdot \boldsymbol{x})\rangle\mathrm{d}z\\ 
    &=\int^{1}_{0}w(z) \mathrm{d}f(z\cdot \boldsymbol{x})\\
    &=w(z)f(z\cdot \boldsymbol{x})|_{z=0}^{z=1}-\int_{0}^{1}f(z\cdot \boldsymbol{x})w'(z)\mathrm{d}z\\
    &\le w(1)f(\boldsymbol{x})-\int_{0}^{1}f(z\cdot \boldsymbol{x})w'(z)\mathrm{d}z.
    \end{aligned}
\end{equation}
Then, we also prove some properties about $\langle \boldsymbol{y},\nabla F(\boldsymbol{x})\rangle$, namely, 
\begin{equation}\label{equ:appendix2}
    \begin{aligned}
    \langle \boldsymbol{y},\nabla F(\boldsymbol{x})\rangle&=\int_{0}^{1} w(z)\langle \boldsymbol{y},\nabla f(z\cdot \boldsymbol{x})\rangle\mathrm{d}z\\ 
    &\ge\int^{1}_{0}w(z)\langle \boldsymbol{y}\lor(z\cdot \boldsymbol{x})-z\cdot \boldsymbol{x},\nabla f(z\cdot \boldsymbol{x})\rangle\mathrm{d}z\\
    &\ge\gamma\int^{1}_{0}w(z)(f(\boldsymbol{y}\lor(z\cdot \boldsymbol{x}))-f(z\cdot \boldsymbol{x}))\mathrm{d}z \\
    &\ge(\gamma\int_{0}^{1}w(z)\mathrm{d}z)f(\boldsymbol{y})-\int^{1}_{0}\gamma w(z)f(z\cdot \boldsymbol{x})\mathrm{d}z,
    \end{aligned}
\end{equation}
where the first inequality follows from $\boldsymbol{y}\ge\boldsymbol{y}\lor(z\cdot \boldsymbol{x})-z\cdot \boldsymbol{x}\ge\boldsymbol{0}$ and $\nabla f(z\cdot \boldsymbol{x})\ge\boldsymbol{0}$; the second one comes from the \cref{lemma:2}; and the final inequality follows from $f(\boldsymbol{y}\lor(z\cdot \boldsymbol{x}))\ge f(\boldsymbol{y})$.
    
Finally, putting above the inequality~\eqref{equ:appendix1} and inequality~\eqref{equ:appendix2} together, we have
\begin{equation}\label{equ:rev1}
 \begin{aligned}
    \langle\boldsymbol{y}-\boldsymbol{x}, \nabla F(\boldsymbol{x})\rangle&\ge(\gamma\int_{0}^{1}w(z)\mathrm{d}z)f(\boldsymbol{y})-w(1)f(\boldsymbol{x})+\int^{1}_{0}(w'(z)-\gamma w(z))f(z\cdot \boldsymbol{x})\mathrm{d}z\\
    &=(\gamma\int_{0}^{1}w(z)\mathrm{d}z)(f(\boldsymbol{y})-\frac{w(1)+\int^{1}_{0} (\gamma w(z)-w'(z))\frac{f(z\cdot \boldsymbol{x})}{f(\boldsymbol{x})}\mathrm{d}z}{\gamma\int_{0}^{1}w(z)\mathrm{d}z}f(\boldsymbol{x}))\\
    &=(\gamma\int_{0}^{1}w(z)\mathrm{d}z)(f(\boldsymbol{y})-\theta(w,f,\boldsymbol{x})f(\boldsymbol{x}))\\
    &\ge(\gamma\int_{0}^{1}w(z)\mathrm{d}z)(f(\boldsymbol{y})-\theta(w)f(\boldsymbol{x})),
    \end{aligned}
    \end{equation} where the final inequality follows from $\theta(w)=\max_{f,\boldsymbol{x}}\theta(w,f,\boldsymbol{x})$.
\end{proof}

\subsection{Proof of \texorpdfstring{\cref{thm:1}}{3}} \label{proof:thm1}
\begin{proof}
In this proof, we investigate the optimal value and solution about the following optimization problem:
\begin{equation}\label{equ:19}
   \begin{aligned}
    \min_{w}\theta(w)=\min_{w}\max_{f,\boldsymbol{x}}& \frac{w(1)+\int^{1}_{0} (\gamma w(z)-w'(z))\frac{f(z\cdot \boldsymbol{x})}{f(\boldsymbol{x})}\mathrm{d}z}{\gamma\int_{0}^{1}w(z)\mathrm{d}z}\\
          \rm{s.t.} \ &w(z)\ge 0,\\
          &w(z)\in C^{1}[0,1],\\
          &f(\boldsymbol{x})>0, \\
          &\nabla f(\boldsymbol{x}_1)\ge\gamma\nabla f(\boldsymbol{y}_1)\ge\boldsymbol{0},  \forall \boldsymbol{x}_1\le \boldsymbol{y}_1.
    \end{aligned} 
\end{equation}

(1) Before going into the detail, we first consider a new optimization problem as follows:
\begin{equation}\label{equ:20}
    \begin{aligned}
    \min_{w}\max_{R}\ \ &\theta(w,R)\\
          \rm{s.t.} \ &w(z)\ge 0,\\
          &w(z)\in C^{1}[0,1],\\
          &\gamma\int_{0}^{1}w(z)\mathrm{d}z=1,\\
          &R(z)\ge 0, \\
          &R(1)=1,\\
          & R'(z_{1})\ge\gamma R'(z_{2})\ge 0\ (\forall z_{1}\le z_{2}, z_{1}, z_{2}\in[0,1]),
    \end{aligned} 
\end{equation}where $\theta(w,R)=w(1)+\int^{1}_{0} (\gamma w(z)-w'(z))R(z)\mathrm{d}z$.

Next, we prove the equivalence between problem~\eqref{equ:19} and problem~\eqref{equ:20}. For any fixed point $\boldsymbol{x}\in\mathcal{C}$, we consider the function $m(z)=\frac{f(z\cdot \boldsymbol{x})}{f(\boldsymbol{x})}$ (we assume $f(\boldsymbol{x})>0$), which is satisfied with the constraints of problem~\eqref{equ:20}, i.e., $m(z)\ge 0$, $m(1)=1$, and $m'(z_{1})=\frac{\langle\boldsymbol{x}, \nabla f(z_{1}\cdot\boldsymbol{x})\rangle}{f(\boldsymbol{x})}\ge\frac{\gamma\langle\boldsymbol{x}, \nabla f(z_{2}\cdot\boldsymbol{x})\rangle}{f(\boldsymbol{x})}=\gamma m'(z_{2})\ge 0$ $(\forall z_{1}\le z_{2}, z_{1}, z_{2}\in[0,1])$. Therefore, the optimal objective value of problem~\eqref{equ:20} is larger than that of problem~\eqref{equ:19}. Moreover, for any $R(z)$ satisfying the constrains in problem~\eqref{equ:20}, we can design a function $f_{1}(\boldsymbol{x})=R(x_{1})$, where $x_{1}$ (we assume $x_{1}\in[0,1]$ in the~\cref{sec:pre}) is the first coordinate of point $\boldsymbol{x}$. Also, $f_{1}(\boldsymbol{x})\ge0$ and when $\boldsymbol{x}\le\boldsymbol{y}$, we have $\nabla f_{1}(\boldsymbol{x})\ge\gamma\nabla f_{1}(\boldsymbol{y})$. Hence, $f_{1}$ is also satisfied with the constraints of problem~\eqref{equ:19}. If we set $\boldsymbol{x}=(1,0,\dots,0)\in\mathcal{X}$, $\frac{f_{1}(z\cdot \boldsymbol{x})}{f_{1}(\boldsymbol{x})}=R(z)$ such that the optimal objective value of problem~\eqref{equ:19} is larger than that of problem~\eqref{equ:20}. As a result, the optimization problem~\eqref{equ:20} is equivalent to the problem~\eqref{equ:19}.

(2) Then, we prove the $\min_{w}\max_{f,\boldsymbol{x}}\theta(w,f,\boldsymbol{x})\ge\frac{1}{1-e^{-\gamma}}$. Setting $\widehat{R}(z)=\frac{1-e^{-\gamma z}}{1-e^{-\gamma}}$, we could verify that, if $\gamma\int_{0}^{1}w(z)\mathrm{d}z=1$, 
\begin{equation}\label{equ:21}
    \begin{aligned}
    \theta(w,\widehat{R})&=w(1)+\int^{1}_{0} (\gamma w(z)-w'(z))\widehat{R}(z)\mathrm{d}z\\
    &=w(1)+\frac{\int^{1}_{0} (\gamma w(z)-w'(z))\mathrm{d}z+\int^{1}_{0} e^{-\gamma z}(w'(z)-\gamma w(z))\mathrm{d}z}{1-e^{-\gamma}}\\
    &=w(1)+\frac{1-w(1)+w(0)+e^{-\gamma z}w(z)|_{z=0}^{z=1}}{1-e^{-\gamma}}\\
    &=w(1)+\frac{1-w(1)+w(0)+e^{-\gamma}w(1)-w(0)}{1-e^{-\gamma}}\\
    &=\frac{1}{1-e^{-\gamma}}.
    \end{aligned}
\end{equation}
Also, $\widehat{R}$ is satisfied with the constraints of optimization problem~\eqref{equ:20}, i.e., for any $z\in[0,1]$, $\widehat{R}(z)\ge0$, $\widehat{R}(1)=1$ and $\widehat{R}'(x)=\frac{\gamma e^{-\gamma x}}{1-e^{-\gamma}}\ge\frac{\gamma^{2} e^{-\gamma y}}{1-e^{-\gamma}}=\gamma\widehat{R}'(y)$ where $x\le y$ and $0\le\gamma\le 1$. Therefore, $\max_{R}\theta(w,R)\ge\theta(w,\widehat{R})=\frac{1}{1-e^{-\gamma}}$ and $\min_{w}\max_{f,\boldsymbol{x}}\theta(w,f,\boldsymbol{x})=\min_{w}\max_{R}\theta(w,R)\ge\frac{1}{1-e^{-\gamma}}$.

(3) We consider $\widehat{w}(z)=e^{\gamma(z-1)}$ and observe that $\widehat{w}'(z)=\gamma\widehat{w}(z)$ such that $\theta(\widehat{w},f,\boldsymbol{x})=\frac{\widehat{w}(1)+\int^{1}_{0} (\gamma \widehat{w}(z)-\widehat{w}'(z))\frac{f(z\cdot \boldsymbol{x})}{f(\boldsymbol{x})}\mathrm{d}z}{\gamma\int_{0}^{1}\widehat{w}(z)\mathrm{d}z}=\frac{\widehat{w}(1)}{\gamma\int_{0}^{1}\widehat{w}(z)\mathrm{d}z}=\frac{1}{1-e^{-\gamma}}$ for any function $f$. Also, $\widehat{w}(z)$ is satisfied with the constraints in optimization problem~\eqref{equ:19}, namely, $\widehat{w}(z)\ge 0$ and $\widehat{w}\in C^{1}[0,1]$. Therefore, $\frac{1}{1-e^{-\gamma}}=\min_{w}\max_{f,\boldsymbol{x}}\theta(w,f,\boldsymbol{x})$ and $e^{\gamma(z-1)}\in\arg\min_{w}\theta(w)$.
\end{proof}

\subsection{Proof of \texorpdfstring{\cref{thm:2}}{4}} \label{proof:thm2}
\begin{proof} 

For (i), we first verify that the value $\int_{0}^{1}\frac{e^{\gamma(z-1)}}{z}f(z\cdot \boldsymbol{x})\mathrm{d}z$ is controlled via $f(\boldsymbol{x})$ for any $\boldsymbol{x}\in\mathcal{X}$. For any $\delta\in(0,1)$,  we first have
\begin{equation}\label{equ:22}
\begin{aligned}
 &\int_{0}^{1}\frac{e^{\gamma(z-1)}}{z}f(z\cdot \boldsymbol{x})\mathrm{d}z\\
 &=(\int_{0}^{\delta}+\int_{\delta}^{1}) \frac{e^{\gamma(z-1)}}{z}f(z\cdot \boldsymbol{x})\mathrm{d}z\\
 &\le\int_{0}^{\delta}\frac{f(z\cdot \bx)}{z}\mathrm{d}z+(\int_{\delta}^{1}\frac{1}{z}\mathrm{d}z)f(\boldsymbol{x})\\
 &= \int_{0}^{\delta}\frac{f(z\cdot \bx)}{z}\mathrm{d}z+\ln(\frac{1}{\delta})f(\boldsymbol{x})\\
 &= \int_{0}^{\delta}\frac{\int_{0}^{z}\langle \boldsymbol{x},\nabla f(u\cdot\boldsymbol{x})\rangle \mathrm{d}u}{z}\mathrm{d}z+\ln(\frac{1}{\delta})f(\boldsymbol{x}),
 \end{aligned}
\end{equation} where the first inequality follows from $f(z\cdot \boldsymbol{x})\le f(\boldsymbol{x})$ and $\delta\in[0,1]$, and the final equality from $\int_{0}^{z}\langle \boldsymbol{x},\nabla f(u\cdot\boldsymbol{x})\rangle \mathrm{d}u=f(z\cdot \boldsymbol{x})-f(\boldsymbol{0})=f(z\cdot \boldsymbol{x})$. 

Next, \begin{equation}\label{equ:23}
    \begin{aligned}
     \int_{0}^{\delta}\frac{\int_{0}^{z}\langle \boldsymbol{x},\nabla f(u\cdot\boldsymbol{x})\rangle \mathrm{d}u}{z}\mathrm{d}z&=\int_{0}^{\delta}\langle \boldsymbol{x},\nabla f(u\cdot\boldsymbol{x})\rangle \int_{u}^{\delta}\frac{1}{z}\mathrm{d}z\mathrm{d}u\\
 &=\int_{0}^{\delta}\langle \boldsymbol{x},\nabla f(u\cdot\boldsymbol{x})\rangle \ln(\frac{\delta}{u})\mathrm{d}u\\
 &=\int_{0}^{\delta}(\langle \boldsymbol{x},\nabla f(u\cdot\boldsymbol{x})-\nabla f(\boldsymbol{x})\rangle+\langle \boldsymbol{x}, \nabla f(\boldsymbol{x})\rangle)\ln(\frac{\delta}{u})\mathrm{d}u\\
 &\le\int_{0}^{\delta}\ln(\frac{\delta}{u})\mathrm{d}u(L r^2(\mathcal{X})+\frac{f(\boldsymbol{x})}{\gamma})\\
  &=(u-u\ln(\frac{u}{\delta}))|_{u=0}^{\delta}(L r^2(\mathcal{X})+\frac{f(\boldsymbol{x})}{\gamma})\\
  &=\delta(L r^2(\mathcal{X})+\frac{f(\boldsymbol{x})}{\gamma}),
    \end{aligned}
\end{equation} where the first equality follows from the Fubini's theorem; in the first inequality, we use  $\langle \boldsymbol{x},\nabla f(u\cdot\boldsymbol{x})-\nabla f(\boldsymbol{x})\rangle\le L\left\|\boldsymbol{x}\right\|^{2}$, which is derived from the $L$-smooth property, and $\langle \boldsymbol{x}, \nabla f(\boldsymbol{x})\rangle\le\frac{  f(\boldsymbol{x})}{\gamma}$, following from the \cref{lemma:2} and $f(\boldsymbol{0})=0$; the final equality follows from $\lim_{u\rightarrow 0_{+}} u\ln(u)=0$.

From \cref{equ:22} and \cref{equ:23}, for any $\delta\in(0,1)$, we have 
\begin{equation}\label{equ:24}
    \begin{aligned}
    \int_{0}^{1}\frac{e^{\gamma(z-1)}}{z}f(z\cdot \boldsymbol{x})\mathrm{d}z&\le \ln(\frac{1}{\delta})f(\boldsymbol{x})+\delta(L_{*}r^2(\mathcal{X})+\frac{f(\boldsymbol{x})}{\gamma})\\
     &\le\ln(\frac{1}{\delta})(f(\boldsymbol{x})+c)+\delta(L_{*}r^2(\mathcal{X})+\frac{f(\boldsymbol{x})}{\gamma}),
    \end{aligned}
\end{equation} where the second inequality comes from $c>0$.

If we set $\delta=\frac{f(\boldsymbol{x})+c}{\frac{f(\boldsymbol{x})}{\gamma}+L r^2(\mathcal{X})}\in[0,1]$ ($0\le\gamma\le 1$ and $0<c\le L_{*}r^{2}(\mathcal{X})$), we have 
\begin{align*}
    \int_{0}^{1}\frac{e^{\gamma(z-1)}}{z}f(z\cdot \boldsymbol{x})\mathrm{d}z&\le\ln(\frac{1}{\delta})(f(\boldsymbol{x})+c)+\delta(L_{*}r^2(\mathcal{X})+\frac{f(\boldsymbol{x})}{\gamma})\\
     &=(1+\ln(\frac{1}{\delta}))(f(\boldsymbol{x})+c)\\
     &\le(1+\ln(\tau)(f(\boldsymbol{x})+c),
\end{align*} where the final inequality is derived from $\frac{1}{\delta}\le\tau$ and $\tau=max(\frac{1}{\gamma},\frac{L_{*}r^{2}(\mathcal{X})}{c})$.

As a result, the value $\int_{0}^{1}\frac{e^{\gamma(z-1)}}{z}f(z\cdot \boldsymbol{x})\mathrm{d}z$ is well-defined. We also could verify that $\nabla \int_{0}^{1}\frac{e^{\gamma(z-1)}}{z}f(z\cdot \boldsymbol{x})\mathrm{d}z=\int_{0}^{1}e^{\gamma(z-1)}\nabla f(z\cdot \boldsymbol{x})\mathrm{d}z$ so that we could set $F(\boldsymbol{x})=\int_{0}^{1}\frac{e^{\gamma(z-1)}}{z}f(z\cdot \boldsymbol{x})\mathrm{d}z$.

For the final one, 
\begin{equation}\label{equ:25}
    \begin{aligned}
    \left\|\nabla F(\boldsymbol{x})-\nabla F(\boldsymbol{y})\right\|&=\left\|\int_{0}^{1}e^{\gamma(z-1)}(\nabla f(z\cdot \boldsymbol{x})-\nabla f(z\cdot \boldsymbol{y}))\mathrm{d}z\right\|\\
    &\le\int_{0}^{1}e^{\gamma(z-1)}\left\|\nabla f(z\cdot \boldsymbol{x})-\nabla f(z\cdot \boldsymbol{y})\right\|\mathrm{d}z\\
    &\le L(\int_{0}^{1}e^{\gamma(z-1)}z \mathrm{d}z)\left\|\boldsymbol{x}-\boldsymbol{y}\right\|\\
    &=\frac{\gamma+e^{-\gamma}-1}{\gamma^2}L\left\|\boldsymbol{x}-\boldsymbol{y}\right\|.
    \end{aligned}
\end{equation}
and
\begin{equation}
    \begin{aligned}
        |F(\bx)-F(\by)|&\leq \int_0^1 \frac{e^{\gamma(z-1)}}{z} \left|f(z\cdot \bx)-f(z\cdot\by)\right|\mathrm{d}z
        \\&\leq \int_0^1 e^{\gamma(z-1)} L_1\|\bx-\by\|\mathrm{d}z=\frac{1-e^{-\gamma}}{\gamma}L_1\|\bx-\by\|.
    \end{aligned}
\end{equation}
\end{proof}

\subsection{Proof of \texorpdfstring{\cref{nonmonotone auxiliary}}{5}} \label{proof:nonmontone auxilliary}
Before proving this lemma, we first show a lemma that bounds the $f(\boldsymbol{x}\vee \boldsymbol{y})$ for non-monotone DR-submodular function $f$.
\begin{lemma}[Restatement of Lemma 3 in \citep{bian2017continuous}] \label{vee bound}
    Given $\bx\in\X $, for any $\by\in \X$, it holds $f(\bx \vee \by)\geq (1-\|\bx\|_{\infty}) f(\by)$.
\end{lemma}

Now we can prove Lemma \ref{nonmonotone auxiliary}.
\begin{proof}
    For simplicity, let $\bx^{\alpha z}:=(1-\alpha z)\cdot \underline{\bx} +\alpha z \cdot \bx$, we first bound $\langle \by, \nabla F(\bx)\rangle$, 
    \begin{equation}
        \begin{aligned}
            &\langle \by, \nabla F(\bx)\rangle \\
            &= \int_0^1 \omega(z)\langle \by, \nabla f(\bx^{\alpha z})\rangle \mathrm{d}z\\
            &= \int_0^1 \omega(z)\langle \by-\bx^{\alpha z}\wedge \by,\nabla f(\bx^{\alpha z})\rangle \mathrm{d}z + \int_0^1 \omega(z)\langle \bx^{\alpha z}\wedge \by-\bx^{\alpha z},\nabla f(\bx^{\alpha z})\rangle \mathrm{d}z \\
            &\quad\quad  + \int_0^1 \omega(z)\langle\bx^{\alpha z},\nabla F(\bx^{\alpha z})\rangle \mathrm{d}z\\
            &=\int_0^1 \omega(z)\langle \bx^{\alpha z}\vee \by-\bx^{\alpha z},\nabla f(\bx^{\alpha z})\rangle \mathrm{d}z + \int_0^1 \omega(z)\left(f(\bx^{\alpha z}\wedge \by) - f(\bx^{\alpha z}) \right)\mathrm{d}z\\
            &\quad\quad +\int_0^1 \omega(z)\langle\bx^{\alpha z},\nabla f(\bx^{\alpha z})\rangle \mathrm{d}z\\
            &\geq\int_0^1 \omega(z) f(\bx^{\alpha z}\vee \by)  \mathrm{d}z -\int_0^1 2\omega(z) f(\bx^{\alpha z})\mathrm{d}z + \int_0^1 \omega(z)\langle\bx^{\alpha z},\nabla f(\bx^{\alpha z})\rangle \mathrm{d}z.
        \end{aligned}
    \end{equation}
    The third equality is because $\by-\bx^{\alpha z}\wedge \by = \bx^{\alpha z}\vee \by -\bx^{\alpha z}$. Next we bound $\int_0^1 \omega(z) f(\bx^{\alpha z} \vee \by)\mathrm{d}z$.
    \begin{equation}
        \begin{aligned}
            \int_0^1 \omega(z) f(\bx^{\alpha z}\vee \by) \mathrm{d}z &\geq \int_0^1 (1-\|\bx^{\alpha z}\|)\omega(z) f(\by)\mathrm{d}z.\\
            &\geq \int_0^1 \left(1-((1-\alpha z) \|\underline{\bx}\|_{\infty}+\alpha z\|\bx\|_{\infty})\right)\omega(z)f(\by) \mathrm{d}z\\
            &\geq \int_0^1 \left(1-((1-\alpha z) \|\underline{\bx}\|_{\infty}+\alpha z)\right)\omega(z)f(\by) \mathrm{d}z\\
            &=\int_0^1 (1-\alpha z)(1-\|\underline{\bx}\|_{\infty}) \omega(z) f(\by) \mathrm{d}z.
        \end{aligned}
    \end{equation}
    Then, 
    \begin{equation}\label{y bound}
        \begin{aligned}
            \langle \by,\nabla F(\bx)\rangle &\geq \int_0^1 (1-\alpha z)(1-\|\underline{\bx}\|_{\infty}) \omega(z) f(\by) \mathrm{d}z-\int_0^1 2\omega(z) f(\bx^{\alpha z})\\
            &\quad\quad + \int_0^1 \omega(z)\langle\bx^{\alpha z},\nabla f(\bx^{\alpha z})\rangle \mathrm{d}z.
        \end{aligned}
    \end{equation}
    Next we bound $\langle \bx,\nabla F(\bx)\rangle$.
    \begin{equation}\label{x bound}
        \begin{aligned}
            \langle \bx,\nabla F(\bx)\rangle &= \int_0^1 \omega(z)\langle \bx - \bx^{\alpha z},\nabla f(\bx^{\alpha z})\rangle \mathrm{d}z +\int_0^1 \omega(z)\langle\bx^{\alpha z},\nabla f(\bx^{\alpha z})\rangle \mathrm{d}z\\
        \end{aligned}
    \end{equation}
    For the first term,
    \begin{equation}\label{x bound 1}
        \begin{aligned}
            &\int_0^1 \omega(z)\langle \bx - \bx^{\alpha z},\nabla f(\bx^{\alpha z})\rangle \mathrm{d}z
            \\&=\int_0^1 \omega(z)\langle (1-\alpha z)(\bx-\underline{\bx}),\nabla f(\bx^{\alpha z})\rangle \mathrm{d}z\\
            &=\int_0^1 \frac{(1-\alpha z)\omega(z)}{\alpha }\langle \alpha (\bx-\underline{\bx}),\nabla f(\bx^{\alpha z})\rangle \mathrm{d}z\\
            &=\int_0^1 \frac{(1-\alpha z)\omega(z)}{\alpha } d f(\bx^{\alpha z})\\
            &=\frac{(1-\alpha z)\omega(z)}{\alpha} f(\bx^{\alpha z})\bigg|_{z=0}^1-\int_0^1 \frac{(1-\alpha z) \omega'(z)-\alpha \omega(z)}{\alpha}f(\bx^{\alpha z}) \mathrm{d}z \\
            &=\frac{(1-\alpha)\omega(1)}{\alpha} f(\bx^{\alpha})-\frac{\omega(0)}{\alpha}f(\underline{\bx})-\int_0^1\frac{(1-\alpha z)\omega'(z)}{\alpha}f(\bx^{\alpha z})\mathrm{d}z +\int_0^1 \omega(z) f(\bx^{\alpha z})\mathrm{d}z\\
            &\leq \frac{(1-\alpha)\omega(1)}{\alpha} f(\bx^{\alpha})-\int_0^1\frac{(1-\alpha z)\omega'(z)}{\alpha}f(\bx^{\alpha z})\mathrm{d}z +\int_0^1 \omega(z) f(\bx^{\alpha z})\mathrm{d}z
        \end{aligned}
    \end{equation}
    Combine (\ref{y bound}), (\ref{x bound}) and (\ref{x bound 1}), we have,
    \begin{equation}
        \begin{aligned}
            &\langle \by-\bx ,\nabla F(\bx^{\alpha z})\rangle\\
            &\geq \int_0^1 (1-\alpha z)(1-\|\underline{\bx}\|_{\infty}) \omega(z) f(\by) \mathrm{d}z-\left(\int_0^1 3\omega(z) \frac{f(\bx^{\alpha z})}{f(\bx^{\alpha})}\mathrm{d}z \right)f(\bx^{\alpha})\\
            &\quad \quad +\left(\int_0^1\frac{(1-\alpha z)\omega'(z)}{\alpha}\frac{f(\bx^{\alpha z})}{f(\bx^{\alpha})}\mathrm{d}z\right) f(\x^{\alpha}) -\frac{(1-\alpha)\omega(1)}{\alpha} f(\bx^{\alpha})
            \\&=\left( (1-\|\bx\|_{\infty})\int_0^1(1-\alpha z)\omega(z)\mathrm{d}z\right)f(\by)\\
            &\quad\quad - \left(\frac{(1-\alpha)\omega(1)}{\alpha}+\int_0^1\left(3\omega(z)-\frac{1-\alpha z}{\alpha}\omega'(z)\right)\frac{f(\bx^{\alpha z})}{f(\bx^{\alpha})} \mathrm{d}z \right) f(\bx^{\alpha})
            \\&\geq \left((1-\|\bx\|_{\infty})\int_0^1(1-\alpha z)\omega(z)\mathrm{d}z\right) \left(f(\by)-\theta(\omega)f(\alpha \bx + (1-\alpha \underline{\bx}))\right).
        \end{aligned}
    \end{equation}
    Where $\theta(\omega)= \max_{f,\bx}\theta(\omega,f,\bx)$ and
    \begin{equation}
        \theta(\omega,f,\bx) = \frac{\frac{(1-\alpha)\omega(1)}{\alpha}+\int_0^1\left(3\omega(z)-\frac{1-\alpha z}{\alpha}\omega'(z)\right)\frac{f(\bx^{\alpha z})}{f(\bx^{\alpha})} \mathrm{d}z}{(1-\|\bx\|_{\infty})\int_0^1(1-\alpha z)\omega(z)\mathrm{d}z}.
    \end{equation}

    \subsection{Proof of \texorpdfstring{\cref{nonmonotone nonoblivious properties}}{6}}\label{proof: nonmontone nonoblivious properties}
    \begin{proof}
        We first verify that 
    \begin{align}\label{F the form}
        \int_0^1 \frac{1}{4 z(1-\frac{z}{2})^3}\left(f\left(\frac{z}{2} \cdot (\bx-\underline{\bx})+\underline{\bx}\right) - f(\underline{\bx})\right)\mathrm{d}z
    \end{align}
    is bounded for any $\boldsymbol{x}$. The following holds
    \begin{equation}
        \begin{aligned}
            (\ref{F the form})&\leq \int_0^{1}\frac{1}{4 z(1-\frac{z}{2})^3}\left(\frac{z}{2}L_1\|\bx-\underline{\bx}\|\right) \mathrm{d}z\\
            &\leq \left(\int_0^1 \frac{1}{8(1-\frac{z}{2})^3}\mathrm{d}z\right)L_1\mathrm{diam}(\mathcal{C})\\
            &\leq \frac{3}{8}L_1 \mathrm{diam}(\mathcal{C}).
        \end{aligned}
    \end{equation}
    Then we can check that $\nabla \int_0^1 \frac{1}{4 z(1-\frac{z}{2})^3} \left(f\left(\frac{z}{2} \cdot (\bx-\underline{\bx})+\underline{\bx}\right) - f(\underline{\bx})\right)\mathrm{d}z = \int_0^1 \frac{1}{8(1-\frac{z}{2})^3}\nabla f\left(\frac{z}{2} \cdot (\bx-\underline{\bx})+\underline{\bx}\right) \mathrm{d}z$. Thus, $(i)$ holds.

    For any $\bx,\by\in \X$, 
    \begin{equation}
        \begin{aligned}
            \nabla F(\bx)- \nabla F(\by) &= \int_0^1 \frac{1}{8(1-\frac{z}{2})^3}\left(\nabla f\left(\frac{z}{2}\cdot(\bx-\underline{\bx})+\underline{\bx}\right)-\nabla f\left(\frac{z}{2}\cdot(\by-\underline{\bx})+\underline{\bx}\right)\right)\mathrm{d}z\\
            &\leq \int_0^1 \frac{1}{8(1-\frac{z}{2})^3} \frac{zL}{2} \|\bx-\by\| \mathrm{d}z\\
            &=\left(\int_0^1 \frac{z}{8(1-\frac{z}{2})^3}\mathrm{d}z\right) \frac{L}{2}\|\bx-\by\|\\
            &=\frac{1}{8}L\|\bx-\by\|.
        \end{aligned}
    \end{equation}
    and 
    \begin{equation}
        \begin{aligned}
            |F(\bx)-F(\by)| &= \int_0^1\frac{1}{4z(1-\frac{z}{2})^3} \left| f\left(\frac{z}{2}\cdot(\bx-\underline{\bx})+\underline{\bx}\right)- f\left(\frac{z}{2}\cdot(\by-\underline{\bx})+\underline{\bx}\right)\right|\mathrm{d}z
            \\&\leq \int_0^1 \frac{1}{8(1-\frac{z}{2})^3}L_1 \|\bx-\by\| \mathrm{d}z
            \\&\leq \frac{3}{8}L_1\|\bx-\by\|.
        \end{aligned}
    \end{equation}
    Thus, $(ii)$ holds.
    \end{proof}
\end{proof}

\subsection{Proof of \texorpdfstring{\cref{prop:1}}{7}} \label{proof:prop1}
\begin{proof}
For the first one, fixed $z$, $\mathbb{E}\left(\left.\widetilde{\nabla}f(z\cdot \boldsymbol{x})\right|\boldsymbol{x},z\right)=\nabla f(z\cdot \boldsymbol{x})$ such that $\mathbb{E}\left(\left.\widetilde{\nabla}f(z\cdot \boldsymbol{x})\right|\boldsymbol{x}\right)=\mathbb{E}_{z\sim\mathbf{Z}_{\uparrow}}\left(\mathbb{E}\left(\left.\widetilde{\nabla}f(z\cdot \boldsymbol{x})\right|\boldsymbol{x},z
\right)\right)=\mathbb{E}_{z\sim\mathbf{Z}_{\uparrow}}\left(\left.\nabla f(z\cdot \boldsymbol{x})\right|\boldsymbol{x}\right)=\int_{z=0}^{1} \frac{\gamma e^{\gamma(z-1)}}{1-e^{-\gamma}}\nabla f(z\cdot \boldsymbol{x})\mathrm{d}z=\frac{\gamma}{1-e^{-\gamma}}F(\boldsymbol{x})$. For the second one, 
        \begin{align*}
        &\mathbb{E}\left(\left.\left\|\frac{1-e^{-\gamma}}{\gamma}\widetilde{\nabla}f(z\cdot \boldsymbol{x})-\nabla F(\boldsymbol{x})\right\|^{2}\right|\boldsymbol{x}\right)\\
        =&\mathbb{E}\left(\left.\left\|\frac{1-e^{-\gamma}}{\gamma}(\widetilde{\nabla}f(z\cdot \boldsymbol{x})-\nabla f(z\cdot \boldsymbol{x}))+\frac{1-e^{-\gamma}}{\gamma}\nabla f(z\cdot \boldsymbol{x})-\nabla F(\boldsymbol{x})\right\|^{2}\right|\boldsymbol{x}\right)\\
        \le & 2\mathbb{E}_{z\sim\mathbf{Z}_{\uparrow}}\left(\mathbb{E}\left(\left.\left\|\frac{1-e^{-\gamma}}{\gamma}(\widetilde{\nabla}f(z\cdot \boldsymbol{x})-\nabla f(z\cdot \boldsymbol{x}))\right\|^{2}\right|\boldsymbol{x},z\right) +\left\|\frac{1-e^{-\gamma}}{\gamma}\nabla f(z\cdot \boldsymbol{x})-\nabla F(\boldsymbol{x})\right\|^{2}\right)\\
        \le & 2\frac{(1-e^{-\gamma})^{2}\sigma^{2}}{\gamma^{2}}+2\mathbb{E}_{z\sim\mathbf{Z}_{\uparrow}}\left(\left.\left\|\frac{1-e^{-\gamma}}{\gamma}\nabla f(z\cdot \boldsymbol{x})-\nabla F(\boldsymbol{x})\right\|^{2}\right|\boldsymbol{x}\right) \\
        \le & 2\frac{(1-e^{-\gamma})^{2}\sigma^{2}}{\gamma^{2}}+2\mathbb{E}_{z\sim\mathbf{Z}_{\uparrow}}\left(\left.\left\|\int_{0}^{1} e^{\gamma(u-1)}(\nabla f(z\cdot \boldsymbol{x})-\nabla f(u\cdot\boldsymbol{x}))\mathrm{d}u\right
        \|^{2}\right|\boldsymbol{x}\right)\\
        \le & 2\frac{(1-e^{-\gamma})^{2}\sigma^{2}}{\gamma^{2}}+2\mathbb{E}_{z\sim\mathbf{Z}_{\uparrow}}\left(\left.\left(\int_{0}^{1}e^{\gamma(u-1)}|z-u|L\left\|\boldsymbol{x}\right
        \|\mathrm{d}u\right)^{2}\right|\boldsymbol{x}\right)\\
        \le & 2\frac{(1-e^{-\gamma})^{2}\sigma^{2}}{\gamma^{2}}+2\mathbb{E}_{z\sim\mathbf{Z}_{\uparrow}}\left(\left.\int_{0}^{1}e^{\gamma(u-1)}\mathrm{d}u\int_{u=0}^{1}e^{\gamma(u-1)}(z-u)^2L^{2}\left\|\boldsymbol{x}\right
        \|^{2}\mathrm{d}u\right|\boldsymbol{x}\right)\\
        = & 2\frac{(1-e^{-\gamma})^{2}\sigma^{2}}{\gamma^{2}}+2 \int_{z=0}^{1}\int_{u=0}^{1}e^{\gamma(u+z-2)}(z-u)^{2}L^{2}\left\|\boldsymbol{x}\right
        \|^{2} \mathrm{d}u\mathrm{d}z\\
        \le & 2\frac{(1-e^{-\gamma})^{2}\sigma^{2}}{\gamma^{2}}+\frac{2L^{2}r^{2}(\mathcal{X})(1-e^{-2\gamma})}{3\gamma},
        \end{align*}
     where the first and fifth inequalities come from Cauchy–Schwarz inequality.
    \end{proof}

\subsection{Proof of \texorpdfstring{\cref{prop:2}}{8}}\label{proof:prop2}
\begin{proof}
    \begin{align*}
        \E\left(\widetilde{\nabla} F(\bx) \big| \bx\right)&= \E_z\left(\E\left(\frac{3}{8}\widetilde{\nabla} f(\frac{z}{2}(\bx-\underline{\bx})+\underline{\bx})\big| \bx,z\right)\big| \bx\right)\\
        &=\E_z \left(\frac{3}{8}\nabla f\left(\frac{z}{2}(\bx-\underline{\bx})+\underline{\bx}\right)\big| \bx\right)\\
        &=\int_0^1 \frac{1}{8(1-\frac{z}{2})^3}\nabla f\left(\frac{z}{2}(\bx-\underline{\bx})+\underline{\bx}\right) \mathrm{d}z = \nabla F(\bx),
    \end{align*}
    which shows that $(i)$ holds. For the second one, 
    \begin{align*}
        &\E\left(\|\widetilde{\nabla} F(\bx) - \nabla F(\bx) \|^2 \bigg| \bx\right) 
        \\& \leq \E\left( \left\|\frac{3}{8}\widetilde{\nabla} f(\frac{z}{2}(\bx-\underline{\bx})+\underline{\bx})- \frac{3}{8}\nabla f(\frac{z}{2}(\bx-\underline{\bx})+\underline{\bx}) \right\|^2 + \left\|\frac{3}{8}\nabla f(\frac{z}{2}(\bx-\underline{\bx})+\underline{\bx})-\nabla F(\bx)\right\|^2 \bigg| \bx \right)\\
        & \leq \E_{z\sim\mathbf{Z}_{\sim}}\left( \E\left(\left\|\frac{3}{8}\widetilde{\nabla} f(\frac{z}{2}(\bx-\underline{\bx})+\underline{\bx})- \frac{3}{8}\nabla f(\frac{z}{2}(\bx-\underline{\bx})+\underline{\bx}) \right\|^2 \bigg| \bx,z \right)\bigg| \bx\right)\\
        &\quad\quad + \E_{z\sim\mathbf{Z}_{\sim}}\left(\left\|\frac{3}{8}\nabla f(\frac{z}{2}(\bx-\underline{\bx})+\underline{\bx}) - \int_0^1 \frac{1}{8(1-\frac{z}{2})^3}\nabla f(\frac{u}{2}(\bx-\underline{\bx})+\underline{\bx})\mathrm{d}u\right\|^2  \bigg| \bx\right)\\
        &\leq \E_z\left(\frac{3}{8} \sigma^2 \big| \bx\right) + \int_0^1 \left(\frac{1}{3(1-\frac{z}{2})^3} \left\| \frac{3}{8}\nabla f(\frac{z}{2}(\bx-\underline{\bx})+\underline{\bx}) - \int_0^1  \frac{1}{8(1-\frac{u}{2})^3}\nabla f(\frac{u}{2}(\bx-\underline{\bx})+\underline{\bx})\mathrm{d}u\right\|^2 \right)\mathrm{d}z\\
        &= \frac{3}{8}\sigma^2+ \int_0^1 \left(\frac{1}{3(1-\frac{z}{2})^3} \left\| \int_0^1   \frac{1}{8(1-\frac{u}{2})^3}\left(\nabla f(\frac{z}{2}(\bx-\underline{\bx})+\underline{\bx})-\nabla f(\frac{u}{2}(\bx-\underline{\bx})+\underline{\bx})\right)\mathrm{d}u\right\|^2 \right)\mathrm{d}z\\
        &\leq \frac{3}{8}\sigma^2 +\int_0^1 \frac{1}{3(1-\frac{z}{2})^3}  \int_0^1   \frac{1}{8(1-\frac{u}{2})^3}\left\|\nabla f(\frac{z}{2}(\bx-\underline{\bx})+\underline{\bx})-\nabla f(\frac{u}{2}(\bx-\underline{\bx})+\underline{\bx})\right\|^2  \mathrm{d}u \mathrm{d}z \\
        &\leq \frac{3}{8}\sigma^2 +\int_0^1 \frac{1}{3(1-\frac{z}{2})^3}  \int_0^1 \frac{1}{8(1-\frac{u}{2})^3} \frac{(z-u)^2 L^2}{4}\|\bx-\underline{\bx}\|^2 \mathrm{d}u \mathrm{d}z\\
        &\leq \frac{3}{8}\sigma^2 + \frac{L^2 \mathrm{diam}^2(\mathcal{C})}{4}\int_0^1 \int_0^1 \frac{(z-u)^2}{24(1-\frac{z}{2})^3(1-\frac{u}{2})^3} \mathrm{d}u \mathrm{d}z\\
        &= \frac{3}{8}\sigma^2 + \frac{\ln (64)-4}{12} L^2 \mathrm{diam}^2(\X).
    \end{align*}
\end{proof}
\section{Proofs in Section~\ref{sec:gradient_ascent}}\label{Appendix:B}
In this section, we omit the subscripts of $F_{\uparrow}$ and $F_{\sim}$ and use $F$ to represent both non-oblivious functions, which will not lead to ambiguity.

\subsection{Proof of \texorpdfstring{\cref{thm:4}}{9}}\label{Appendix:B1}
Before verifying the \cref{thm:4} and \cref{thm:4 nonmonotone}, we first provide following lemma. 
\begin{lemma} \label{lemma:10} In the $t$-round update in ~\cref{alg:1},
if we select Option I, then for any $\boldsymbol{y}\in\mathcal{C}$ and $\mu_{t}>0$, we have 
\begin{align*}
    &\mathbb{E}\left(F(\boldsymbol{x}_{t+1})-F(\boldsymbol{x}_{t})+f(\boldsymbol{x}_{t})-(1-e^{-\gamma})f(\boldsymbol{y})\right)\\
    \ge & \mathbb{E}\left(\frac{1}{2\eta_{t}}(\left\|\boldsymbol{y}-\boldsymbol{x}_{t+1}\right\|^{2}-\left\|\boldsymbol{y}-\boldsymbol{x}_{t}\right\|^{2})-\frac{1}{2\mu_{t}}\left\|\nabla F(\boldsymbol{x}_{t})-\widetilde{\nabla} F(\boldsymbol{x}_t) \right\|^{2}+(\frac{1}{2\eta_{t}}-\frac{\mu_{t}+L_{\gamma}}{2})\left\|\boldsymbol{x}_{t+1}-\boldsymbol{x}_{t})\right\|^{2}\right).
\end{align*} 

\end{lemma}
\begin{proof} 
From the \cref{thm:2}, when $f$ is $L$-$smooth$, the non-oblivious function $F$ is $L_{\gamma}$-$smooth$. Hence
\begin{equation}\label{equ:43}
    \begin{aligned}
    F(\boldsymbol{x}_{t+1})-F(\boldsymbol{x}_{t})&=\int_{0}^{1}\langle \boldsymbol{x}_{t+1}-\boldsymbol{x}_{t}, \nabla F(\boldsymbol{x}_{t}+z(\boldsymbol{x}_{t+1}-\boldsymbol{x}_{t}))\rangle \mathrm{d}z \\
    &\ge \langle \boldsymbol{x}_{t+1}-\boldsymbol{x}_{t}, \nabla F(\boldsymbol{x}_t) \rangle-\frac{L_{\gamma}}{2}\left\|\boldsymbol{x}_{t+1}-\boldsymbol{x}_{t}\right\|^{2}.\\
    \end{aligned}
\end{equation} 
 
Then,
\begin{equation}\label{equ:44}
 \begin{aligned}
 &\langle \boldsymbol{x}_{t+1}-\boldsymbol{x}_{t}, \nabla F(\boldsymbol{x}_t) \rangle\\
= &\langle \boldsymbol{x}_{t+1}-\boldsymbol{x}_{t}, \widetilde{\nabla} F(\boldsymbol{x}_t) \rangle+\langle \boldsymbol{x}_{t+1}-\boldsymbol{x}_{t}, \nabla F(\boldsymbol{x}_{t})-\widetilde{\nabla} F(\boldsymbol{x}_t) \rangle\\
\ge & \langle \boldsymbol{x}_{t+1}-\boldsymbol{x}_{t}, \widetilde{\nabla} F(\boldsymbol{x}_t) \rangle-\frac{1}{2\mu_{t}}\left\|\nabla F(\boldsymbol{x}_{t})-\widetilde{\nabla} F(\boldsymbol{x}_t) \right\|^{2}-\frac{\mu_{t}}{2}\left\|\boldsymbol{x}_{t+1}-\boldsymbol{x}_{t}\right\|^{2},
\end{aligned}
\end{equation} where the first inequality is from the Young's inequality.
   
Since $\bx_{t+1} = \mathcal{P}_{\mathcal{C}}\left(\bx_t +\eta \widetilde{\nabla} F(\boldsymbol{x}_t)\right)$, we have, for any $\boldsymbol{y}\in\mathcal{C}$
\begin{equation}\label{equ:45}
   \begin{aligned}
    &\langle \boldsymbol{x}_{t+1}-\boldsymbol{x}_{t}, \widetilde{\nabla} F(\boldsymbol{x}_t) \rangle\\
    = & \langle\boldsymbol{x}_{t+1}-\boldsymbol{y}, \widetilde{\nabla} F(\boldsymbol{x}_t) \rangle+\langle \boldsymbol{y}-\boldsymbol{x}_{t}, \widetilde{\nabla} F(\boldsymbol{x}_t) \rangle\\
    \geq & \frac{1}{2\eta_{t}}(\left\|\boldsymbol{y}-\boldsymbol{x}_{t+1}\right\|^{2}+\left\|\boldsymbol{x}_{t+1}-\boldsymbol{x}_{t}\right\|^{2}-\left\|\boldsymbol{y}-\boldsymbol{x}_{t}\right\|^{2})+\langle \boldsymbol{y}-\boldsymbol{x}_{t}, \widetilde{\nabla} F(\boldsymbol{x}_t) \rangle,
\end{aligned}    
\end{equation} where the inequality follows from the~\cref{lem:gradient step}.

From the Equation~\eqref{equ:43}-\eqref{equ:45}, we have
\begin{equation}\label{equ:46}
    \begin{aligned}
    &F(\boldsymbol{x}_{t+1})-F(\boldsymbol{x}_{t})\\
    \ge & \frac{1}{2\eta_{t}}(\left\|\boldsymbol{y}-\boldsymbol{x}_{t+1}\right\|^{2}+\left\|\boldsymbol{x}_{t+1}-\boldsymbol{x}_{t}\right\|^{2}-\left\|\boldsymbol{y}-\boldsymbol{x}_{t}\right\|^{2})+\langle \boldsymbol{y}-\boldsymbol{x}_{t}, \widetilde{\nabla} F(\boldsymbol{x}_t) \rangle\\
    &-\frac{1}{2\mu_{t}}\left\|\nabla F(\boldsymbol{x}_{t})-\widetilde{\nabla} F(\boldsymbol{x}_t) \right\|^{2}-\frac{\mu_{t}+L_{\gamma}}{2}\left\|\boldsymbol{x}_{t+1}-\boldsymbol{x}_{t}\right\|^{2}\\
    \ge &\frac{1}{2\eta_{t}}(\left\|\boldsymbol{y}-\boldsymbol{x}_{t+1}\right\|^{2}-\left\|\boldsymbol{y}-\boldsymbol{x}_{t}\right\|^{2})+\langle \boldsymbol{y}-\boldsymbol{x}_{t}, \widetilde{\nabla} F(\boldsymbol{x}_t) \rangle\\
    &-\frac{1}{2\mu_{t}}\left\|\nabla F(\boldsymbol{x}_{t})-\widetilde{\nabla} F(\boldsymbol{x}_t) \right\|^{2}+(\frac{1}{2\eta_{t}}-\frac{\mu_{t}+L_{\gamma}}{2})\left\|\boldsymbol{x}_{t+1}-\boldsymbol{x}_{t}\right\|^{2}.
    \end{aligned}
\end{equation} 

From the \cref{prop:1}, $\mathbb{E}(\widetilde{\nabla}F(\boldsymbol{x}_{t})|\boldsymbol{x}_{t})=\nabla F(\boldsymbol{x}_{t})$ and we also have
\begin{equation}\label{equ:47}
    \begin{aligned}
      &\mathbb{E}\left(F(\boldsymbol{x}_{t+1})-F(\boldsymbol{x}_{t})\right)\\
      \ge 
      &\mathbb{E}\left(\frac{1}{2\eta_{t}}(\left\|\boldsymbol{y}-\boldsymbol{x}_{t+1}\right\|^{2}-\left\|\boldsymbol{y}-\boldsymbol{x}_{t}\right\|^{2})+\mathbb{E}(\langle \boldsymbol{y}-\boldsymbol{x}_{t}, \widetilde{\nabla} F(\boldsymbol{x}_t) \rangle|\boldsymbol{x}_{t})\right.\\
      &\left.-\frac{1}{2\mu_{t}}\left\|\nabla F(\boldsymbol{x}_{t})-\widetilde{\nabla} F(\boldsymbol{x}_t) \right\|^{2}+(\frac{1}{2\eta_{t}}-\frac{\mu_{t}+L_{\gamma}}{2})\left\|\boldsymbol{x}_{t+1}-\boldsymbol{x}_{t}\right\|^{2}\right)\\
      =&\mathbb{E}\left(\frac{1}{2\eta_{t}}(\left\|\boldsymbol{y}-\boldsymbol{x}_{t+1}\right\|^{2}-\left\|\boldsymbol{y}-\boldsymbol{x}_{t}\right\|^{2})+\langle \boldsymbol{y}-\boldsymbol{x}_{t}, \nabla F(\boldsymbol{x}_t) \rangle \right.\\
      &\left.-\frac{1}{2\mu_{t}}\left\|\nabla F(\boldsymbol{x}_{t})-\widetilde{\nabla} F(\boldsymbol{x}_t) \right\|^{2}+(\frac{1}{2\eta_{t}}-\frac{\mu_{t}+L_{\gamma}}{2})\left\|\boldsymbol{x}_{t+1}-\boldsymbol{x}_{t}\right\|^{2}\right)\\
      \ge&\mathbb{E}\left(\frac{1}{2\eta_{t}}(\left\|\boldsymbol{y}-\boldsymbol{x}_{t+1}\right\|^{2}-\left\|\boldsymbol{y}-\boldsymbol{x}_{t}\right\|^{2})+(1-e^{-\gamma})f(\boldsymbol{y})-f(\boldsymbol{x}_{t})\right.\\
      &\left.-\frac{1}{2\mu_{t}}\left\|\nabla F(\boldsymbol{x}_{t})-\widetilde{\nabla} F(\boldsymbol{x}_t) \right\|^{2}+(\frac{1}{2\eta_{t}}-\frac{\mu_{t}+L_{\gamma}}{2})\left\|\boldsymbol{x}_{t+1}-\boldsymbol{x}_{t}\right\|^{2}\right),
    \end{aligned}
\end{equation} where the final inequality from the definition of $F$.
\end{proof}
Next, we prove the \cref{thm:4}.
\begin{proof}
From the Lemma~\ref{lemma:10}, if we set $\boldsymbol{y}=\boldsymbol{x}^{*}=\arg\max_{\boldsymbol{x}\in\mathcal{C}}f(\boldsymbol{x})$, $\mu_{t}=\frac{\sigma_{\gamma}\sqrt{t}}{\mathrm{diam}(\mathcal{C})}$ and $\eta_t=\frac{1}{\mu_{t}+L_{r}}$, we have $\frac{1}{2\eta_{t}}-\frac{\mu_{t}+L_{\gamma}}{2}=0$. Then,
 \begin{equation}\label{equ:48}
     \begin{aligned}
        &\sum_{t=1}^{T-1}\mathbb{E}\left(F(\boldsymbol{x}_{t+1})-F(\boldsymbol{x}_{t})+f(\boldsymbol{x}_{t})-(1-e^{-\gamma})f(\boldsymbol{x}^{*})\right)\\
        \ge&\sum_{t=1}^{T-1}\mathbb{E}\left(\frac{1}{2\eta_{t}}(\left\|\boldsymbol{x}^{*}-\boldsymbol{x}_{t+1}\right\|^{2}-\left\|\boldsymbol{x}^{*}-\boldsymbol{x}_{t}\right\|^{2})-\sum_{t=1}^{T-1}\frac{1}{2\mu_{t}}\left\|\nabla F(\boldsymbol{x}_{t})-\widetilde{\nabla} F(\boldsymbol{x}_t) \right\|^{2}\right)\\
        \ge&-\sigma_{\gamma}^2\sum_{t=1}^{T-1}\frac{1}{2\mu_{t}}+\sum_{t=2}^{T-1}\mathbb{E}(\left\|\boldsymbol{x}^{*}-\boldsymbol{x}_{t}\right\|^{2})(\frac{1}{2\eta_{t-1}}-\frac{1}{2\eta_{t}})+\mathbb{E}(\frac{\left\|\boldsymbol{x}^{*}-\boldsymbol{x}_{T}\right\|^{2}}{2\eta_{T-1}}-\frac{\left\|\boldsymbol{x}^{*}-\boldsymbol{x}_{1}\right\|^{2}}{2\eta_{1}})\\
        \ge& -\frac{\mathrm{diam}^{2}(\mathcal{C})}{2\eta_{1}}-\sigma_{\gamma}^2\sum_{t=1}^{T-1}\frac{1}{2\mu_{t}}\\
        \ge& -\left(\frac{\mathrm{diam}(\mathcal{C})(\sigma_{\gamma}+L_{\gamma}\mathrm{diam}(\mathcal{C}))}{2}+\frac{3}{2}\sigma_{\gamma}\mathrm{diam}(\mathcal{C})\sqrt{T-1}\right)
        \end{aligned}
 \end{equation} 
the second inequality from the \cref{prop:1} and the Abel's inequality; the third inequality from the definition of $\mathrm{diam}(\mathcal{C})$. The last inequality is because $\sum_{t=1}^{T-1}\frac{1}{\sqrt{t}}\leq 1+\int_1^{T-1} \frac{1}{\sqrt{t}}\mathrm{d}t=1+2\sqrt{T-1}\leq 3\sqrt{T}$.

 Finally, we have:
 \begin{equation}\label{equ:49}
     \begin{aligned}
        \mathbb{E}\left(\sum_{t=1}^{T-1}f( \boldsymbol{x}_{t})+F( \boldsymbol{x}_{T})-F(\bx_1)\right)\ge(1-e^{-\gamma})(T-1)f( \boldsymbol{x}^{*})-\mathrm{diam}(\mathcal{C})\left(\frac{(\sigma_{\gamma}+L_{\gamma})}{2}+\frac{3}{2}\sigma_{\gamma}\mathrm{diam}(\mathcal{C})\sqrt{T-1}\right).
     \end{aligned}
 \end{equation}
According to \cref{thm:2}, $F(\bx)$ is $\frac{1-e^{-\gamma}}{\gamma}L_1$-lipschitz continuous, then
\begin{equation}\label{equ:50}
     \begin{aligned}
        &\mathbb{E}\left(\sum_{t=1}^{T-1}f(\boldsymbol{x}_{t})\right)
        \\&\ge(1-e^{-\gamma})(T-1)f(\boldsymbol{x}^{*})-\mathrm{diam}(\mathcal{C})\left(\frac{(\sigma_{\gamma}+L_{\gamma})}{2}+\frac{3}{2}\sigma_{\gamma}\mathrm{diam}(\mathcal{C})\sqrt{T-1}\right)-\frac{1-e^{-\gamma}}{\gamma}L_1\mathrm{diam}(\mathcal{C}).
     \end{aligned}
 \end{equation}
 Therefore
\begin{equation}\label{equ:52}
     \begin{aligned}
        &\mathbb{E}\left(\sum_{t=1}^{T-1}\frac{1}{T-1}f(\boldsymbol{x}_{t})\right)
        \\&\ge(1-e^{-\gamma})f(\boldsymbol{x}^{*})-\frac{\mathrm{diam}(\mathcal{C})\left(\frac{(\sigma_{\gamma}+L_{\gamma})}{2}+\frac{3}{2}\sigma_{\gamma}\mathrm{diam}(\mathcal{C})\sqrt{T-1}\right)+(1-e^{-\gamma})L_1\mathrm{diam}(\mathcal{C})/\gamma}{T-1}.
     \end{aligned}
 \end{equation}
 We have
\begin{align*}
        \mathbb{E}(f(\boldsymbol{x}_{l}))\ge(1-e^{-\gamma})f(\boldsymbol{x}^{*})-O\left(\frac{1}{\sqrt{T}}\right).
     \end{align*}
\end{proof}

\subsection{Proof of \texorpdfstring{\cref{thm:4 nonmonotone}}{10}}\label{Appendix:B2}
The proof follows the same way of \cref{thm:4}, we have the following lemma similar to \cref{lemma:10}. 
\begin{lemma}\label{nonmonotone single step}
    In \cref{alg:1}, if we select Option II, then the following inequality holds for any $\by\in\mathcal{C}$ and $\eta_t>0$.
    \begin{equation}
        \begin{aligned}
            &\E\left(F(\bx_{t+1})-F(\bx_t) - \frac{1-\|\underline{\bx}\|_{\infty}}{4} f(\by)+ f\left(\frac{\bx_t+\underline{\bx}}{2}\right) \right)
            \\&\geq \E\left(\frac{1}{2\eta_t}\left(\|\by-\bx_{t+1}\|^2-\|\by-\bx_t\|^2\right) \right)
            \\&\quad \quad - \E\left(\frac{1}{2\mu_t}\left\|\nabla F(\bx_t)-\widetilde{\nabla} F(\bx_t)\right\|^2 -\left(\frac{L}{6}+\frac{\mu_t}{2} - \frac{1}{2\eta_t}\right) \left\|\bx_{t+1}-\bx_t\right\|^2 \right).
        \end{aligned}
    \end{equation}
\end{lemma}
\begin{proof}
    By Theorem \ref{nonmonotone nonoblivious properties}, $F(\bx)$ is $\frac{1}{8}L$-smooth, then follow the derivation of \cref{equ:46}, 
    \begin{equation}\label{equ:48.1}
        \begin{aligned}
        F(\boldsymbol{x}_{t+1})-F(\boldsymbol{x}_{t})
        \ge &\frac{1}{2\eta_{t}}(\left\|\boldsymbol{y}-\boldsymbol{x}_{t+1}\right\|^{2}-\left\|\boldsymbol{y}-\boldsymbol{x}_{t}\right\|^{2})+\langle \boldsymbol{y}-\boldsymbol{x}_{t}, \widetilde{\nabla} F(\boldsymbol{x}_t) \rangle\\
        &-\frac{1}{2\mu_{t}}\left\|\nabla F(\boldsymbol{x}_{t})-\widetilde{\nabla} F(\boldsymbol{x}_t) \right\|^{2}+(\frac{1}{2\eta_{t}}-\frac{\mu_{t}}{2}-\frac{L}{16})\left\|\boldsymbol{x}_{t+1}-\boldsymbol{x}_{t}\right\|^{2}.
        \end{aligned}
    \end{equation} 
    
    Take expectations on both sides and apply \cref{corollary: nonmonotone stationary point}, we have 
    \begin{equation}
        \begin{aligned}
            &\E\left(F(\bx_{t+1})-F(\bx_t) \right)
            \\&\geq \E\left(\frac{1}{2\eta_t}\left(\|\by-\bx_{t+1}\|^2-\|\by-\bx_t\|^2\right) \right) + \E\left(\E\left(\langle \by - \bx_t, \widetilde{\nabla} F(\bx_t)\rangle \bigg| \bx_t\right)\right)
            \\&\quad \quad - \E\left(\frac{1}{2\mu_t}\left\|\nabla F(\bx_t)-\widetilde{\nabla} F(\bx_t)\right\|^2 -\left(\frac{L}{16}+\frac{\mu_t}{2} - \frac{1}{2\eta_t}\right) \left\|\bx_{t+1}-\bx_t\right\|^2 \right)
            \\&= \E\left(\frac{1}{2\eta_t}\left(\|\by-\bx_{t+1}\|^2-\|\by-\bx_t\|^2\right) \right) + \E\left(\langle \by - \bx_t, \nabla F(\bx_t)\rangle\right)
            \\&\quad \quad - \E\left(\frac{1}{2\mu_t}\left\|\nabla F(\bx_t)-\widetilde{\nabla} F(\bx_t)\right\|^2 -\left(\frac{L}{16}+\frac{\mu_t}{2} - \frac{1}{2\eta_t}\right) \left\|\bx_{t+1}-\bx_t\right\|^2\right)
            \\&\geq \E\left(\frac{1}{2\eta_t}\left(\|\by-\bx_{t+1}\|^2-\|\by-\bx_t\|^2\right) \right) + \E\left( \frac{1-\|\underline{\bx}\|_{\infty}}{4}f(\by) - f\left(\frac{\bx_t+\underline{\bx}}{2}\right)\right)
            \\&\quad \quad - \E\left(\frac{1}{2\mu_t}\left\|\nabla F(\bx_t)-\widetilde{\nabla} F(\bx_t)\right\|^2 -\left(\frac{L}{16}+\frac{\mu_t}{2} - \frac{1}{2\eta_t}\right) \left\|\bx_{t+1}-\bx_t\right\|^2 \right).
        \end{aligned}
    \end{equation}
\end{proof}

Next we prove \cref{thm:4 nonmonotone}.
\begin{proof}
    Set $\by = \bx^*= \argmax_{\bx \in \mathcal{C}} f(\bx)$ in Lemma \ref{nonmonotone single step}, and let $\mu_t = \frac{1}{\eta_t}-\frac{L}{8}$, then we have
    \begin{equation}
        \begin{aligned}
            &\sum_{t=1}^{T-1}\E \left( F(\bx_{t+1})-F(\bx_t)-\frac{1-\|\underline{\bx}\|_{\infty}}{4} f(\bx^*) +f\left(\frac{\bx_t+\underline{\bx}}{2}\right)\right)
            \\&\geq \sum_{t=1}^{T-1} \E\left(\frac{1}{2\eta_t}(\|\bx^*-\bx_{t+1}\|^2-\|\bx^*-\bx_{t}\|^2)\right)-\sum_{t=1}^{T-1}\E\left(\frac{1}{2\mu_t}\left\|\nabla F(\bx_t)-\widetilde{\nabla}F(\bx_t)\right\|^2\right)
            \\&\geq -\left(\frac{3}{8}\sigma^2 + \frac{\ln (64)-4}{12} L^2 \mathrm{diam}^2(\mathcal{C})\right)\sum_{t=1}^{T-1}\frac{1}{2\mu_t}-\sum_{t=2}^{T-1} \left(\frac{1}{2\eta_{t}}-\frac{1}{2\eta_{t-1}}\right)\E\left(\|\bx^*-\bx_t\|^2\right)
            \\&\quad\quad + \E\left(\frac{\|\bx^*-\bx_{T}\|^2}{2\eta_{T-1}}-\frac{\|\bx^*-\bx_1\|^2}{2\eta_1}\right)
            \\&\geq -\left(\frac{3}{8}\sigma^2 + \frac{\ln (64)-4}{12} L^2 \mathrm{diam}^2(\mathcal{C})\right)\sum_{t=1}^{T-1}\frac{1}{2\mu_t} - \mathrm{diam}^2(\mathcal{C})\sum_{t=2}^{T-1}\left(\frac{1}{2\eta_{t}}-\frac{1}{2\eta_{t-1}}\right) - \frac{\mathrm{diam}^2(\mathcal{C})}{2\eta_1}
            \\&=-\left(\frac{3}{8}\sigma^2 + \frac{\ln (64)-4}{12} L^2 \mathrm{diam}^2(\mathcal{C})\right)\sum_{t=1}^{T-1}\frac{8\eta_t}{2(8-L\eta_t)} -\frac{\mathrm{diam}^2(\mathcal{C})}{2\eta_{T-1}}.
        \end{aligned}
    \end{equation}
    Let $\eta_t = \frac{1}{L\sqrt{t}}$, we get
    \begin{equation}
        \begin{aligned}
            &\sum_{t=1}^{T-1}\E \left( F(\bx_{t+1})-F(\bx_t)-\frac{1-\|\underline{\bx}\|_{\infty}}{4}f(\bx^*) +f\left(\frac{\bx_t+\underline{\bx}}{2}\right)\right)
            \\&\geq -\left(\frac{3}{8}\sigma^2 + \frac{\ln (64)-4}{12} L^2 \mathrm{diam}^2(\mathcal{C})\right)\sum_{t=1}^{T-1}\frac{4}{7L\sqrt{t}} - L\mathrm{diam}^2(\mathcal{C})\sqrt{T-1}
            \\&\geq -\frac{4}{7L}\left(\frac{3}{8}\sigma^2 + \left(\frac{\ln (64)-4}{12} L^2+\frac{4L^{2}}{7}\right) \mathrm{diam}^2(\mathcal{C})\right)\sqrt{T-1}.
        \end{aligned}
    \end{equation}
    Then,
    \begin{equation}
        \begin{aligned}
            &\E\left(\sum_{t=1}^{T-1} f\left(\frac{\bx_t+\underline{\bx}}{2}\right)\right)
            \\&\geq \frac{(1-\|\bx\|_{\infty})(T-1)}{4}f(\bx^*)- -\frac{4}{7L}\left(\frac{3}{8}\sigma^2 + \left(\frac{\ln (64)-4}{12} L^2+\frac{4L^{2}}{7}\right) \mathrm{diam}^2(\mathcal{C})\right)\sqrt{T-1}
            \\&\quad \quad +  F(\bx_1)-F(\bx_T)
            \\&\geq \frac{(1-\|\bx\|_{\infty})(T-1)}{4}f(\bx^*)-\frac{4}{7L}\left(\frac{3}{8}\sigma^2 + \left(\frac{\ln (64)-4}{12} L^2+\frac{4L^{2}}{7}\right) \mathrm{diam}^2(\mathcal{C})\right)\sqrt{T-1}
            \\&\quad \quad - \frac{L \mathrm{diam}^2(\mathcal{C})}{8}.
        \end{aligned}
    \end{equation}
    Which shows,
    \begin{equation}
        \begin{aligned}  
            &\E\left(\sum_{t=1}^{T-1} \frac{1}{T-1} f\left(\frac{\bx_t+\underline{\bx}}{2}\right)\right)
            \\&\geq \frac{(1-\|\bx\|_{\infty})}{4}f(\bx^*)- \frac{4}{7L}\left(\frac{3}{8}\sigma^2 + \left(\frac{\ln (64)-4}{12} L^2+\frac{4L^{2}}{7}\right) \mathrm{diam}^2(\mathcal{C})\right)\frac{1}{\sqrt{T-1}}
            \\&\quad \quad - \frac{L_1 \mathrm{diam}(\mathcal{C})}{8(T-1)}
            \\&\geq \frac{(1-\|\bx\|_{\infty})}{4}f(\bx^*)- O\left(\frac{1}{\sqrt{T}}\right).
        \end{aligned}
    \end{equation}
    Therefore, the solution returned by \cref{alg:1} satisfies
    \[\E\left(f\left(\frac{\bx_l+\underline{\bx}}{2}\right)\right)\geq \frac{1-\|\bx\|_{\infty}}{4}f(\bx^*)- O\left(\frac{1}{\sqrt{T}}\right).\]
\end{proof}
\section{Proofs in Section~\ref{sec:delay}}\label{Appendix:D}
Since it will not lead to ambiguity, we omit the subscripts of $F_{\uparrow,t}$ and $F_{\sim,t}$ and use $F_t$ to represent both non-oblivious functions of $f_t$ according to its monotonicity.
\subsection{Proof of \texorpdfstring{\cref{thm:5}}{21}}\label{Appendix:D1}
\begin{proof}
    We denote $\widetilde{\nabla} F_{t}(\boldsymbol{x}_t)=\frac{1-e^{-\gamma}}{\gamma}\widetilde{\nabla}f(z_{t}\cdot\boldsymbol{x}_{t})$ and $\boldsymbol{x}^{*}=\arg\max_{\boldsymbol{x}\in\mathcal{C}}\sum_{t=1}^{T}f_{t}(\boldsymbol{x})$. From the projection, we know that
    \begin{equation}\label{equ:b43}
        \begin{aligned}
           \left\|\boldsymbol{x}_{t+1}-\boldsymbol{x}^{*}\right\|&\le\left\|\boldsymbol{y}_{t+1}-\boldsymbol{x}^{*}\right\|=\left\|\boldsymbol{x}_{t}+\eta\sum_{s\in\mathcal{F}_{t}}\widetilde{\nabla} F_{s}(\boldsymbol{x}_{s})-\boldsymbol{x}^{*}\right\|,
        \end{aligned}
    \end{equation}where the first inequality from the projection; and the first equality from $\boldsymbol{y}_{t+1}=\boldsymbol{x}_{t}+\eta\sum_{s\in\mathcal{F}_{t}}\frac{1-e^{-\gamma}}{\gamma}\widetilde{\nabla}f_{s}(z_{s}\cdot\boldsymbol{x}_{s})$ in \cref{alg:2}.
    
    We order the set  $\mathcal{F}_{t}=\{s_1,\dots,s_{|\mathcal{F}_{t}|}\}$, where $s_1<s_2<\dots<s_{|\mathcal{F}_{t}|}$ and $|\mathcal{F}_{t}|=\#\{u\in[T]: u+d_{u}-1=t\}$. Moreover, we also denote $\mathcal{F}_{t,m}=\{u\in\mathcal{F}_{t}\ and\ u<m\}$, $\boldsymbol{x}_{t+1,m}=x_{t}+\eta\sum_{s\in\mathcal{F}_{t,m}}\widetilde{\nabla} F_{s}(\boldsymbol{x}_{s})$ and $s_{|\F_{t}|+1}=t+1$. Therefore,
    \begin{equation}\label{equ:b44}
        \begin{aligned}
           \left\|\boldsymbol{x}_{t+1,s_{k+1}}-\boldsymbol{x}^{*}\right\|^{2}&=\left\|\boldsymbol{x}_{t+1,s_{k}}+\eta\widetilde{\nabla}F_{s_{k}}(\boldsymbol{x}_{s_{k}})-\boldsymbol{x}^{*}\right\|^{2}\\
           &=\left\|\boldsymbol{x}_{t+1,s_{k}}-\boldsymbol{x}^{*}\right\|^{2}+2\eta\langle \boldsymbol{x}_{t+1,s_{k}}-\boldsymbol{x}^{*},\widetilde{\nabla}F_{s_{k}}(\boldsymbol{x}_{s_{k}})\rangle+\eta^2\left\|\widetilde{\nabla}F_{s_{k}}(\boldsymbol{x}_{s_{k}})\right\|^{2}
           \end{aligned}
    \end{equation}
    According to \cref{equ:b44}, we have
    
    \begin{equation}\label{equ:b45}
        \begin{aligned}
           &\left\|\boldsymbol{y}_{t+1}-\boldsymbol{x}^{*}\right\|^{2}-\left\|\boldsymbol{x}_{t}-\boldsymbol{x}^{*}\right\|^{2}\\
           = & \sum_{k=1}^{|\mathcal{F}_{t}|}(\left\|\boldsymbol{x}_{t+1,s_{k+1}}-\boldsymbol{x}^{*}\right\|^{2}-\left\|\boldsymbol{x}_{t+1,s_{k}}-\boldsymbol{x}^{*}\right\|^{2})\\
           = & 2\eta\sum_{s\in\mathcal{F}_{t}}\langle \boldsymbol{x}_{t+1,s}-\boldsymbol{x}^{*},\widetilde{\nabla}F_{s}(\boldsymbol{x}_{s})\rangle+\eta^2\sum_{s\in\mathcal{F}_{t}}\left\|\widetilde{\nabla}F_{s}(\boldsymbol{x}_{s})\right\|^{2}\\
           = & 2\eta\sum_{s\in\mathcal{F}_{t}}\langle \boldsymbol{x}_{t+1,s}-\boldsymbol{x}_{s},\widetilde{\nabla}F_{s}(\boldsymbol{x}_{s})\rangle+2\eta\sum_{s\in\mathcal{F}_{t}}\langle \boldsymbol{x}_{s}-\boldsymbol{x}^{*},\widetilde{\nabla}F_{s}(\boldsymbol{x}_{s})\rangle+\eta^2\sum_{s\in\mathcal{F}_{t}}\left\|\widetilde{\nabla}F_{s}(\boldsymbol{x}_{s})\right\|^{2},
           \end{aligned}
    \end{equation} where the first equality follows from setting $\boldsymbol{x}_{t+1,|\mathcal{F}_{t}|+1}=\boldsymbol{y}_{t+1}$; the second from \cref{equ:b44}. 
    
    \noindent Therefore,
    \begin{equation}\label{equ:b46}
        \begin{aligned}
           &\mathbb{E}(\left\|\boldsymbol{y}_{t+1}-\boldsymbol{x}^{*}\right\|^{2}-\left\|\boldsymbol{x}_{t}-\boldsymbol{x}^{*}\right\|^{2})\\
           = & 2\eta\mathbb{E}\left(\sum_{s\in\mathcal{F}_{t}}\langle \boldsymbol{x}_{t+1,s}-\boldsymbol{x}_{s},\widetilde{\nabla}F_{s}(\boldsymbol{x}_{s})\rangle+\sum_{s\in\mathcal{F}_{t}}\langle \boldsymbol{x}_{s}-\boldsymbol{x}^{*},\mathbb{E}(\widetilde{\nabla}F_{s}(\boldsymbol{x}_{s})|\boldsymbol{x}_{s})\rangle\right)+\eta^2\mathbb{E}(\sum_{s\in\mathcal{F}_{t}}\left\|\widetilde{\nabla}F_{s}(\boldsymbol{x}_{s})\right\|^{2})\\
           = & 2\eta\mathbb{E}\left(\sum_{s\in\mathcal{F}_{t}}\langle \boldsymbol{x}_{t+1,s}-\boldsymbol{x}_{s},\widetilde{\nabla}F_{s}(\boldsymbol{x}_{s})\rangle+\sum_{s\in\mathcal{F}_{t}}\langle \boldsymbol{x}_{s}-\boldsymbol{x}^{*},\nabla F_{s}(\boldsymbol{x}_{s})\rangle\right)+\eta^2\mathbb{E}(\sum_{s\in\mathcal{F}_{t}}\left\|\widetilde{\nabla}F_{s}(\boldsymbol{x}_{s})\right\|^{2})\\
           \le & 2\eta\mathbb{E}\left(\sum_{s\in\mathcal{F}_{t}}\langle \boldsymbol{x}_{t+1,s}-\boldsymbol{x}_{s},\widetilde{\nabla}F_{s}(\boldsymbol{x}_{s})\rangle+\sum_{s\in\mathcal{F}_{t}}\left(f_{s}(\boldsymbol{x}_{s})-(1-e^{-\gamma})f_{s}(\boldsymbol{x}^{*})\right)\right)+\eta^2\mathbb{E}(\sum_{s\in\mathcal{F}_{t}}\left\|\widetilde{\nabla}F_{s}(\boldsymbol{x}_{s})\right\|^{2})
          \end{aligned}
    \end{equation}where the first inequality from the definition of non-oblivious function $F$.
    
    \noindent Therefore, we have:
    \begin{equation}\label{equ:b47}
        \begin{aligned}
           &2\eta\mathbb{E}\left((1-e^{-\gamma})\sum_{t=1}^{T}f_t(\boldsymbol{x}^{*})-\sum_{t=1}^{T}f_{t}(\boldsymbol{x}_{t})\right)\\
           = & 2\eta\mathbb{E}\left(\sum_{t=1}^{T}\sum_{s\in\mathcal{F}_{t}}\left((1-e^{-\gamma})f_s(\boldsymbol{x}^{*})-f_{s}(\boldsymbol{x}_{s})\right)\right)\\
           \le & \sum_{t=1}^{T}\left(\mathbb{E}(\left\|\boldsymbol{x}_{t}-\boldsymbol{x}^{*}\right\|^{2}-\left\|\boldsymbol{y}_{t+1}-\boldsymbol{x}^{*}\right\|^{2})+2\eta\mathbb{E}(\sum_{s\in\mathcal{F}_{t}}\langle \boldsymbol{x}_{t+1,s}-\boldsymbol{x}_{s},\widetilde{\nabla}F_{s}(\boldsymbol{x}_{s})\rangle)+\eta^2\mathbb{E}(\sum_{s\in\mathcal{F}_{t}}\left\|\widetilde{\nabla}F_{s}(\boldsymbol{x}_{s})\right\|^{2})\right)\\
           \le & \sum_{t=1}^{T}\left(\mathbb{E}(\left\|\boldsymbol{x}_{t}-\boldsymbol{x}^{*}\right\|^{2}-\left\|\boldsymbol{x}_{t+1}-\boldsymbol{x}^{*}\right\|^{2})+2\eta\mathbb{E}(\sum_{s\in\mathcal{F}_{t}}\langle \boldsymbol{x}_{t+1,s}-\boldsymbol{x}_{s},\widetilde{\nabla}F_{s}(\boldsymbol{x}_{s})\rangle)+\eta^2\mathbb{E}(\sum_{s\in\mathcal{F}_{t}}\left\|\widetilde{\nabla}F_{s}(\boldsymbol{x}_{s})\right\|^{2})\right)\\
           \le & \mathrm{diam}^{2}(\mathcal{C})+\sum_{t=1}^{T}\left(2\eta\mathbb{E}(\sum_{s\in\mathcal{F}_{t}}\langle \boldsymbol{x}_{t+1,s}-\boldsymbol{x}_{s},\widetilde{\nabla}F_{s}(\boldsymbol{x}_{s})\rangle)+\eta^2\mathbb{E}(\sum_{s\in\mathcal{F}_{t}}\left\|\widetilde{\nabla}F_{s}(\boldsymbol{x}_{s})\right\|^{2})\right)\\
           \le & \mathrm{diam}^{2}(\mathcal{C})+\eta^{2}\E\left(\sum_{t=1}^{T}|\mathcal{F}_{t}|\left(\frac{1-e^{-\gamma}}{\gamma}\right)^2\|\widetilde{\nabla}f(z_t\cdot\bx_t)\|^2\right)+2\eta\sum_{t=1}^{T}\left(\mathbb{E}\left(\sum_{s\in\mathcal{F}_{t}}\langle \boldsymbol{x}_{t+1,s}-\boldsymbol{x}_{s},\widetilde{\nabla}F_{s}(\boldsymbol{x}_{s})\rangle\right)\right)\\
           \leq & \mathrm{diam}^{2}(\mathcal{C})+\left(\frac{1-e^{-\gamma}}{\gamma}\right)^2\eta^2 \widetilde{G}^2 T + 2\eta\sum_{t=1}^{T}\left(\mathbb{E}\left(\sum_{s\in\mathcal{F}_{t}}\langle \boldsymbol{x}_{t+1,s}-\boldsymbol{x}_{s},\widetilde{\nabla}F_{s}(\boldsymbol{x}_{s})\rangle\right)\right).
        \end{aligned}
    \end{equation}
    \noindent For the final part in \cref{equ:b47}, 
    \begin{equation}\label{equ:b48}
        \begin{aligned}
           &\E\left(\langle \boldsymbol{x}_{t+1,s}-\boldsymbol{x}_{s},\widetilde{\nabla}F_{s}(\boldsymbol{x}_{s})\rangle\right)\\
           \le & \E\left(\left\|\widetilde{\nabla}F_{s}(\boldsymbol{x}_{s})\right\|\left\|\boldsymbol{x}_{t+1,s}-\boldsymbol{x}_{s}\right\|\right)\\
           \le & \E\left(\left\|\widetilde{\nabla}F_{s}(\boldsymbol{x}_{s})\right\|(\left\|\boldsymbol{x}_{t+1,s}-\boldsymbol{x}_{t}\right\|+\left\|\boldsymbol{x}_{t}-\boldsymbol{x}_{s}\right\|)\right)\\
           \le & \E\left(\left\|\widetilde{\nabla}F_{s}(\boldsymbol{x}_{s})\right\|(\left\|\boldsymbol{x}_{t+1,s}-\boldsymbol{x}_{t}\right\|+\sum_{m=s}^{t-1}\left\|\boldsymbol{y}_{m+1}-\boldsymbol{x}_{m}\right\|)\right)\\
           \le & \E\left(\left\|\widetilde{\nabla}F_{s}(\boldsymbol{x}_{s})\right\|( \sum_{k\in\mathcal{F}_{t,s}}\eta\|\widetilde{\nabla}F_k(\bx_k)\|+\sum_{m=s}^{t-1} \sum_{k\in\mathcal{F}_{m}}\eta\|\widetilde{\nabla}F_k(\bx_k)\|)\right) \\
            \le & \left(\frac{1-e^{-\gamma}}{\gamma}\right)^2 \eta(|\mathcal{F}_{t,s}|+\sum_{m=s}^{t-1}|\mathcal{F}_{m}|)\widetilde{G}^2
        \end{aligned}
    \end{equation} where the third inequality follows from $\left\|\boldsymbol{x}_{t}-\boldsymbol{x}_{s}\right\|\le\left\|\boldsymbol{y}_{t}-\boldsymbol{x}_{s}\right\|\le\left\|\boldsymbol{y}_{t}-\boldsymbol{x}_{t-1}\right\|+\left\|\boldsymbol{x}_{t-1}-\boldsymbol{x}_{s}\right\|\le\dots\le\sum_{m=s}^{t-1}\left\|\boldsymbol{y}_{m+1}-\boldsymbol{x}_{m}\right\|$.
    
    \noindent Finally, we have
    \begin{equation}\label{equ:b49}
        \begin{aligned}
           &\mathbb{E}\left((1-e^{-\gamma})\sum_{t=1}^{T}f_t(\boldsymbol{x}^{*})-\sum_{t=1}^{T}f_{t}(\boldsymbol{x}_{t})\right)\\
           \le & \frac{\mathrm{diam}^{2}(\mathcal{C})}{2\eta}+\left(\frac{1-e^{-\gamma}}{\gamma}\right)^2\frac{\eta \widetilde{G}^2 T}{2}+\left(\frac{1-e^{-\gamma}}{\gamma}\right)^2\eta \widetilde{G}^2\sum_{t=1}^{T}\sum_{s\in\mathcal{F}_{t}}\left(|\mathcal{F}_{t,s}|+\sum_{m=s}^{t-1}|\mathcal{F}_{m}|\right)
        \end{aligned}
    \end{equation}
    
    Firstly, $\sum_{t=1}^{T}|\mathcal{F}_{t}|\le T$. Next, we investigate the $|\mathcal{F}_{t,s}|+\sum_{m=s}^{t-1}|\mathcal{F}_{m}|$ when $s\in\mathcal{F}_{t}$. 
    
    When $s\in\mathcal{F}_{t}$, i.e., $s+d_{s}-1=t$, for any $q\in(\mathcal{F}_{t,s}\bigcup(\cup_{m=s}^{t-1}\mathcal{F}_{m}))$, if $s+1\le q\le t-1$, the feedback of round $q$ must be delivered before the round $t$, namely, $q+d_{q}-1\le t-1$. Moreover, if $q\le s-1$, the feedback of round $q$ could be delivered between round $s$ and round $t$. Therefore, 
    \begin{equation}\label{equ:b50}
        \begin{aligned}
           |\mathcal{F}_{t,s}|+\sum_{m=s}^{t-1}|\mathcal{F}_{m}|= & |\{i| s+1\le i\le t-1,\ and \ i+d_{i}-1\le t-1 \}|\\&+|\{i| 1\le i\le s-1,\ and \ s\le i+d_{i}-1\le t \}|.
        \end{aligned}
    \end{equation} 
    
    \noindent When $s\in\mathcal{F}_{t}$, we can derive that $|\{i| s+1\le i\le t-1,\ and\ i+d_{i}-1\le t-1 \}|\le t-s-1\le d_{s}$. Thus, $\sum_{t=1}^{T}\sum_{s\in\mathcal{F}_{t}}|\{i| s+1\le i\le t-1,\ and\ i+d_{i}-1\le t-1 \}|\le\sum_{i=1}^{T}d_{i}=D$.
    
    Next, for each $b\in\{i| 1\le i\le s-1,\ and\ s\le i+d_{i}-1\le t \}$, we have $b\le s\le b+d_{b}-1\le s+d_{s}-1$ so that $\sum_{t=1}^{T}\sum_{s\in \mathcal{F}_{t}}|\{i| 1\le i\le s-1,\ and\ s \le i+d_{i}-1\le t \}|\le\sum_{i=1}^{T}|\{s|\  i<s\le i+d_{i}-1\le s+d_{s}-1\}|\le\sum_{i=1}^{T} d_{i}$. 
    
    Hence,
    \begin{equation}\label{equ:51}
        \begin{aligned}
           &\mathbb{E}\left((1-e^{-\gamma})\sum_{t=1}^{T}f_t(\boldsymbol{x}^{*})-\sum_{t=1}^{T}f_{t}(\boldsymbol{x}_{t})\right)\\
           \le & \frac{\mathrm{diam}^{2}(\mathcal{C})}{2\eta}+\left(\frac{1-e^{-\gamma}}{\gamma}\right)^2\frac{\eta \widetilde{G}^2 T}{2} + \left(\frac{1-e^{-\gamma}}{\gamma}\right)^2\eta \widetilde{G}^2 D\\
           \le & O(\sqrt{D}).
        \end{aligned}
    \end{equation} where the final equality from $\eta=\frac{\mathrm{diam}(\mathcal{C})}{\widetilde{G}\sqrt{D}}$.
    \end{proof}

\subsection{Proof of \texorpdfstring{\cref{thm:5 nonmonotone}}{22}}\label{Appendix:D2}
\begin{proof}
    In this proof, $\widetilde{\nabla} F_s$ represent $\widetilde{\nabla} F_{\sim,s}$, we omit the subscript $\sim$ which indicates that it is a non-oblivious function designed for non-monotone functions.
    
    Similar to (\ref{equ:b46}), we have
    \begin{equation}
        \begin{aligned}
            &\E\left(\|\boldsymbol{y}_{t+1}-\bx^*\|^2-\|\bx_t-\bx^*\|^2\|\right)\\
            =&2\eta\mathbb{E}\left(\sum_{s\in\mathcal{F}_{t}}\langle \boldsymbol{x}_{t+1,s}-\boldsymbol{x}_{s},\widetilde{\nabla}F_{s}(\boldsymbol{x}_{s})\rangle+\sum_{s\in\mathcal{F}_{t}}\langle \boldsymbol{x}_{s}-\boldsymbol{x}^{*},\nabla F_{s}(\boldsymbol{x}_{s})\rangle\right)+\eta^2\mathbb{E}(\sum_{s\in\mathcal{F}_{t}}\left\|\widetilde{\nabla}F_{s}(\boldsymbol{x}_{s})\right\|^{2})\\
            \leq & 2\eta \mathbb{E}\left(\sum_{s\in\mathcal{F}_{t}}\langle \boldsymbol{x}_{t+1,s}-\boldsymbol{x}_{s},\widetilde{\nabla}F_{s}(\boldsymbol{x}_{s})\rangle+\sum_{s\in\mathcal{F}_{t}} \left(f_s\left(\frac{\bx_s+\underline{\bx}}{2}\right) -\frac{1-\|\underline{\bx}\|_{\infty}}{4}f_s(\bx^*)\right) \right)+\eta^2\mathbb{E}\left(\sum_{s\in\mathcal{F}_{t}}\left\|\widetilde{\nabla}F_{s}(\boldsymbol{x}_{s})\right\|^{2}\right).
        \end{aligned}
    \end{equation}
    Therefore, the following inequality holds from the similar derivation in (\ref{equ:b47}) $\sim$ (\ref{equ:b50}),
    \begin{equation}
        \begin{aligned}
            &2\eta \E\left(\frac{1-\|\underline{\bx}\|_{\infty}}{4}\sum_{t=1}^T f_t(\bx^*) - \sum_{t=1}^T f_t\left(\frac{\bx_t+\underline{\bx}}{2}\right)\right)
            \\\leq &\mathrm{diam}^2(\mathcal{C}) + \sum_{t=1}^T\left(2\eta \E\left(\sum_{s\in \mathcal{F}_t}\langle\bx_{t+1,s}-\bx_s,\widetilde{\nabla} F_s(\bx_s)\rangle\right) + \eta^2 \E\left(\sum_{s\in \mathcal{F}_t}\|\widetilde{\nabla}F_s(\bx_s)\|^2\right)\right)\\
            \le & \mathrm{diam}^{2}(\mathcal{C})+\eta^{2}\E\left(\sum_{t=1}^{T}\frac{3|\mathcal{F}_{t}|}{8}\left\|\widetilde{\nabla}f\left(\frac{z_t}{2}\cdot\bx_t+(1-\frac{z_t}{2})\cdot\underline{\bx}\right)\right\|^2\right)+2\eta\sum_{t=1}^{T}\left(\mathbb{E}\left(\sum_{s\in\mathcal{F}_{t}}\langle \boldsymbol{x}_{t+1,s}-\boldsymbol{x}_{s},\widetilde{\nabla}F_{s}(\boldsymbol{x}_{s})\rangle\right)\right)\\
           \leq & \mathrm{diam}^{2}(\mathcal{C})+\frac{3}{8}\eta^2 \widetilde{G}^2 T+ 2\eta\sum_{t=1}^{T}\left(\mathbb{E}\left(\sum_{s\in\mathcal{F}_{t}}\langle \boldsymbol{x}_{t+1,s}-\boldsymbol{x}_{s},\widetilde{\nabla}F_{s}(\boldsymbol{x}_{s})\rangle\right)\right)\\
           \leq &\mathrm{diam}^{2}(\mathcal{C})+\frac{3}{8}\eta^2 \widetilde{G}^2 T + \frac{3}{4}\eta^2 \widetilde{G}^2 D .
        \end{aligned}
    \end{equation}
    Hence,
    \begin{equation}
        \begin{aligned}
            \E\left(\frac{1-\|\underline{\bx}\|_{\infty}}{4}\sum_{t=1}^T f_t(\bx^*) -\sum_{t=1}^T f_t\left(\frac{\bx_t+\underline{\bx}}{2}\right)\right)
            \leq & \frac{\mathrm{diam}^2(\mathcal{C})}{2\eta} + \frac{9\eta \widetilde{G}^2 D}{16}\\
            = & O(\sqrt{D}),
        \end{aligned}
    \end{equation}
    where the final equality is from $\eta=\frac{\mathrm{diam}(\mathcal{C})}{\widetilde{G}\sqrt{D}}$.
\end{proof}
\section{Proofs in Section~\ref{sec:bandit}}\label{Appendix:bandit}
\subsection{Proof of \texorpdfstring{\cref{lem:interior}}{23}} \label{sect:proof of lemma 6}
\begin{proof}
    We first show the convexity of $\mathcal{C}_{\delta,\by}$. Let $\bx_1,\bx_2\in \mathcal{C}_{\delta,\by}$ and $h\in [0,1]$. Consider their convex combination $\bx_3:=h\bx_1+(1-h)\bx_2$. Since 
    \begin{align*}
        \by+ (1+\delta) (\bx_3-\by)= h(\by+ (1+\delta) (\bx_1-\by))+ (1-h)\left(\by+ (1+\delta) (\bx_2-\by)\right)
    \end{align*}
    and $\by+ (1+\delta) (\bx_1-\by)\in \mathcal{C}$, $\by+ (1+\delta) (\bx_2-\by)\in \mathcal{C}$, we have $\by+ (1+\delta) (\bx_3-\by)\in \mathcal{C}$. Then $\pi_{\by}(\bx_3)\leq (1+\delta)^{-1}$ and $\bx_3\in\mathcal{C}_{\delta,\by}$, which shows the convexity of $\mathcal{C}_{\delta,\by}$.

    Then we turn to prove $\mathbb{B}(\bx,\delta R)\subseteq \mathcal{C}$. Let $\bv \in \mathbb{B}(\boldsymbol{0},R)$, consider $\bx + \delta \bv$. Check that
    \[\bx+\frac{\delta}{1+\delta}\bv = \frac{\delta}{1+\delta}\left(\by+ \bv\right) + \frac{1}{1+\delta}\left(\by+(1+\delta)(\bx-\by)\right) \]
    holds. Since $\by+\bv\in \mathbb{B}(\by,R)\subseteq \mathcal{C}$ and $\by+(1+\delta)(\bx-\by)\in\mathcal{C}$, we have $\bx+\frac{\delta}{1+\delta}\bv\in\mathcal{C}$. Then $\mathbb{B}(\bx,\frac{\delta}{1+\delta} R)\subseteq \mathcal{C}$.
\end{proof}

\subsection{Supporting Lemmas}
The following lemma shows that by project some point $\by\in \mathrm{int}(\mathcal{C})$ onto the Minkowski set, we obtain a point that is close to $\by$.
\begin{lemma}[\citep{abernethy2008competing}]\label{lem:Minkowski projection}
    Let $\mathcal{C}$ be a compact convex set, $\by\in \mathrm{int}(\mathcal{C})$, $\by^*\in \mathcal{C}$ and $\hat{\by}^*\triangleq \mathcal{P}_{\mathcal{C}_{\delta',\by}}(\by^*)$ be the projection of $\by^*$ onto the Minkowski set $\mathcal{C}_{\delta',\by}$, then
    \[\|\by^*-\hat{\by}^*\|\leq \frac{\delta'}{1+\delta'} \mathrm{diam}(\mathcal{C}).\]
    Moreover, if $\delta'$ is set to $\frac{\delta}{R-\delta}$, we have
    \[\|\by^*-\hat{\by}^*\|\leq \frac{\delta}{R} \mathrm{diam}(\mathcal{C}).\]
\end{lemma}
\begin{proof}
    Consider the point $\bx$ in the segment $[\by,\by^*]$ satisfying $\frac{\|\bx-\by\|}{\|\by^*-\by\|}=\frac{1}{1+\delta}$. Since $\by+(1+\delta')(\bx-\by)=\by^*\in \mathcal{K}$, we can deduce that $\bx\in \mathcal{C}_{\delta,\by}$. Thus,
    \[\|\hat{\by}^*-\by^*\|\leq \|\bx-\by^*\|=\left(1-\frac{1}{1+\delta'}\right)\|\by^*-
    \by\|\leq \frac{\delta'}{1+\delta'}\mathrm{diam}(\mathcal{C}).\]
    Plug $\delta'=\frac{\delta}{R-\delta}$, we have
    \[\|\by^*-\hat{\by}^*\|\leq \frac{\delta}{R} \mathrm{diam}(\mathcal{C}).\]
\end{proof}

We then show the properties of the $\delta$-smoothed version of DR-submodular functions in the following lemma.
\begin{lemma}\label{lem:submodular smooth}
    The following properties hold for $\delta$-smoothed version of a twice differentiable function $f(\boldsymbol{x})$.
    \begin{itemize}
        \item[(i)] If $f(\bx)$ is monotone, then its $\delta$-smoothed version $\hat{f}^{\delta}(\bx)$ is also monotone.
        \item[(ii)] If $f(\boldsymbol{x})$ is DR-submodular, then its $\delta$-smoothed version $\hat{f}^{\delta}(\boldsymbol{x})$ is also DR-submodular. Moreover, if $f(\boldsymbol{x})$ is monotone $\gamma$-weakly DR-submodular, then $\hat{f}^{\delta}(\bx)$ is also monotone $\gamma$-weakly DR-submodular.
        \item[(iii)] If $f(\boldsymbol{x})$ is $L_1$-lipschitz continuous and $L_2$-smooth, then $\hat{f}^{\delta}(\boldsymbol{x})$ is $L_1$-lipschitz continuous and $L_2$-smooth. 
    \end{itemize}
\end{lemma}

\begin{proof}
    \begin{itemize}
        \item[(i)] If $\bx\leq \by$, then
        \begin{align*}
            \hat{f}^{\delta}(\bx) &= \E_{\boldsymbol{u} \sim \mathbb{B}_d}\left(f(\bx+ \delta\boldsymbol{u})\right)
            \leq \E_{\boldsymbol{u} \sim \mathbb{B}_d}\left(f(\by+ \delta\boldsymbol{u})\right)
            =\hat{f}^{\delta}(\by),
        \end{align*}
        which shows the monotonicity of $\hat{f}^{\delta}$.
        \item[(ii)] By Leibniz integral rule, for any $i,j\in [d]$,
    \begin{align*}
        \frac{\partial}{\partial x_i x_j}\hat{f}^{\delta}(\boldsymbol{x}) &= \int_{\boldsymbol{u}\in\mathbb{B}_{d}}\frac{1}{\mbox{Vol}(\mathbb{B}_d)}\frac{\partial}{\partial x_i x_j}f(\boldsymbol{x}+\delta\boldsymbol{v})d\boldsymbol{v}
        \\&\leq 0.
    \end{align*}
    The last inequality is because $\frac{\partial}{\partial \boldsymbol{x}_i \boldsymbol{x}_j}f(\boldsymbol{x}+\delta\boldsymbol{v})\leq 0$ for any $i,j\in [d]$. Moreover, if $f$ is $\gamma$-weakly DR-submodular, then for any $\bx\leq \by$ and $i\in [n]$,
    \begin{align*}
        \frac{\partial \hat{f}^{\delta}}{\partial x_i} (\bx) &=  \int_{\boldsymbol{u}\in\mathbb{B}_{d}}  \frac{1}{\mbox{Vol}(\mathbb{B}_d)}\frac{\partial f}{\partial x_i }(\bx+\delta \boldsymbol{u})d\boldsymbol{u}
        \\&\geq \int_{\boldsymbol{u}\in\mathbb{B}_{d}}  \frac{\gamma}{\mbox{Vol}(\mathbb{B}_d)}\frac{\partial f}{\partial x_i }(\by+\delta \boldsymbol{u})d\boldsymbol{u} = \gamma\frac{\partial \hat{f}^{\delta}}{\partial x_i} (\by).
    \end{align*}
    Since the inequality holds for any $\bx\leq \by$ and $i\in [n]$,
    \[\inf_{\bx\leq \by}\inf_{i\in [n]} \frac{[\nabla \hat{f}^{\delta}(\bx)]_i}{[\nabla \hat{f}^{\delta}(\by)]_i}\geq \gamma.\]
    Then, $\hat{f}^{\delta}$ is $\gamma$-weakly DR-submodular.
        \item[(ii)] For any $\bx,\by$,
            \begin{align*}
                \hat{f}^{\delta}(\boldsymbol{x})- \hat{f}^{\delta}(\boldsymbol{y})&= \int_{\boldsymbol{u}\in\mathbb{B}_d} \frac{1}{\mbox{Vol}(\mathbb{B}_d)}\left(f(\boldsymbol{x}+\delta\boldsymbol{u})-f(\boldsymbol{y}+\delta\boldsymbol{u})\right)d\boldsymbol{u}
                \\&\leq \int_{\boldsymbol{v}\in\mathbb{B}_d} \frac{1}{\mbox{Vol}(\mathbb{B}_d)}L_1\|\boldsymbol{x}+\delta\boldsymbol{u}-\boldsymbol{y}-\delta\boldsymbol{u}\|d\boldsymbol{v}
                \\&=L_1 \|\boldsymbol{x}-\boldsymbol{y}\|.
            \end{align*}
            Thus, $\hat{f}^{\delta}(\boldsymbol{x})$ is $L_1$-lipschitz continuous.
            \begin{align*}
                \nabla \hat{f}^{\delta}(\boldsymbol{x})- \nabla \hat{f}^{\delta}(\boldsymbol{y})&= \nabla \int_{\boldsymbol{u}\in\mathbb{B}_d} \frac{1}{\mbox{Vol}(\mathbb{B}_d)}\left(f(\boldsymbol{x}+\delta\boldsymbol{u})-f(\boldsymbol{y}+\delta\boldsymbol{u})\right)d\boldsymbol{u}
                \\&=\int_{\boldsymbol{u}\in\mathbb{B}_d} \frac{1}{\mbox{Vol}(\mathbb{B}_d)}\left(\nabla f(\boldsymbol{x}+\delta\boldsymbol{u})-\nabla f(\boldsymbol{y}+\delta\boldsymbol{u})\right)d\boldsymbol{u}
                \\&\leq \int_{\boldsymbol{u}\in\mathbb{B}_d} \frac{1}{\mbox{Vol}(\mathbb{B}_d)}L_2 \|\boldsymbol{x}-\boldsymbol{y}\|d\boldsymbol{u}
                \\&=L_2\|\boldsymbol{x}-\boldsymbol{y}\|.
            \end{align*}
    \end{itemize} 
\end{proof}
From \cref{lem:submodular smooth}, the $\delta$-smoothed version of $f$ inherits the good properties of $f$, such as DR-submodularity, monotonicity and smoothness. This indicates that our non-oblivous boosting technique can also apply on $\hat{f}$. 

The next lemma shows that $\hat{f}$ is not far from $f$.

\begin{lemma}\label{lem:smoothed version bias}
    If $f$ is $L_1$-lipschitz, then for any $\by$, $|\hat{f}^{\delta}(\by) - f(\by)|\leq L_1\delta$.
\end{lemma}
\begin{proof}
    \begin{align*}
        |\hat{f}^{\delta}(\by) - f(\by)| &= \left| \E_{\boldsymbol{u}\sim \delta \mathbb{B}_d} \left(f(\by+\boldsymbol{u}) - f(\by)\right)\right|
        \\&\leq L_1\cdot \E_{\boldsymbol{u}\sim \delta \mathbb{B}_d}\left(\|\boldsymbol{u}\|\right)\leq L_1\delta.
    \end{align*}
\end{proof}

Let $\hat{F}_t^{\delta}$ denotes the non-oblivious function of $\hat{f}_t$. That is, $\nabla \hat{F}_t^{\delta}(\bx) = \int_0^1 e^{\gamma(z-1)}\nabla \hat{f}^{\delta}_t(z\cdot\boldsymbol{x})\mathrm{d}z$ if we select option I, and $\nabla \hat{F}_t^{\delta}(\bx) = \int_0^1\frac{1}{8(1-\frac{z}{2})^3}\nabla \hat{f}_{t}^{\delta}\left(\frac{z}{2}(\bx-\underline{\bx})+\underline{\bx}\right) \mathrm{d}z$ if we select option II. The following lemma shows that $\widetilde{\nabla}F_t(\by_t)$ is an unbiased estimate of $\nabla \hat{F}_t^{\delta}(\by_t)$ with bounded variance.

\begin{lemma}\label{lem:estimate option I}
    If we select Option I in \cref{algo:BBGA}, the following holds
    \begin{itemize}
        \item[(i)] $\left\|\E\left(\left.\widetilde{\nabla} F_t(\by_t) \right| \by_t\right)-\nabla \hat{F}_t^{\delta}(\by_t)\right\| \leq \frac{1-e^{-\gamma}}{\gamma} L_2 \delta \mathrm{diam}(\mathcal{C})$.
        \item[(ii)] $\E\left(\left.\|\widetilde{\nabla}F_t(\by_t)\|^2\right|\by_t\right)\leq \frac{(1-e^{-\gamma})^2}{\gamma^2}\frac{d^2M^2}{\lambda^2\delta^2}$.
    \end{itemize}
\end{lemma}
\begin{proof}
\begin{itemize}
    \item[(i)] \begin{align*}
        &\E\left(\left.\widetilde{\nabla} F_t(\by_t)\right| \by_t\right) 
        \\&=  \E\left(\left.\lambda\E\left(\left.\widetilde{\nabla}F_t(\by_t)\right| \by_t,z_t,\Upsilon_t=\mathrm{explore}\right)+ (1-\lambda)\E\left(\left.\widetilde{\nabla}F_t(\by_t)\right| \by_t,z_t,\Upsilon_t=\mathrm{exploit}\right)\right| \by_t\right)
        \\&= \E\left(\left.\lambda\E\left(\left. \frac{1-e^{-\gamma}}{\gamma}\frac{d}{\lambda \delta}f_t(z_t\cdot \by_t +(1-z_t)\boldsymbol{0}_{\delta} +\delta\cdot \bv_t)\bv_t\right| \by_t,z_t\right)\right| \by_t\right)
        \\&=  \E\left(\left.\frac{1-e^{-\gamma}}{\gamma}\nabla \hat{f}^{\delta}_t(z_t\cdot\by_t+(1-z_t)\boldsymbol{0}_{\delta})\right| \by_t\right).
    \end{align*}
    The third equality is because \cref{lem:fkm estimator}.

    Since 
    \begin{align*}
        \nabla \hat{F}^{\delta}_t(\by_t)&=\E\left(\left.\frac{1-e^{-\gamma}}{\gamma}\nabla \hat{f}^{\delta}_t(z_t\cdot\by_t)\right| \by_t\right),
    \end{align*}
    we have 
    \begin{align*}
        \left\|\E\left(\left.\widetilde{\nabla} F_t(\by_t) \right| \by_t\right)-\nabla \hat{F}_t^{\delta}(\by_t)\right\|&\leq \frac{1-e^{-\gamma}}{\gamma}\E\left(\left.\left\|\nabla \hat{f}_t^{\delta}(z_t\cdot \by_t+(1-z_t)\boldsymbol{0}_{\delta})-\nabla \hat{f}_t^{\delta}(z_t\cdot \by_t)\right\|\right| \by_t\right)
        \\& \leq \frac{1-e^{-\gamma}}{\gamma} L_2\|\boldsymbol{0}_{\delta}\|
        \\&\leq \frac{1-e^{-\gamma}}{\gamma} L_2 \delta \mathrm{diam}(\mathcal{C}).
    \end{align*}
    Where the second inequality is because of the smoothness of $\hat{f}_t^{\delta}$, the third inequality is from $\|\boldsymbol{0}_{\delta}\|\leq \delta \mathrm{diam}(\mathcal{C})$ by \cref{lem:Minkowski projection}.
    \item[(ii)] \begin{align*}
        \E\left(\left.\|\widetilde{\nabla} F_t(\by_t)\|^2\right|\by_t\right)&\leq \E\left(\left.\left\|\frac{1-e^{-\gamma}}{\gamma}\frac{d}{\lambda \delta} f_t(z_t\cdot \by_t+(1-z_t)\boldsymbol{0}_{\delta}+\delta\cdot \bv_t)\bv_t \right\|^2\right|\by_t\right)
        \\&\leq \frac{(1-e^{-\gamma})^2}{\gamma^2}\frac{d^2}{\lambda^2\delta^2}f_t^2(z_t\cdot \by_t+(1-z_t)\boldsymbol{0}_{\delta}+\delta\cdot \bv_t)\|\bv_t\|^2 
        \\& \leq\frac{(1-e^{-\gamma})^2}{\gamma^2}\frac{d^2M^2}{\lambda^2\delta^2}.
    \end{align*}
\end{itemize}

\end{proof}

\begin{lemma}
     If we select Option II in \cref{algo:BBGA}, the following holds
    \begin{itemize}
        \item[(i)] $\left\|\E\left(\left.\widetilde{\nabla} F_t(\by_t) \right| \by_t\right)-\nabla \hat{F}_t^{\delta}(\by_t)\right\| \leq \frac{3}{8} L_2 \delta \mathrm{diam}(\mathcal{C})$.
        \item[(ii)] $\E\left(\left.\|\widetilde{\nabla}F_t(\by_t)\|^2\right|\by_t\right)\leq \frac{9}{64}\frac{d^2M^2}{\lambda^2\delta^2}$.
    \end{itemize}
\end{lemma}
\begin{proof}
\begin{itemize}
    \item[(i)] \begin{align*}
        &\E\left(\left.\widetilde{\nabla} F_t(\by_t)\right| \by_t\right) 
        \\&=  \E\left(\left.\lambda\E\left(\left.\widetilde{\nabla}F_t(\by_t)\right| \by_t,z_t,\Upsilon_t=\mathrm{explore}\right)+ (1-\lambda)\E\left(\left.\widetilde{\nabla}F_t(\by_t)\right| \by_t,z_t,\Upsilon_t=\mathrm{exploit}\right)\right| \by_t\right)
        \\&= \E\left(\left.\lambda\E\left(\left. \frac{3}{8}\frac{d}{\lambda \delta}f_t\left(\frac{z_t}{2} (\by_t - \underline{\bx}_{\delta})+\underline{\bx}_{\delta} +\delta \cdot \bv_t\right)\bv_t\right| \by_t,z_t\right)\right| \by_t\right)
        \\&=  \E\left(\left.\frac{3}{8}\nabla \hat{f}^{\delta}_t\left(\frac{z_t}{2} (\by_t - \underline{\bx}_{\delta})+\underline{\bx}_{\delta}\right)\right| \by_t\right)
    \end{align*}
    The third equality is because \cref{lem:fkm estimator}. Notice that
    \[\nabla \hat{F}^{\delta}_t(\by_t)= \E\left(\left.\frac{3}{8}\nabla \hat{f}^{\delta}_t\left(\frac{z_t}{2} (\by_t - \underline{\bx})+\underline{\bx}\right)\right| \by_t\right),\]
    then we have 
    \begin{align*}
        \left\|\E\left(\left.\widetilde{\nabla} F_t(\by_t) \right| \by_t\right)-\nabla \hat{F}_t^{\delta}(\by_t)\right\| &\leq \frac{3}{8}\E\left(\left.\left\|\nabla\hat{f}^{\delta}_t\left(\frac{z_t}{2} (\by_t - \underline{\bx}_{\delta})+\underline{\bx}_{\delta}\right)- \nabla \hat{f}^{\delta}_t\left(\frac{z_t}{2} (\by_t - \underline{\bx})+\underline{\bx}\right)\right\|\right| \by_t\right)
        \\&\leq \frac{3}{8}L_2\|\underline{\bx}_{\delta}-\underline{\bx}\|
        \\&\leq \frac{3}{8}L_2\delta \mathrm{diam}(\mathcal{C}).
    \end{align*}
    
    \item[(ii)] \begin{align*}
        \E\left(\left|\|\widetilde{\nabla} F_t(\by_t)\|^2\right|\by_t\right)&\leq \E\left(\left.\left\|\frac{3}{8}\frac{d}{\lambda \delta}f_t\left(\frac{z_t}{2} (\by_t - \underline{\bx}_{\delta})+\underline{\bx}_{\delta} +\delta \cdot \bv_t\right)\bv_t \right\|^2\right|\by_t\right)
        \\&\leq \frac{9}{64}\frac{d^2}{\lambda^2\delta^2}f_t^2\left(\frac{z_t}{2} (\by_t - \underline{\bx}_{\delta})+\underline{\bx}_{\delta} +\delta \cdot \bv_t\right)\|\bv_t\|^2 
        \\& \leq \frac{9}{64}\frac{d^2M^2}{\lambda^2\delta^2}.
    \end{align*}
\end{itemize}

\end{proof}

\subsection{Proof of \texorpdfstring{\cref{thm:monotone bandit}}{24}}\label{Appendix:monotone_bandit}

\begin{proof}
    Let $F_t$ be the non-oblivious function of monotone function $f_t$. Let $\bx^*\triangleq \argmax_{\bx\in \mathcal{C}}\sum_{t=1}^{T}f_t(\bx)$, $\bx^*_{\delta'} \triangleq \mathcal{P}_{\mathcal{C_{\delta',\by}}}(\bx^*)$. Consider $\|\by_{t+1}-\bx_{\delta'}^*\|^2$,
    \begin{align*}
        \|\by_{t+1}-\bx^*_{\delta'}\|^2 &\leq \|\by_t+\eta \widetilde{\nabla} F_{t}(\by_t)-\bx^*_{\delta'}\|^2 \\
        & =\|\by_t-\bx_{\delta'}^*\|^2 + 2\eta \langle \by_t-\bx_{\delta'}^*,\widetilde{\nabla} F_t(\by_t) \rangle+\eta^2 \|\widetilde{\nabla} F_{t}(\by_t)\|^2.
    \end{align*}
    Therefore,
    \begin{equation}\label{eq:85}
        \begin{aligned}
            &\E\left(\left.\|\by_{t+1}-\bx^*_{\delta'}\|^2-\|\by_t-\bx^*_{\delta'}\|^2\right| \by_t \right)
        \\&\leq \E\left(\left. 2\eta \langle \by_t-\bx^*_{\delta'},\widetilde{\nabla} F_{t}(\by_t) \rangle+\eta^2 \|\widetilde{\nabla} F_{t}(\by_t)\|^2\right| \by_t \right)
        \\&\leq \E\left(\left.2\eta \langle \by_t-\bx^*_{\delta'},\nabla \hat{F}^{\delta}_{t}(\by_t) \rangle+2\eta \langle \by_t-\bx^*_{\delta'}, \widetilde{\nabla}F_{t}(\by_t)-\nabla \hat{F}^{\delta}_t(\by_t)\rangle+\eta^2 \|\widetilde{\nabla} F_{t}(\by_t)\|^2\right| \by_t\right)\\
        &\leq 2\eta \E\left(\hat{f}^{\delta}_t(\by_t) - (1-e^{-\gamma})\hat{f}^{\delta}_t(\bx^*_{\delta'})\mid \by_t\right)+2\eta\left\langle \by_t-\bx^*_{\delta'}, \E\left(\left.\widetilde{\nabla}F_{t}(\by_t)\right| \by_t\right)-\nabla \hat{F}^{\delta}_t(\by_t)\right\rangle
        \\&\quad\quad+\eta^2\E \left(\left.\|\widetilde{\nabla} F_{t}(\by_t)\|^2\right| \by_t\right)\\
        &\leq 2\eta \E\left(\hat{f}^{\delta}_t(\by_t) - (1-e^{-\gamma})\hat{f}^{\delta}_t(\bx^*_{\delta'})\mid \by_t\right)+2\frac{1-e^{-\gamma}}{\gamma}\eta \|\by_t-\bx_{\delta'}^*\|L_2\delta \mathrm{diam}(\mathcal{C})+\eta^2\E \left(\left.\|\widetilde{\nabla} F_{t}(\by_t)\|^2\right| \by_t\right)
        \\&\leq 2\eta \E\left(\hat{f}^{\delta}_t(\by_t) - (1-e^{-\gamma})\hat{f}^{\delta}_t(\bx^*_{\delta'})\mid \by_t\right)+2\frac{1-e^{-\gamma}}{\gamma}\eta L_2\delta \mathrm{diam}^2(\mathcal{C})+\frac{(1-e^{-\gamma})^2}{\gamma^2} \frac{\eta^2 d^2}{\lambda^2\delta^2}M^2.
        \end{aligned}
    \end{equation}
    The third inequality is due to \cref{corollary: monotone stationary point}. The fourth inequality is due to \cref{lem:estimate option I}. Then,
    \begin{equation}\label{eq:86}
        \begin{aligned}
            &\E\left(\sum_{t=1}^T\left( (1-e^{-\gamma})\hat{f}^{\delta}_t(\bx_{\delta'}^*)-\hat{f}^{\delta}_t(\by_t)\right)\right)
        \\&\leq \frac{(1-e^{-\gamma})^2}{\gamma^2} \frac{\eta d^2 M^2}{2\lambda^2\delta^2} T +\frac{1-e^{-\gamma}}{\gamma}L_2\delta \mathrm{diam}^2(\mathcal{C})T + \frac{1}{2\eta}\E \left(\sum_{t=1}^T\left(\|\by_t-\bx^*_{\delta'}\|^2-\|\by_{t+1}-\bx^*_{\delta'}\|^2\right)\right)
        \\& \leq \frac{(1-e^{-\gamma})^2}{\gamma^2}\frac{\eta d^2 M^2}{2\lambda^2\delta^2}T +\frac{1-e^{-\gamma}}{\gamma}L_2\delta \mathrm{diam}^2(\mathcal{C}) T+ \frac{\|\by_1-\bx_{\delta'}^*\|^2}{2\eta}
        \\& \leq \frac{(1-e^{-\gamma})^2}{\gamma^2}\frac{\eta d^2 M^2}{2\lambda^2\delta^2}T + \frac{1-e^{-\gamma}}{\gamma} L_2\delta \mathrm{diam}^2(\mathcal{C}) T+\frac{\mathrm{diam}^2(\mathcal{C})}{2\eta}.
        \end{aligned}
    \end{equation}
    The $(1-e^{-\gamma})$-regret can be bounded as
    \begin{align*}
        &(1-e^{-\gamma})\sum_{t=1}^T f_t(\bx^*) - \E \left(\sum_{t=1}^T f_t(\bx_t)\right)
        \\&= (1-e^{-\gamma})\sum_{t=1}^T f_t(\bx^*) - \lambda\E\left( \sum_{t=1}^T\E\left(\left. f_t(\bx_t)\right|\Upsilon_t = \mathrm{explore}\right)\right)-(1-\lambda)\E\left( \sum_{t=1}^T\E\left(\left. f_t(\bx_t)\right|\Upsilon_t = \mathrm{exploit}\right)\right)
        \\& \leq (1-e^{-\gamma})\sum_{t=1}^T f_t(\bx^*) - (1-\lambda) \E \left(\sum_{t=1}^T f_t(\by_t)\right)
        \\& =  \E\left(\sum_{t=1}^T\left( (1-e^{-\gamma})\hat{f}^{\delta}_t(\bx_{\delta'}^*)-\hat{f}^{\delta}_t(\by_t)\right)\right) +(1-e^{-\gamma})\E\left(\sum_{t=1}^T\left(f_t(\bx^*_{\delta'})-\hat{f}^{\delta}_t(\bx^*_{\delta'})\right)\right)
        \\&\quad\quad + \E\left(\sum_{t=1}^T\left(\hat{f}^{\delta}_t(\by_{t})-(1-\lambda)f_t(\by_t)\right)\right) + (1-e^{-\gamma})\E\left(\sum_{t=1}^T\left(f_t(\bx^*)-f_t(\bx_{\delta'}^*)\right)\right)
        \\&\leq  \E\left(\sum_{t=1}^T\left( (1-e^{-\gamma})\hat{f}^{\delta}_t(\bx_{\delta'}^*)-\hat{f}^{\delta}_t(\by_t)\right)\right) +(1-e^{-\gamma})\E\left(\sum_{t=1}^T\left|f_t(\bx^*_{\delta'})-\hat{f}^{\delta}_t(\bx^*_{\delta'})\right|\right)
        \\&\quad\quad + \E\left(\sum_{t=1}^T\left|\hat{f}^{\delta}_t(\by_{t})-f_t(\by_t)\right|\right)+ \lambda \E\left(\sum_{t=1}^T f_t(\by_t)\right) + (1-e^{-\gamma})\E\left(\sum_{t=1}^T\left|f_t(\bx_{\delta'}^*)-f_t(\bx^*)\right|\right).
    \end{align*}
    Notice that
    $|f_t(\bx_{\delta'}^*)-\hat{f}_t^{\delta}(\bx_{\delta'}^*)|\leq L_1 \delta$ and $|\hat{f}^{\delta}_t(\by_t)-f_t(\by_t)|\leq L_1\delta$ by \cref{lem:smoothed version bias}. Also, $|f_t(\bx_{\delta'}^*)-f_t(\bx^*)|\leq L_1\|\bx_{\delta'}^*-\bx^*\|\leq L_1\frac{\delta}{R} \mathrm{diam}(\mathcal{C})$ by \cref{lem:Minkowski projection}. Therefore,
    \begin{align*}
        &(1-e^{-\gamma})\sum_{t=1}^T f_t(\bx^*) - \E \left(\sum_{t=1}^T f_t(\bx_t)\right)
        \\&\leq  \frac{(1-e^{-\gamma})^2}{\gamma^2}\frac{\eta d^2 M^2}{2\lambda^2\delta^2}T +\frac{1-e^{-\gamma}}{\gamma} L_2\delta \mathrm{diam}^2(\mathcal{C}) T + \frac{\mathrm{diam}^2(\mathcal{C})}{2\eta}
        \\&\quad \quad + (1-e^{-\gamma})L_1\delta T + L_1\delta T + \lambda M T+ (1-e^{-\gamma})L_1\frac{\delta}{R} T\mathrm{diam}(\mathcal{C}).
    \end{align*}
    
    Let $\lambda = \mathrm{diam}^{2/3}(\mathcal{C}) d^{1/3}T^{-1/5}, \delta = \mathrm{diam}^{-1/3}(\mathcal{C})d^{1/3}T^{-1/5}, \eta = \mathrm{diam}^{4/3}(\mathcal{C})d^{-1/3}T^{-4/5}$, we have,
    \begin{align*}
        (1-e^{-\gamma})\sum_{t=1}^T f_t(\bx_{\delta'}^*) - \E \left(\sum_{t=1}^T f_t(\bx_t)\right)\leq O(d^{1/3}T^{4/5}).
    \end{align*}
\end{proof}

\subsection{Proof of \texorpdfstring{\cref{thm:nonmonotone bandit}}{25}}\label{Appendix:non_monotone_bandit}
\begin{proof}
    Similar to \eqref{eq:85}, we have the following inequality,
    \begin{align*}
        &\E\left(\left.\|\by_{t+1}-\bx_{\delta'}^*\|-\|\by_t-\bx_{\delta'}^*\|^2 \right| \by_t\right)
            \\ \leq &\E\left(\left.2\eta \langle \by_t-\bx^*_{\delta'},\nabla \hat{F}^{\delta}_{t}(\by_t) \rangle+2\eta \langle \by_t-\bx^*_{\delta'}, \widetilde{\nabla}F_{t}(\by_t)-\nabla \hat{F}^{\delta}_t(\by_t)\rangle+\eta^2 \|\widetilde{\nabla} F_{t}(\by_t)\|^2\right| \by_t\right)
            \\ \leq & 2\eta \E \left( \left. \hat{f}_t^{\delta}\left(\frac{\by_t+\underline{\bx}}{2}\right) - \frac{1-\|\underline{\bx}\|_{\infty}}{4}\hat{f}_t^{\delta}(\bx_{\delta'}^*) \right| \by_t\right) +2\eta\E\left(\left.\langle \by_t-\bx^*_{\delta'}, \widetilde{\nabla}F_{t}(\by_t)-\nabla \hat{F}^{\delta}_t(\by_t)\rangle\right| \by_t\right)
            \\&\quad\quad+\eta^2\E \left(\left.\|\widetilde{\nabla} F_{t}(\by_t)\|^2\right| \by_t\right)
            \\\leq & 2\eta \E\left(\left.\hat{f}^{\delta}_t(\frac{\by_t+\underline{\bx}}{2}) -\frac{1-\|\underline{\bx}\|_{\infty}}{4}\hat{f}^{\delta}_t(\bx^*_{\delta'})\right| \by_t\right)+\frac{3}{4}L_2\eta \delta\mathrm{diam}^2(\mathcal{C})+\eta^2\E \left(\left.\|\widetilde{\nabla} F_{t}(\by_t)\|^2\right| \by_t\right)
            \\\leq & 2\eta \E\left(\left.\hat{f}^{\delta}_t(\frac{\by_t+\underline{\bx}}{2}) - \frac{1-\|\underline{\bx}\|_{\infty}}{4}\hat{f}^{\delta}_t(\bx^*_{\delta'})\right| \by_t\right)+\frac{3}{4}L_2\eta \delta\mathrm{diam}^2(\mathcal{C})+\frac{9}{64}\frac{\eta^2 d^2}{\lambda^2\delta^2}M^2.
    \end{align*}
    Then, 
    \begin{align*}
        &\E\left(\sum_{t=1}^T\left( \frac{1-\|\underline{\bx}\|_{\infty}}{4}\hat{f}^{\delta}_t(\bx_{\delta'}^*)-\hat{f}^{\delta}_t\left(\frac{\by_t+\underline{\bx}}{2}\right)\right)\right) \leq \frac{9}{64}\frac{\eta d^2 M^2}{2\lambda^2\delta^2}T + \frac{3}{8}L_2\delta\mathrm{diam}^2(\mathcal{C})T+\frac{\mathrm{diam}^2(\mathcal{C})}{2\eta}.
    \end{align*}
    The $ (1-\|\underline{\bx}\|_{\infty})/4$-regret can be bounded as
    \begin{align*}
        & \frac{1-\|\underline{\bx}\|_{\infty}}{4}\sum_{t=1}^T f_t(\bx^*) - \E \left(\sum_{t=1}^T f_t(\bx_t)\right)
        \\\leq & \E\left(\sum_{t=1}^T\left(  \frac{1-\|\underline{\bx}\|_{\infty}}{4}\hat{f}^{\delta}_t(\bx_{\delta'}^*)-\hat{f}^{\delta}_t(\by_t)\right)\right) + \frac{1-\|\underline{\bx}\|_{\infty}}{4}\E\left(\sum_{t=1}^T\left|f_t(\bx^*_{\delta'})-\hat{f}^{\delta}_t(\bx^*_{\delta'})\right|\right)
        \\&\quad\quad + \E\left(\sum_{t=1}^T\left|\hat{f}^{\delta}_t(\by_{t})-f_t(\by_t)\right|\right)+ \lambda \E\left(\sum_{t=1}^T f_t(\by_t)\right) +  \frac{1-\|\underline{\bx}\|_{\infty}}{4}\sum_{t=1}^T\E(\left|f_t(\bx_{\delta'}^*)-f_t(\bx^*)\right|)
        \\\leq &\frac{9}{64}\frac{\eta d^2 M^2}{2\lambda^2\delta^2}T +\frac{3}{8}L_2\delta\mathrm{diam}^2(\mathcal{C})T+ \frac{\mathrm{diam}^2(\mathcal{C})}{2\eta} + \frac{1-\|\underline{\bx}\|_{\infty}}{4}L\delta T 
        \\&\quad \quad + L_1\delta T + \lambda M T+ \frac{1-\|\underline{\bx}\|_{\infty}}{4}L_1\frac{\delta}{R} T\mathrm{diam}(\mathcal{C}).
    \end{align*}
    Let $\lambda = \mathrm{diam}^{2/3}(\mathcal{C}) d^{1/3}T^{-1/5}, \delta = \mathrm{diam}^{-1/3}(\mathcal{C})d^{1/3}T^{-1/5}, \eta = \mathrm{diam}^{4/3}(\mathcal{C})d^{-1/3}T^{-4/5}$, we have,
    \begin{align*}
        \frac{1-\|\underline{\bx}\|_{\infty}}{4}\sum_{t=1}^T f_t(\bx_{\delta'}^*) - \E \left(\sum_{t=1}^T f_t(\bx_t)\right)\leq O(d^{1/3}T^{4/5}).
    \end{align*} 
\end{proof}
\section{Proofs in Section~\ref{sec:minimax}}\label{Appendix:minimax}

The next Lemma is immediately derived according to \cref{corollary: monotone stationary point}, \cref{corollary: nonmonotone stationary point} and the property of convex functions.
\begin{lemma}\label{lem:nonoblivious first order relation}
    Let $\hat{f}$ be a multi-linear extension of a convex-submodular function $f$. For any $\bx,\bx_1,\bx_2\in \mathcal{K}$, $\by,\by_1,\by_2\in\mathcal{C}$, the following holds 
    \begin{align}
        \langle \bx_2-\bx_1, \nabla_{\bx} \hat{f}(\bx_1,\by)\rangle &\leq \hat{f}(\bx_2,\by) - \hat{f}(\bx_1,\by).
    \end{align}
    If $f(\bx,S)$ is monotone w.r.t. $S$, then 
    \begin{align}
        \left\langle \by_2-\by_1, \int_0^1 e^{z-1} \nabla_{\by} \hat{f}(\bx,z\cdot \by_1) \mathrm{d}z\right\rangle &\geq \left(1-\frac{1}{e}\right)\hat{f}(\bx,\by_2) - \hat{f}(\bx,\by_1).
    \end{align}
    If $f(\bx,S)$ is not assumed to be monotone, then
    \begin{align}
        \left\langle \by_2-\by_1, \int_0^1 \frac{1}{8(1-\frac{z}{2})^3} \nabla_{\by} \hat{f}\left(\bx,\frac{z}{2}\cdot\by_1+(1-\frac{z}{2})\cdot\underline{\by}\right) \mathrm{d}z\right\rangle &\geq \frac{1-\|\underline{\by}\|_{\infty}}{4}\hat{f}(\bx,\by_2) -  \hat{f}\left(\bx,\frac{\by_1+\underline{\by}}{2}\right),
    \end{align}where $\underline{\boldsymbol{y}}:=\argmin_{\boldsymbol{y}\in \mathcal{C}} \|\boldsymbol{y}\|_{\infty}$.
\end{lemma}

\subsection{Proof of \texorpdfstring{\cref{thm:minimax}}{26}}\label{Appendix:mono_minimax}
\begin{proof}
At first, we prove that, for any $\bx\in\mathcal{K}$,
\begin{equation*}
    \begin{aligned}
        \|\bx_{t+1}-\bx\|^2&\le\|\bx_{t}-\eta\widetilde{\nabla}_{\bx} \hat{f}(\bx_t,\by_t)-\bx\|^2\\
        &=\|\bx_{t}-\bx\|^2-2\eta \langle \bx_{t} - \bx, \widetilde{\nabla}_{\bx} \hat{f}(\bx_t,\by_t)\rangle+\|\eta\widetilde{\nabla}_{\bx} \hat{f}(\bx_t,\by_t)\|^{2}.
    \end{aligned}
\end{equation*}
As a result, we have that
\begin{equation}\label{equ:m3}
    2\eta \langle \bx_{t} - \bx, \widetilde{\nabla}_{\bx} \hat{f}(\bx_t,\by_t)\rangle\le \|\bx_{t}-\bx\|^2- \|\bx_{t+1}-\bx\|^2+\|\eta\widetilde{\nabla}_{\bx} \hat{f}(\bx_t,\by_t)\|^{2}.
\end{equation}
Similarly, we also can show that, for any $\by\in\mathcal{C}$,
\begin{equation}\label{equ:m4}
    -2\eta\left\langle \by_{t} - \by, (1-e^{-1})\widetilde{\nabla}_{\by} \hat{f}(\bx_t,z_t\cdot \by_t)\right\rangle\le\|\by_t-\by\|^2 - \|\by_{t+1}-\by\|^2+\|\eta(1-e^{-1})\widetilde{\nabla}_{\by} \hat{f}(\bx_t,z_t\cdot \by_t)\|^{2}.
\end{equation}

    According to \cref{lem:nonoblivious first order relation}, we have
    \begin{equation}\label{equ:m1}
        \begin{aligned}
            2\eta\left(\hat{f}(\bx,\by_t)-\hat{f}(\bx_{t},\by_t)\right)&\geq 2\eta\langle \bx-\bx_{t}, \widetilde{\nabla}_{\bx} \hat{f}(\bx_t,\by_t) \rangle 
            \\& \geq \|\bx_{t+1}-\bx\|^2-\|\bx_t-\bx\|^2-\|\eta\widetilde{\nabla}_{\bx} \hat{f}(\bx_t,\by_t)\|^{2},
        \end{aligned}
    \end{equation}
    and
    \begin{equation}\label{equ:m2}
        \begin{aligned}
            &2\eta\left((1-e^{-1})\hat{f}(\bx_t,\by)-\hat{f}(\bx_t,\by_t)\right)
            \\&\leq  2\eta \left\langle \by-\by_t, \int_0^1 e^{z-1}\nabla_{\by} \hat{f}(\bx_t,z\cdot \by_t) \mathrm{d}z\right\rangle
            \\& = 2\eta\E \left(\left\langle \by-\by_t, (1-e^{-1})\widetilde{\nabla}_{\by} \hat{f}(\bx_t,z_t\cdot \by_t) \right\rangle \bigg| \bx_t,\by_t\right)
            \\&\leq \E\left(\left.\|\by_t-\by\|^2-\|\by_{t+1}-\by\|^2+\|\eta(1-e^{-1})\widetilde{\nabla}_{\by} \hat{f}(\bx_t,z_t\cdot \by_t)\|^{2} \right| \bx_t, \by_t \right).
        \end{aligned}
    \end{equation}
    Combining (\ref{equ:m1}) and (\ref{equ:m2}), we have 
    \begin{equation}
        \begin{aligned}
            &2\eta \E \left((1-e^{-1})\hat{f}(\bx_t,\by)-\hat{f}(\bx,\by_t)\right)\\
            &=2\eta \E \left((1-e^{-1})\hat{f}(\bx_t,\by)-\hat{f}(\bx_t,\by_t)+\hat{f}(\bx_t,\by_t)-\hat{f}(\bx,\by_t)\right)
        \\&\leq \E\left(\|\by_t-\by\|^2-\|\by_{t+1}-\by\|^2+\|\bx_t-\bx\|^2-\|\bx_{t+1}-\bx\|^2\right)\\
        &+\E\left(\|\eta(1-e^{-1})\widetilde{\nabla}_{\by} \hat{f}(\bx_t,z_t\cdot \by_t)\|^{2}+\|\eta\widetilde{\nabla}_{\bx} \hat{f}(\bx_t,\by_t)\|^{2}\right).
        \end{aligned}
    \end{equation}

     Sum over $t$ and divide by $2\eta T$,
     \begin{equation}
         \begin{aligned}
             &\E\left(\sum_{t=1}^{T} \frac{1}{T}\left((1-e^{-1})\hat{f}(\bx_t,\by)-\hat{f}(\bx,\by_t)\right)\right)
             \\&\leq\frac{\|\by_1-\by\|^2+\|\bx_1-\bx\|^2}{2\eta T} + \frac{1}{2\eta T}\sum_{t=1}^T\E\left(\|\eta \widetilde{\nabla}_{\bx} \hat{f}(\bx_{t},\by_t) \|^2 + \|\eta (1-e^{-1})\widetilde{\nabla}_{\by} \hat{f}(\bx_t,z_t \cdot\by_t)\|^2\right)
             \\&\leq \frac{\mathrm{diam}^2(\mathcal{C})+\mathrm{diam}^2(\mathcal{K})}{2\eta T}+\frac{(2-e^{-1})}{2}\eta \widetilde{G}^2.
         \end{aligned}
     \end{equation}
    Let $\eta = \frac{\sqrt{\mathrm{diam}^2(\mathcal{C})+\mathrm{diam}^2(\mathcal{K})}}{\widetilde{G}\sqrt{T}}$, $\bx_{sol}=\sum_{t=1}^T\frac{1}{T}\bx_t$, $\by^* = \argmax_{\by\in\mathcal{C}} \E\left(\hat{f}(\bx_{sol},\by)\right)$, $\bx^*= \arg \min_{\bx\in \mathcal{K}}\max_{\by\in\mathcal{C}} \hat{f}(\bx,\by)$, then
    \begin{equation}\label{eq:77}
        \begin{aligned}
            \E\left(\sum_{t=1}^{T} \frac{1}{T}\left((1-e^{-1})\hat{f}(\bx_t,\by)-\hat{f}(\bx,\by_t)\right)\right)\leq \frac{(3-e^{-1})\widetilde{G}\sqrt{\mathrm{diam}^2(\mathcal{C})+\mathrm{diam}^2(\mathcal{K})}}{2\sqrt{T}}.
        \end{aligned}
    \end{equation}
    Therefore,
    \begin{equation}
        \begin{aligned}
            (1-e^{-1})\E\left(f\left(\bx_{sol},\by^*\right)\right)-OPT
            &=(1-e^{-1})\E\left(f\left(\bx_{sol},\by^*\right)\right)-\max_{\by\in\mathcal{C}}\hat{f}(\bx^*,\by)
            \\&\leq \E\left(\sum_{t=1}^T \frac{1}{T}\left((1-e^{-1})\hat{f}(\bx_t,\by^*)-\hat{f}(\bx^*,\by_t)\right)\right)
            \\&\leq\frac{(3-e^{-1})\widetilde{G}\sqrt{\mathrm{diam}^2(\mathcal{C})+\mathrm{diam}^2(\mathcal{K})}}{2\sqrt{T}}.
        \end{aligned}
    \end{equation}
    The first inequality comes from the convexity of $\hat{f}(\bx,\by)$ w.r.t. $\bx$ and $\max_{\by\in\mathcal{C}}\hat{f}(\bx^*,\by)\geq \hat{f}(\bx^*,\by)\geq \hat{f}(\bx^*,\by_t), \forall t\in [T]$. The second inequality is achieved by setting $\by = \by^*, \bx=\bx^*$ in \eqref{eq:77}.
    Let $T = \frac{(3-e^{-1})^2 \widetilde{G}^2\left(\mathrm{diam}^2(\mathcal{C})+\mathrm{diam}^2(\mathcal{K})\right)}{4\epsilon^2}$, $\bx_{sol}$ is a $(1-e^{-1},\epsilon)$-approximation solution.
\end{proof}

\subsection{Proof of \texorpdfstring{\cref{thm:minimax nonmonotone}}{27}}\label{Appendix:non_mono_minimax}
\begin{proof}
    Similar to (\ref{equ:m3}) and (\ref{equ:m4}), for any $\bx\in \mathcal{K}$ and $\by\in\mathcal{C}$, we have
    \begin{equation}
        \begin{aligned}
            2\eta\left\langle\bx_t-\bx, \widetilde{\nabla}_{\bx}\hat{f}\left(\bx_t,\frac{\by_t+\underline{\by}}{2}\right)\right\rangle \leq \|\bx_t-\bx\|^2-\|\bx_{t+1}-\bx\|^2 +\|\eta\widetilde{\nabla}_{\bx}\hat{f}\left(\bx_t,\frac{\by_t+\underline{\by}}{2}\right)\|^2
        \end{aligned}
    \end{equation}
    and 
    \begin{equation}
    \begin{aligned}
         -2\eta\left\langle\by_t-\by, \frac{3}{8}\widetilde{\nabla}_{\by}\hat{f}\left(\bx_t,\frac{z_t}{2}\cdot\by_t +\left(1-\frac{z_t}{2}\right)\underline{\by}\right)\right\rangle&\leq \|\by_t-\by\|^2-\|\by_{t+1}-\by\|^2\\&+\|\frac{3\eta}{8}\widetilde{\nabla}_{\by}\hat{f}\left(\bx_t,\frac{z_t}{2}\cdot\by_t +\left(1-\frac{z_t}{2}\right)\underline{\by}\right)\|^2.
    \end{aligned}
    \end{equation}

    According to \cref{lem:nonoblivious first order relation}, we have
    \begin{equation}\label{equ:m5}
        \begin{aligned}
            2\eta\left(\hat{f}\left(\bx,\frac{\by_t+\underline{\by}}{2}\right)-\hat{f}\left(\bx_t,\frac{\by_t+\underline{\by}}{2}\right)\right)&\geq 2\eta\left\langle \bx-\bx_{t}, \widetilde{\nabla}_{\bx} \hat{f}\left(\bx_t,\frac{\by_t+\underline{\by}}{2}\right) \right\rangle 
            \\& \geq \|\bx_{t+1}-\bx\|^2-\|\bx_t-\bx\|^2-\|\eta\widetilde{\nabla}_{\bx}\hat{f}\left(\bx_t,\frac{\by_t+\underline{\by}}{2}\right)\|^2,
        \end{aligned}
    \end{equation}
    and
    \begin{equation}\label{equ:m6}
        \begin{aligned}
            &2\eta\left(\frac{1-\|\underline{\by}\|_{\infty}}{4}\hat{f}(\bx_t,\by)-\hat{f}\left(\bx_t,\frac{\by_t+\underline{\by}}{2}\right)\right)
            \\&\leq  2\eta \left\langle \by-\by_t, \int_0^1 \frac{1}{8(1-\frac{z}{2})^3}\nabla_{\by} \hat{f}\left(\bx_t,\frac{z}{2}\cdot\by_t + (1-\frac{z}{2})\underline{\by_t}\right) \mathrm{d}z\right\rangle
            \\& = 2\eta\E \left(\left\langle \by-\by_t,\frac{3}{8} \widetilde{\nabla}_{\by} \hat{f}\left(\bx_t,\frac{z_t}{2}\cdot\by_t + (1-\frac{z_t}{2})\cdot\underline{\by_t}\right) \right\rangle \bigg| \bx_t,\by_t\right)
            \\&\leq \E\left(\left.\|\by_t-\by\|^2-\|\by_{t+1}-\by\|^2 +\|\frac{3\eta}{8}\widetilde{\nabla}_{\by}\hat{f}\left(\bx_t,\frac{z_t}{2}\cdot\by_t +\left(1-\frac{z_t}{2}\right)\underline{\by}\right)\|^2\right| \bx_t, \by_t \right).
        \end{aligned}
    \end{equation}
    Combining (\ref{equ:m5}) and (\ref{equ:m6}), we have 
    \begin{equation}
        \begin{aligned}
            &2\eta \E \left(\frac{1-\|\underline{\by}\|_{\infty}}{4}\hat{f}(\bx_t,\by)-\hat{f}\left(\bx_t,\frac{\by_t+\underline{\by}}{2}\right)\right)
        \\&\leq \E\left(\|\by_t-\by\|^2-\|\by_{t+1}-\by\|^2+\|\bx_t-\bx\|^2-\|\bx_{t+1}-\bx\|^2\right)\\
        &+\E\left(\|\frac{3\eta}{8}\widetilde{\nabla}_{\by}\hat{f}\left(\bx_t,\frac{z_t}{2}\cdot\by_t +\left(1-\frac{z_t}{2}\right)\underline{\by}\right)\|^2+\|\eta\widetilde{\nabla}_{\bx}\hat{f}\left(\bx_t,\frac{\by_t+\underline{\by}}{2}\right)\|^2\right).
        \end{aligned}
    \end{equation}

     Sum over $t$ and divide by $\eta T$,
     \begin{equation}
         \begin{aligned}
             &\E\left(\sum_{t=1}^{T} \frac{1}{T}\left(\frac{1-\|\underline{\by}\|_{\infty}}{4}\hat{f}(\bx_t,\by)-\hat{f}\left(\bx_t,\frac{\by_t+\underline{\by}}{2}\right)\right)\right)
             \\&\leq \frac{\mathrm{diam}^2(\mathcal{K})+\mathrm{diam}^2(\mathcal{C})}{2\eta T} + \frac{11}{16}\eta \widetilde{G}^2.
         \end{aligned}
     \end{equation}
    Let $\eta=\frac{\sqrt{\mathrm{diam}^2(\mathcal{K})+\mathrm{diam}^2(\mathcal{C})}}{\widetilde{G}\sqrt{T}}$, $\bx_{sol}=\sum_{t=1}^T\frac{1}{T}\bx_t$, $\by^* = \argmax_{\by\in\mathcal{C}} \E\left(\hat{f}(\bx_{sol},\by)\right)$, $\bx^*= \arg \min_{\bx\in \mathcal{K}}\max_{\by\in\mathcal{C}} \hat{f}(\bx,\by)$,
    \begin{equation}
        \begin{aligned}
            \frac{1-\|\underline{\by}\|_{\infty}}{4}\E\left(f\left(\bx_{sol},\by^*\right)\right)-OPT
            &=\frac{1-\|\underline{\by}\|_{\infty}}{4}\E\left(f\left(\bx_{sol},\by^*\right)\right)-\max_{\by\in\mathcal{C}}\hat{f}(\bx^*,\by)
            \\&\leq \E\left(\sum_{t=1}^T \frac{1}{T}\left(\frac{1-\|\underline{\by}\|_{\infty}}{4}\hat{f}(\bx_t,\by^*)-\hat{f}\left(\bx^*,\frac{\by_t+\underline{\by}}{2}\right)\right)\right)
            \\&\leq \frac{19G\sqrt{\mathrm{diam}^2(\mathcal{K})+\mathrm{diam}^2(\mathcal{C})}}{16\sqrt{T}}
        \end{aligned}
    \end{equation}
    The first inequality comes from the convexity of $\hat{f}(\bx,\by)$ w.r.t. $\bx$ and $\max_{\by\in\mathcal{C}}\hat{f}(\bx^*,\by)\geq \hat{f}(\bx^*,\by)\geq \hat{f}(\bx^*,\frac{\by_t+\underline{\by}}{2}), \forall t\in [T]$. Let $T = \frac{361 \widetilde{G}^2\left(\mathrm{diam}^2(\mathcal{C})+\mathrm{diam}^2(\mathcal{C})\right)}{256\epsilon^2}$, since $\|\underline{\by}\|_{\infty}=0$ when $\mathcal{C}$ is a matroid convex hull, $\bx_{sol}$ is a $(\frac{1}{4},\epsilon)$-approximation solution.
\end{proof}
\section{Experiments about Submodular Quadratic Programming}\label{Appendix:QP}
\begin{table}[t]\label{tab:QP_mono}
	\renewcommand\arraystretch{1.35}
	\centering
	\caption{\small \cref{tab:tab:QP_online_mono} shows the final $(1-1/e)$-Regret ratio and running time of online monotone DR-submodular quadratic programming. Note that `\textbf{Feedback Type}' means the form of objectives revealed by the environment during the process of online learning, `Full Feedback', `Delayed Feedback' and `Bandit Feedback' means that the object function is returned in full, delayed and bandit setting respectively. `\textbf{$(1-1/e)$-Regret Ratio}' means the ratio between $(1-1/e)$-Regret and timestamp at the $150$-th iteration, where we use a $500$-round continuous greedy method, namely, Algorithm 1 in \citep{bian2017guaranteed} as baseline to compute the $(1-1/e)$-regret.}
	\vspace{0.5em}
	\resizebox{0.8\textwidth}{!}{
		\setlength{\tabcolsep}{1.0mm}{
			\begin{tabular}{c|c|c|c}
				\toprule[1.5pt]
				\textbf{Feedback Type} &\textbf{Algorithm}& \textbf{$(1-1/e)$-Regret Ratio}&\textbf{Running time(seconds)}\\
				\hline
				\multirow{7}{*}{Full Feedback}& OGA(5)&0.620&0.191s\\
    \cline{2-4}
				~&\textbf{OBGA(5)}&\textbf{0.609}&\textbf{0.194s}\\
    \cline{2-4}
   ~&3/2-Meta-FW&0.708&422.81s\\
    \cline{2-4}
    ~&3/2-Meta-FW-VR&0.602&421.80s\\
    \cline{2-4}
 ~&Mono-FW&8.690&0.193s\\ 
    \cline{2-4}
 ~&3/4-Meta-FW-VR&0.636&7.80s\\   
 \cline{2-4}
  ~&1/2-Meta-FW-VR& 0.664&2.22s\\
    \hline
     \hline
    \multirow{6}{*}{Delayed Feedback}&OGA(5)&1.016&0.214s\\
    \cline{2-4}
				~&\textbf{OBGA(5)}&\textbf{1.013}&\textbf{0.214s}\\
    \cline{2-4}
 ~&3/2-Meta-FW&1.076&455.72s\\
\cline{2-4}
~&3/2-Meta-FW-VR&1.002&456.81s\\
\cline{2-4}
~&3/4-Meta-FW-VR&1.024&8.64s\\
\cline{2-4}
~&1/2-Meta-FW-VR&1.046&2.46s\\
    \hline
     \hline
\multirow{2}{*}{Bandit Feedback}&\textbf{Bandit-BGA(5)}&\textbf{23.265}&\textbf{0.037s}\\
\cline{2-4}
~&Bandit-FW&61.339&0.164s\\
\midrule[1.5pt]
\end{tabular}}}
	\vspace{-1.5em}
	\label{tab:tab:QP_online_mono}
\end{table}
\begin{table}[t]
	\renewcommand\arraystretch{1.35}
	\centering
 \caption{\small\cref{tab:QP_online_non_mono} shows the  final regret ratio and running time of online non-monotone DR-submodular quadratic programming. Note that `\textbf{Feedback Type}' means the form of objectives revealed by the environment during the process of online learning, `Full Feedback', `Delayed Feedback' and `Bandit Feedback' means that the object function is returned in full, delayed and bandit setting respectively. `\textbf{Regret Ratio}' means the ratio between regret and time horizon at the $150$-th iteration, where we use a $500$-round deterministic Measured Frank Wolfe, namely, Algorithm 2 in \citep{mitra2021submodular+} as baseline to compute the regret.}
	\vspace{0.5em}
	\resizebox{0.77\textwidth}{!}{
		\setlength{\tabcolsep}{1.0mm}{
			\begin{tabular}{c|c|c|c}
				\toprule[1.5pt]
				\textbf{Feedback Type} &\textbf{Algorithm}& \textbf{Regret Ratio}&\textbf{Running time(seconds)}\\
				\hline
				\multirow{8}{*}{Full Feedback}& OGA(5)&0.082&0.218s\\
    \cline{2-4}
				~&\textbf{OBGA(5)}&\textbf{0.022}&\textbf{0.226s}\\
    \cline{2-4} 
    ~&Non-mono-Meta-FW&0.0340&10.010s\\
    \cline{2-4}
   ~&3/2-Measured-MFW&0.215& 374.882s\\
    \cline{2-4}
    ~&3/2-Measured-MFW-VR&0.079&384.325s\\
    \cline{2-4}
    ~&3/4-Measured-MFW-VR&0.147&8.744s\\
    \cline{2-4}
 ~&Mono-MFW&0.231&0.190s\\ 
    \cline{2-4}
 ~&1/2-Meta-FW-VR&0.186&2.479s\\
    \hline
     \hline
     
    \multirow{7}{*}{Delayed Feedback}&OGA(5)&0.093&0.230s\\
    \cline{2-4}
				~&\textbf{OBGA(5)}&\textbf{0.024}&\textbf{0.245s}\\
    \cline{2-4}
 ~&3/2-Measured-MFW&0.2145&398.24s\\
\cline{2-4}
 ~&Non-mono-Meta-FW&0.0337&10.83s\\
\cline{2-4}
~&3/2-Measured-MFW-VR&0.0763&409.405s\\
\cline{2-4}
~&3/2-Measured-MFW-VR&0.141&9.243s\\
\cline{2-4}
~&1/2-Measured-MFW-VR&0.181&2.642s\\
    \hline
     \hline
\multirow{2}{*}{Bandit Feedback}&\textbf{Bandit-BGA(5)}&\textbf{0.042}&\textbf{0.040s}\\
\cline{2-4}
~&Bandit-MFW&0.231&0.176s\\
\midrule[1.5pt]
\end{tabular}}}
	\vspace{-1.5em}
	\label{tab:QP_online_non_mono}
\end{table}
\subsection{Non-Convex/Non-Concave Quadratic Programming}\label{sec:qp}
\ \ \ \textbf{Monotone Settings:} We consider the quadratic objective $f(\boldsymbol{x}) = \frac{1}{2}\boldsymbol{x}^{T}\boldsymbol{H}\boldsymbol{x} + \boldsymbol{h}^{T}\boldsymbol{x}$ and constraints $P=\{\boldsymbol{x}\in \mathbb{R}^{n}_{+} | \boldsymbol{A}\boldsymbol{x}\le\boldsymbol{b}, \boldsymbol{0}\le\boldsymbol{x}\le\boldsymbol{u}, \boldsymbol{A}\in \mathbb{R}^{m\times n}_{+}, \boldsymbol{b}\in \mathbb{R}_{+}^{m}\}$. Following \cite{bian2017guaranteed}, we choose the matrix $\boldsymbol{H}\in \mathbb{R}^{n\times n}$ to be a randomly generated symmetric matrix with entries uniformly distributed in $[-1,0]$, and the matrix $\boldsymbol{A}$ to be a random matrix with entries uniformly distributed in $[0,1]$. 
It can be verified that $f$ is a continuous DR-submodular function. We also set $\boldsymbol{b}=\boldsymbol{u}=\boldsymbol{1}$, $n=50$ and $m=\lfloor0.2n\rfloor$. To ensure the monotonicity, we set $\boldsymbol{h}=-\boldsymbol{H}^{T}\boldsymbol{u}$. Thus, the objective becomes $f(x)=(\frac{1}{2}\boldsymbol{x}-\boldsymbol{u})^{T}\boldsymbol{H}\boldsymbol{x}\ge0$. We consider the Gaussian noise for gradient, i.e., $[\widetilde{\nabla}f(\boldsymbol{x})]_{i}=[\nabla f(\boldsymbol{x})]_{i}+\delta\mathcal{N}(0,1)$ for any $i\in[n]$ where $\delta=5$. Furthermore, we start all algorithms from the origin.
As shown in Figure~\ref{graph21}, BGA(5) converges faster than GA(5) and achieves nearly the same objective values as GA after $70$ iterations. Similar to the previous experiment, BGA(5) and GA(5) exceed Frank-Wolfe-type algorithms with respect to the convergence rate.

\textbf{Non-Monotone Settings:} We consider the quadratic objective $g(\boldsymbol{x}) = \frac{1}{2}\boldsymbol{x}^{T}\boldsymbol{H}\boldsymbol{x} + \boldsymbol{h}^{T}\boldsymbol{x}+c$ and constraints $P=\{\boldsymbol{x}\in \mathbb{R}^{n}_{+} | \boldsymbol{A}\boldsymbol{x}\le\boldsymbol{b}, \boldsymbol{0}\le\boldsymbol{x}\le\boldsymbol{u}, \boldsymbol{A}\in \mathbb{R}^{m\times n}_{+}, \boldsymbol{b}\in \mathbb{R}_{+}^{m}\}$. Similarly, we choose the matrix $\boldsymbol{H}\in \mathbb{R}^{n\times n}$ to be a randomly generated symmetric matrix with entries uniformly distributed in $[-1,0]$, $\boldsymbol{h}$ to be a random vector with entries uniformly distributed in $[0,1]$ and the matrix $\boldsymbol{A}$ to be a random matrix with entries uniformly distributed in $[0,1]$. As a result, $g$ is a continuous DR-submodular function. To ensure the $g\ge 0$, we set $c=n$. We also set $\boldsymbol{b}=\boldsymbol{u}=\boldsymbol{1}$, $n=50$ and $m=\lfloor0.2n\rfloor$. We consider the Gaussian noise for gradient, i.e., $[\widetilde{\nabla}g(\boldsymbol{x})]_{i}=[\nabla g(\boldsymbol{x})]_{i}+\delta\mathcal{N}(0,1)$ for any $i\in[n]$ where $\delta=1$.

According to the results in Figure~\ref{graph22}, BGA(5) achieves better function value than Measured FW and Non-mono FW. Measured FW-VR surpasses BGA(5) after $100$ iterations, which may be caused by the down-closed property of $P$. Despite the bad approximation guarantee of gradient ascent method in Lemma~\ref{qx_add_lemma:1}, GA(5) achieves the best result over all other four algorithms in the setting of Figure~\ref{graph22}.
\begin{figure*}[t]
\vspace{-1.0em}
\centering
\subfigure[Monotone QP \label{graph21}]{\includegraphics[width=0.45\linewidth]{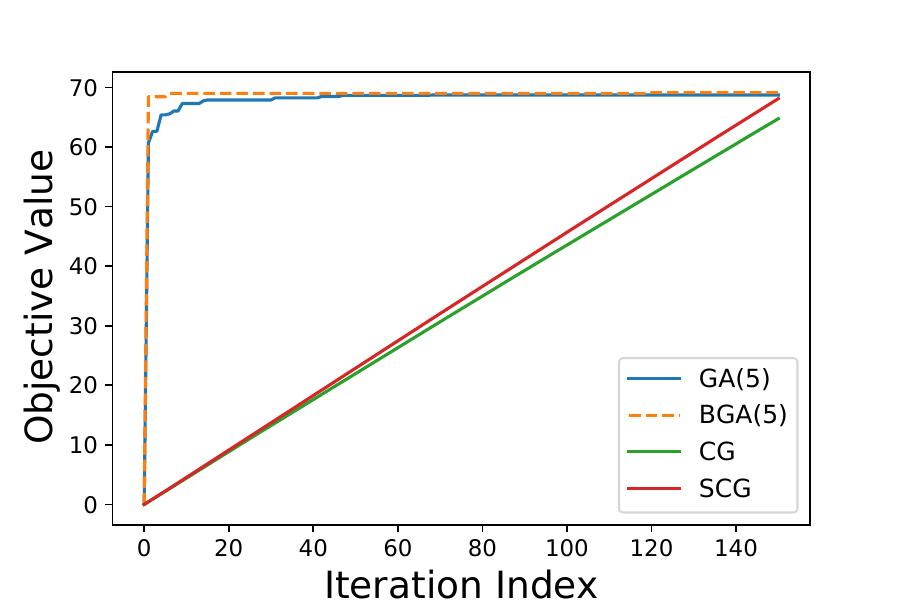}}
\subfigure[Non-Monotone QP\label{graph22}]{\includegraphics[width=0.45\linewidth]{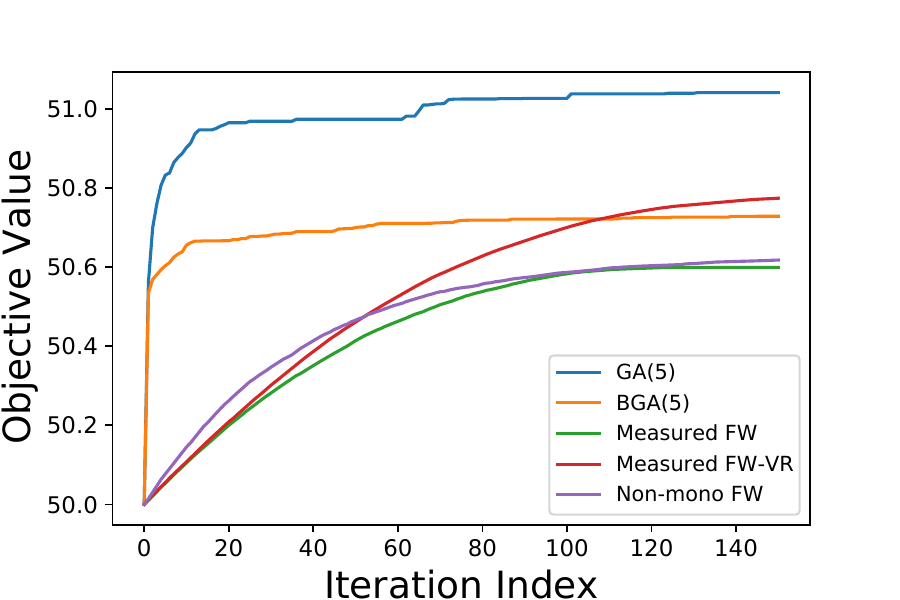}}
\caption{In ~\cref{graph21}, we test the performance of \textbf{GA(5)},\textbf{BGA(5)},\textbf{CG},and \textbf{SCG} in simulated continuous monotone
DR-submodular quadratic programming. \cref{graph22} reports the results of \textbf{GA(5)},\textbf{BGA(5)},\textbf{Measured FW},\textbf{Measured FW-VR} and \textbf{Non-mono FW} in simulated non-monotone
DR-submodular quadratic programming.} 
\vspace{-1.0em}
\end{figure*}
\subsection{Online Non-Convex/Non-Concave Quadratic Programming}
\begin{figure*}[t]
\vspace{-1.0em}
\centering
\subfigure[Monotone Case \label{graph_appendix_QP_1}]{\includegraphics[width=0.3\linewidth]{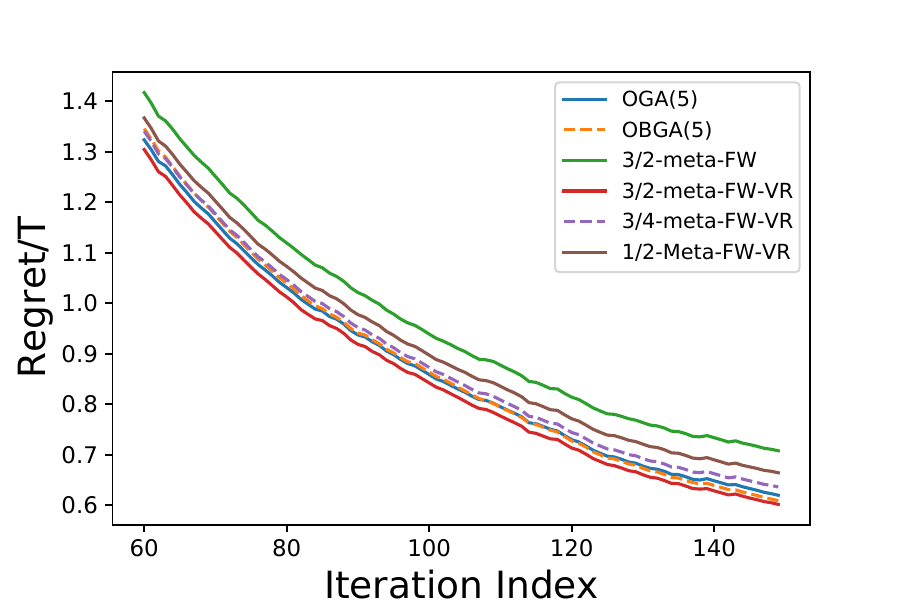}}
\subfigure[Delayed Monotone Case\label{graph_appendix_QP_2}]{\includegraphics[width=0.3\linewidth]{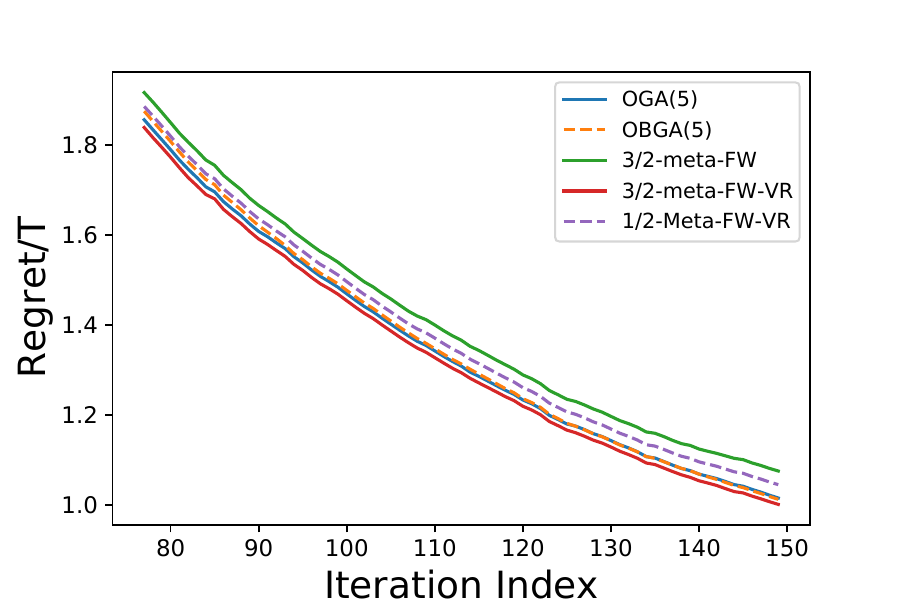}}
\subfigure[Bandit Monotone Case\label{graph_appendix_QP_3}]{\includegraphics[width=0.3\linewidth]{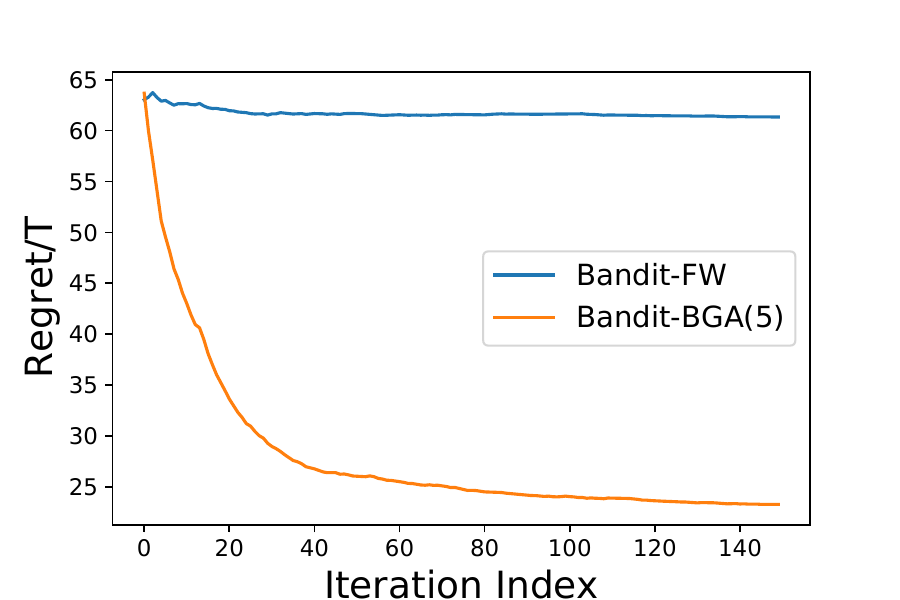}}

\subfigure[Non-Monotone Case\label{graph_appendix_QP_4}]{\includegraphics[width=0.3\linewidth]{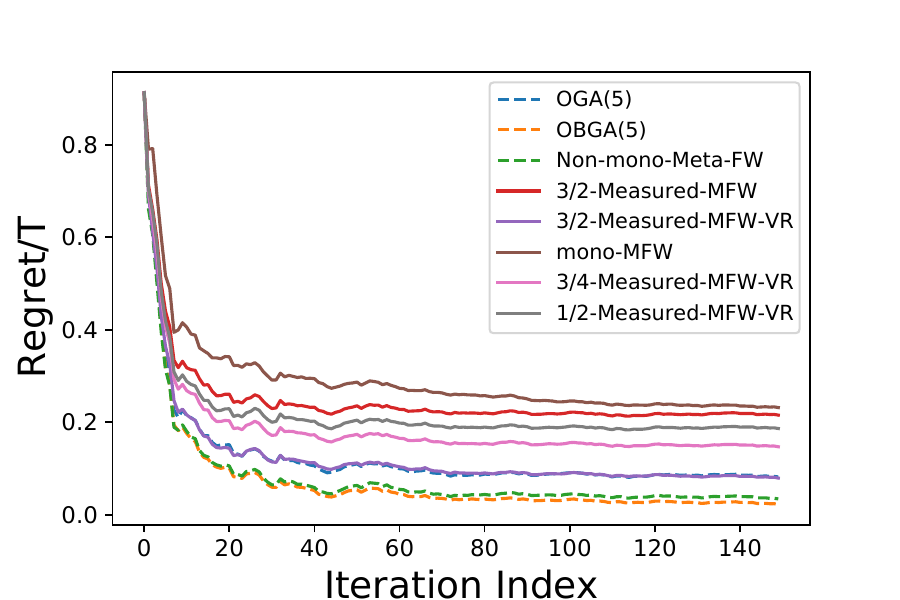}}
\subfigure[Delayed Non-Monotone Case~\label{graph_appendix_QP_5}]{\includegraphics[width=0.3\linewidth]{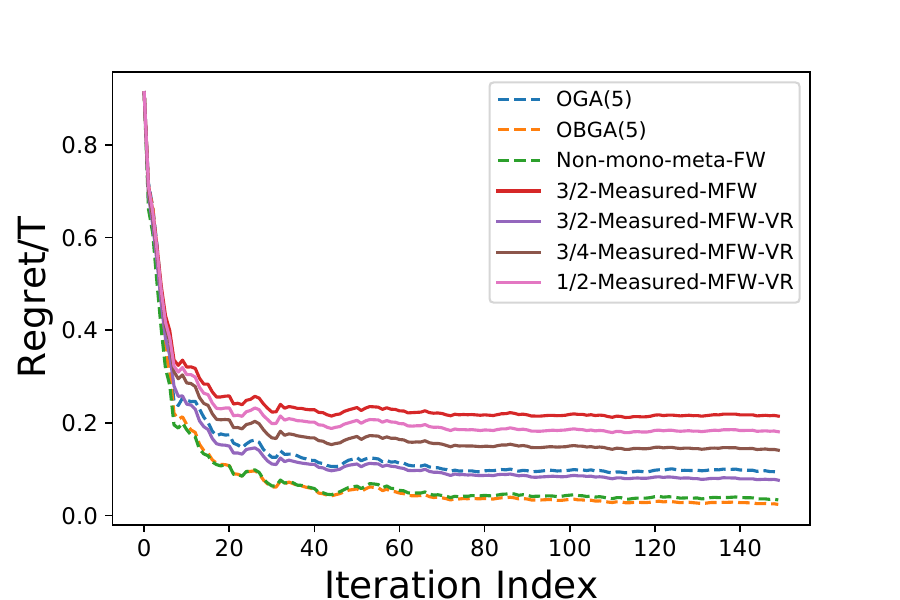}}
\subfigure[Bandit Non-Monotone Case\label{graph_appendix_QP_6}]{\includegraphics[width=0.3\linewidth]{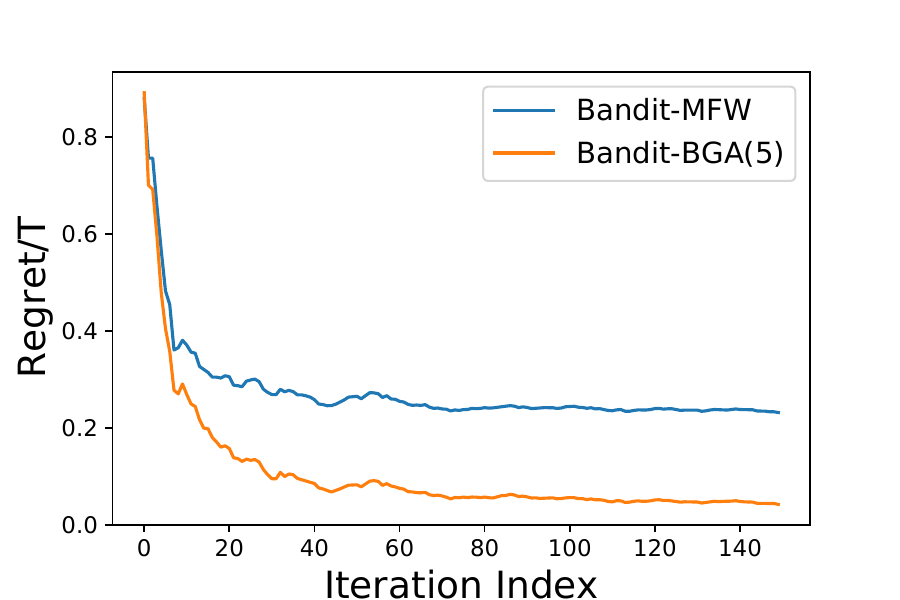}}
\caption{In ~\cref{graph_appendix_QP_1}-\ref{graph_appendix_QP_3}, we report the results for the online monotone DR-submodular quadratic programming under full information, delayed feedback and bandit feedback. Similarly, ~\cref{graph_appendix_QP_4}-\ref{graph_appendix_QP_6} show the results of three different scenarios about general online DR-submodular quadratic programming tasks.}
\vspace{-1.0em}
\end{figure*}
The same as the monotone cases in \cref{sec:qp}, we first generate $T=150$ quadratic objective functions $f_{1}, f_2, \dots, f_{T}$. The symmetric random matrix $H_{t}$ of each $f_t$ is uniformly generated from $[-1,0]^{n\times n}$ for $t=1,\ldots, T$, and the matrix $\boldsymbol{A}$ in constraint is randomly generated from the uniform distribution in $[0,1]^{m\times n}$. As for non-monotone cases, we also generate $T=150$ general quadratic objective functions $g_{1},\dots,g_{T}$ as the non-monotone part of \cref{sec:qp}.  We consider adding the standard Gaussian noise for the gradient of each $f_{t}$ or $g_{t}$. To simulate the feedback delays, we generate a uniform random number $d_{t}$ from $\{1,2,3,4,5\}$ for the $t$-th round stochastic gradient information.  We present the trend of the ratio between regret and time horizon in the \cref{graph_appendix_QP_1}-\ref{graph_appendix_QP_6}, and report the running time and the ratio at $150$-th iteration in \cref{tab:tab:QP_online_mono}-\ref{tab:QP_online_non_mono}, where we leverage the results of deterministic Frank Wolfe algorithms with $500$ iterations as a baseline to compute the regret at each time horizon.

As shown in \cref{tab:tab:QP_online_mono}-\ref{tab:QP_online_non_mono}, our OBGA(5) achieves the minimum $(1-1/e)$-regret except monotone cases with both full and delayed feedback. Moreover, OBGA(5) performs better than OGA(5) at the final stage for all six experiments. According to \cref{tab:tab:QP_online_mono}, our OBGA(5) can be 2000 times faster than the best Frank-Wolfe-tyle algorithm `3/2-Meta-FW-VR' in monotone settings. Similarly, our OBGA(5) is more effective than the best non-monotone Frank-Wolfe-tyle algorithm `3/2-Measured-MFW-VR' according to \cref{tab:QP_online_non_mono}. 
\end{appendices}
\end{document}